\documentclass[english]{article}

\usepackage{microtype}
\usepackage{graphicx}
\usepackage{subcaption}
\usepackage[percent]{overpic}
\usepackage{booktabs} %

\usepackage{babel}

\usepackage{paralist}

\usepackage{calc}

\usepackage[hyphens]{url}
\usepackage{hyperref}

 \usepackage[accepted]{icml2026}

\usepackage{amsmath}
\usepackage{amssymb}
\usepackage{mathtools}
\usepackage{amsthm}

\usepackage{capt-of} %
\usepackage{subcaption}

\usepackage{dsfont}
\usepackage{mathrsfs}

\newtheoremstyle{definition_style}
{15pt}
{10pt}
{\itshape}
{}
{\bfseries}
{ }
{\newline}
{}

\newtheorem{Th}{Theorem}[section]
\newtheorem{Lemma}[Th]{Lemma}
\newtheorem{Cor}[Th]{Corollary}

\theoremstyle{definition}
\newtheorem{Def}[Th]{Definition}
\newtheorem{Thm}[Th]{Theorem}

\newtheorem{Lem}[Th]{Lemma}

\newtheorem{Ex}[Th]{Example}

\usepackage[capitalize,noabbrev]{cleveref}

\theoremstyle{plain}
\newtheorem{theorem}{Theorem}[section]

\newtheorem{lemma}[theorem]{Lemma}
\newtheorem{corollary}[theorem]{Corollary}
\theoremstyle{definition}

\theoremstyle{remark}

\usepackage[textsize=tiny]{todonotes}

\definecolor{myCustomGreen}{rgb}{0.13, 0.55, 0.13}  %

\icmltitlerunning{Graph Neural Networks Are Not Continuous Across Graph Resolutions}

\begin{document}

\twocolumn[
  \icmltitle{Graph Neural Networks Are Not Continuous Across Graph Resolutions}

  \icmlsetsymbol{equal}{*}

  \begin{icmlauthorlist}
    \icmlauthor{Christian Koke}{aithyra,mcml,tum}
    \icmlauthor{Yuesong Shen}{mcml,tum}
    \icmlauthor{Abhishek Saroha}{mcml,tum}
    \icmlauthor{Marvin Eisenberger}{tum}
    \icmlauthor{Bastian Rieck}{fribourg,helmholtz}
    \icmlauthor{Michael Bronstein}{aithyra,oxford}
    \icmlauthor{Daniel Cremers}{mcml,tum}
  \end{icmlauthorlist}

  \icmlaffiliation{aithyra}{AITHYRA}
  \icmlaffiliation{mcml}{Munich Center for Machine Learning}
  \icmlaffiliation{tum}{Technical University of Munich, Munich}
  \icmlaffiliation{fribourg}{University of Fribourg}
  \icmlaffiliation{helmholtz}{Institute of Computational Biology}
\icmlaffiliation{oxford}{University of Oxford}

  \icmlcorrespondingauthor{Christian Koke}{christian.koke@tum.de}

  \icmlkeywords{Machine Learning, ICML}

  \vskip 0.3in
]

\frenchspacing

\printAffiliationsAndNotice{}  %
\begin{abstract}
We show that contrary to conventional wisdom in the community, graph neural networks (GNNs) are not continuous with respect to all natural modes of graph convergence. As a result, GNNs may generate substantially different latent repre-\\ sentations for graphs that are very similar. In par-\\ ticular they assign vastly different latent embeddings to graphs that represent the same underlying object at different resolution scales.   We trace this failure of continuity back to a struc-\\
tural obstruction arising from commonly used in-\\
formation-propagation schemes. Building on this insight, we then derive a principled modification to standard GNN architectures which equips models with continuity across scales. The proposed modification enables consistent integration of distinct resolutions and reliable generalization between them. We systematically validate our theoretical findings in a wide range of numerical experiments. Our code is available at \url{https://github.com/christiankoke/scale-continuous-gnns}.
\end{abstract}

\section{Introduction}\label{introduction}

Graph Neural Networks (GNNs) \cite{scarselli2009graph, BrunaOrig, mpnncm}
have become the dominant framework for learning on graph structured data and have achieved widespread success in domains as varied as protein structure prediction \citep{jumper2021highly},  material science \citep{xie2018crystal}, weather forecasting \citep{global_weather} or catalyst screening \citep{catalyst_screening}.

A key theoretical pillar supporting their widespread adoption is the
commonly held belief that GNNs define \emph{stable} and
\emph{continuous} mappings: small changes in the input graph are
expected to induce correspondingly small changes in learned
representations. Indeed, a substantial body of prior work has
established continuity, stability, and robustness for GNNs under a broad
range of perturbation modes: Previous investigations e.g. considered GNN-stability under small variations in edge-weights \citep{LevieStabSingle, DBLP:journals/tsp/GamaBR20} or graph rewiring \citep{kenlay2021onWS}.  Consistency of latent representations 
 has also been established across graphs that faithfully approximate the same metric measure space \citep{DBLP:journals/corr/abs-1907-12972}, graphon \citep{DBLP:conf/nips/RuizCR20,  DBLP:journals/corr/abs-2109-10096},  graphing \citep{RoddenberryGBS22} or graphop \citep{le2023limits}.\footnote{Appendix \ref{app:related work} further summarizes these existing perturbation models, also discussing their inapplicability to our setting.}
Taken together, these results have contributed to a prevailing view in
the field that graph neural networks define broadly stable and
well-behaved mappings: \emph{similar input graphs are mapped to similar
latent embeddings}.

In this work, we show that this is, however, not always the case:
 In particular we demonstrate that standard GNN architectures can produce \emph{drastically different latent representations} for graphs that describe the same underlying object at different \emph{resolution scales}. 
As a consequence, models trained on graphs at one resolution, for
instance, \emph{fail} to generalize to equivalent graphs at another resolution.

We trace this failure to the previously understudied fact that graph
neural networks are generally \emph{not} continuous with respect to all natural modes of graph convergence: We observe that there is a natural interpolating sequence of graphs starting at an original graph and terminating at its coarse-grained version. Mathematically, this sequence smoothly interpolates between the original- and coarse heat kernel.
This convergence is not respected by GNNs: Even though input graphs converge (in the heat sense) the generated sequence of latent embeddings does not converge.

 This observation reveals a fundamental limitation of existing GNN architectures. It also motivates the development of modifications to existing architecures, in order to render them \emph{continuous} across graph resolution scales.  

 Concretely, our paper makes the following \textbf{contributions}:

\begin{itemize}
	\item We identify and empirically demonstrate a previously
    underappreciated failure mode of standard GNNs, i.e., \textbf{a lack of continuity across graph resolutions}.
	\item We provide a \textbf{theoretical explanation} linking this failure mode to discontinuity under graph convergence. 
	\item We show that the \textbf{underlying obstruction is structural}, originating in standard propagation-schemes.
	\item We  show how to \textbf{modify propagation schemes} to yield \textbf{provably scale-continuous GNN architectures}.
	\item We establish \textbf{theoretical guarantees} and present \textbf{extensive empirical validation} demonstrating our modifications facilitate reliable \textbf{generalization across scales}.
\end{itemize}

\section{Graphs and Graph Learning}

We first briefly review graphs and graph neural networks:

\subsection{Graphs and their Description}\label{sec:graph_description}

We consider finite, weighted graphs \(G = (V,E)\) equipped with
a \emph{node weight matrix} \(M = \operatorname{diag}(\mu_v)_{v\in V}\).
Each weight \(\mu_v > 0\) reflects  node-wise information. Examples are
physical mass or charge of atom $v$ in a molecular graph,  number of
users represented if  $v$ is a node in a social network, or a volume
element if \(v\) is part of a point cloud. Often, all node masses are set
to $\mu_v = 1$ so that $M = I$. Edges
\((u,v)\in E\) of a graph carry symmetric non-negative (potentially binary) weights \(A_{uv}\) which we
collect in the adjacency matrix \(A\).  The corresponding degree matrix
is \(D = \operatorname{diag}(d_v)\) with \(d_v = \sum_{u} A_{uv}\).

The natural self-adjoint Laplacian associated with a node-weighted graph
is $ L = M^{-1}(D - A)$. It is self-adjoint with respect to the \(M\)-weighted inner
product $ \langle f,g\rangle_M = f^\top M g$. When the node
weights encode local volumes or sampling densities it may be thought of
as the discrete counterpart of the Laplacian $L = - \Delta$ on
a continuum domain \citep{Chung_1997}.
While the natural Laplacian is central to our work, more commonly used in machine learning applications is instead the normalized Laplacian
$
 \mathcal{L}
 = I - D^{-1/2} A\, D^{-1/2}
$. It is self-adjoint in the standard Euclidean inner product and its
spectrum is contained in the interval \([0,2]\).

\subsection{Graph Neural Networks}

Graph neural networks are machine learning methods adapted to handle graph structured data. Corresponding network architectures  broadly fall  into two categories:

\paragraph{Message Passing Networks:}
Message Passing Neural Networks (MPNNs) \cite{mpnncm, bronstein2021geometric}  are arguably the most widely used class of GNNs. At each layer \(k\), node features \(h_v^{(k)}\) are updated by aggregating information from neighbouring nodes:
	\begin{align}
	m_{vu}^{(k)} &=    \phi(h_v^{(k)}, h_u^{(k)}, A_{uv}) \label{eq:mf_form} \\
	h_v^{(k+1)} &= \text{AGG}_{u \in \mathcal{N}(v)}  m_{vu}. 
\end{align}
Here $A_{uv}$ denotes edge weights.
 Architectutes differ in which message- and aggregation functions are being used \citep{GIN, GAT, Kipf, SAGE}, with typcial choices for the neighborhood aggregation function \(\operatorname{AGG}\) being 
 \(\operatorname{sum}\), \(\operatorname{avg}\), or (seldom) \(\operatorname{max}\) \citep{fey2025pyg2}.

\paragraph{Spectral Networks:}
Spectral graph concolutional layers \citep{BrunaOrig, Bresson} apply learnable filters $g_\theta(\mathcal{L})$  to node-wise information via matrix-multiplication. The filters themselves are composed of learnable scalar functions $g_{\theta}(\cdot)$  which are applied to the normalized graph Laplacian $\mathcal{L}$.
In practice, the filter \(g_\theta\) is not implemented via
eigendecomposition but approximated via
polynomials  \citep{chebII, bernnet, koke2024holonets, Cayley}. The prototypical example is ChebNet, where  filters are parametrized as $
g_\theta(\lambda)
\;\approx\; \sum_{\ell=0}^{K} \theta_\ell\, T_\ell(\lambda),
$
with \(T_\ell\) the $\ell^{\text{th}}$-Chebyshev polynomial.
With node feature matrix $X$ and learnable weight matrices \(\{W_\ell\}_{\ell}\) , a spectral layer is then efficiently implemented as
\begin{align}\label{eq:spectral_prop}
X^{(k+1)} = 
\sum_{\ell=0}^{K} T_\ell(\mathcal{L})\, X^{(k)}\, W_\ell .
\end{align}

\section{The Failure Mode: Lack of Generalization
}\label{sec:failure_mode}

To showcase how standard graph neural networks are unable to
consistently integrate varying scales experimentally, we make use of the QM$7$ dataset \citep{rupp},
consisting of organic molecules containing both hydrogen and 
heavy atoms.  Prediction target is the molecular atomization energy. Each molecule is represented by a weighted adjacency matrix whose 
entries 
$A_{ij}	 = Z_iZ_j\cdot|\vec{x}_i-\vec{x}_j|^{-1}$
correspond to Coulomb energies between atoms $i,j$, with $|\vec{x}_i-\vec{x}_j|$ denoting the interatomic distance and $Z_i$ atomic charge.
From a physical perspective, describing a molecule at the level of interacting atoms corresponds to a specific choice of resolution scale:
Interactions of individual protons and neutrons inside the various atomic nuclei are discarded.\footnote{ At an even higher resolution scale, also interactions of quarks and gluons (in turn making up protons and neutrons) are neglegted.} 
Instead, only an aggregate description is used wherein each nucleus is only described by a single node. 

To test the ability of GNNs to generalize across scales,
we additionally also consider a version of QM$7$ where we lower the resolution scale even further:
To this end we aggregate each heavy atomic core together with its surrounding (single-proton) hydrogen atoms into super-nodes. As depicted in Fig. \ref{collapse_weighted_0}, this is done via standard  graph coarsification \cite{pmlr-v80-loukas18a, loukas2019graph}; Appendix \ref{app:qm7_experiment} provides exact details.

	\begin{figure}[H]
		(a) \includegraphics[scale=0.4]{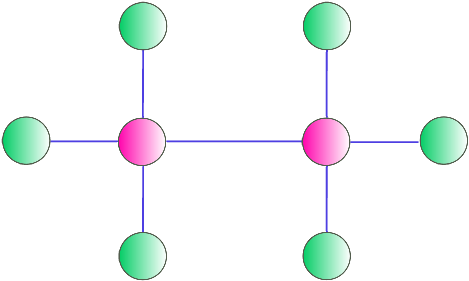} \ \ \ \ \ \ \ \ \ \ \ \
		(b) \includegraphics[scale=0.45, trim= 0 -45 0 0]{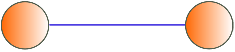}
		\caption{(a) 
			Original graph
			$G$ corresponding to the Ethane molecule with Carbon in purple and Hydrogen in green (b) Coarse grained $\underline{G}$ with aggregate Carbon-Hydrogen super-nodes in orange } 
		\label{collapse_weighted_0}
		\vspace{-2mm}
	\end{figure}

This 
QM$7_{\text{coarse}}$ dataset models  data obtained from a resolution-limited observation process
unable to resolve positions of individual (small) hydrogen atoms
and only providing information about how many are bound to a given heavy atom. 
Using the high-resolution graphs $\{G\}$ of QM$7$ and the low-resolution graphs $\{\underline{G}\}$ in  QM$7_{\text{coarse}}$, we then investigate the
ability of GNNs\footnote{Appendix \ref{app:qm7_experiment} contains further details on models.}  to consistently handle various scales
by confronting models during inference with a resolution-scale different from the one  they were trained on.

Table \ref{tab:qm7_baseline_model_results} collects results. 
Mean-absolute-errors (MAEs) during inference increase significantly, when going from a same-resolution  setting to a cross-resolution setting. 
\emph{None of the standard architecture types are
able to consistently handle more than one scale.}
Also employing multi-scale propagation schemes (SAG, SAG-M, UFGNet, Lanczos)
does not result in consistently incorporating multiple scales:
Corresponding cross-resolution MAEs 
are among the largest.

\begin{table}[h!]
	\small
	\centering
	\caption{QM7 regression. Mean Absolute Error (MAE $\downarrow$) in kcal/mol for training and inference at different resolutions scales.}
	\label{tab:qm7_baseline_model_results}
	
	\vspace{-2mm}
	\setlength{\tabcolsep}{1.5pt}
	\renewcommand{\arraystretch}{1.1}
	
	\begin{tabular}{lcccc}
		\toprule
		\textbf{Resolution} & \multicolumn{4}{c}{\textbf{MAE ($\downarrow$) on QM7 [kcal/mol]}} \\
		\cmidrule(lr){2-5}
		
		Training: 
		& \multicolumn{2}{c}{\textbf{High}} 
		& \multicolumn{2}{c}{\textbf{Low}} \\
		\cmidrule(lr){2-3} \cmidrule(lr){4-5}
		
		Inference: 
		& \textbf{Low} 
		& High 
		& Low 
		& \textbf{High} \\
		\midrule
		
		GCN  
		& \textbf{136.7$\pm{\scriptstyle 6.6}$} 
		& (63.6$\pm{\scriptstyle 1.3}$) 
		& (63.6$\pm{\scriptstyle 1.3}$)
		& \textbf{138.1$\pm{\scriptstyle 2.4}$} \\
		
		GATv2  
		& \textbf{423.5$\pm{\scriptstyle 337.1}$} 
		& (67.4$\pm{\scriptstyle 8.2}$) 
		& (59.7$\pm{\scriptstyle 2.7}$) 
		& \textbf{257.4$\pm{\scriptstyle 139.1}$} \\
		
		ChebNet   
		& \textbf{447.8$\pm{\scriptstyle 6.0}$} 
		& (66.7$\pm{\scriptstyle 1.4}$) 
		& (71.5$\pm{\scriptstyle 2.1}$) 
		& \textbf{158.7$\pm{\scriptstyle 57.4}$} \\
		
		GIN
		& \textbf{658.4$\pm{\scriptstyle 85.8}$} 
		& (17.$\pm{\scriptstyle 2.8}$) 
		& (38.3$\pm{\scriptstyle 21.6}$) 
		& \textbf{1835.4$\pm{\scriptstyle 925.8}$} \\
		
		SAG  
		& \textbf{589.7$\pm{\scriptstyle 44.9}$} 
		& (68.2$\pm{\scriptstyle 2.6}$) 
		& (107.9$\pm{\scriptstyle 1.1}$) 
		& \textbf{283.6$\pm{\scriptstyle 39.3}$} \\
		
		SAG-M  
		& \textbf{194.4$\pm{\scriptstyle 29.9}$}
		& (66.6$\pm{\scriptstyle 1.9}$) 
		& (77.8$\pm{\scriptstyle 6.0}$) 
		& \textbf{219.5$\pm{\scriptstyle 11.7}$} \\
		
		UFGNet  
		& \textbf{131.5$\pm{\scriptstyle 6.1}$} 
		& (62.4$\pm{\scriptstyle 0.7}$) 
		& (69.4$\pm{\scriptstyle 0.7}$) 
		& \textbf{148.1$\pm{\scriptstyle 6.3}$} \\
		
		Lanczos  
		& \textbf{938.4$\pm{\scriptstyle 2.5}$} 
		& (9.9$\pm{\scriptstyle 2.5}$) 
		& (88.2$\pm{\scriptstyle 2.7}$) 
		& \textbf{658.6$\pm{\scriptstyle 199.2}$} \\
		
		\bottomrule
	\end{tabular}
	\vspace{-1mm}
\end{table}

We can trace the inability of common models to generalize back to the
difference in latent embeddings $\{F\}$ and $\{\underline{F}\}$, 
generated
 for original
 $\{G\}$  and coarsified
graphs $\{\underline{G}\}$.

\begin{table}[h!]
	\small
	\centering
	\caption{
Embedding difference $\|F - \underline{F}\|$ across resolution scales averaged over 5 runs (mean$\pm$std). Lower is better ($\downarrow$).
	}
	\label{tab:embed_diff}

	\setlength{\tabcolsep}{4pt}
	\renewcommand{\arraystretch}{1.1}
	
	\begin{tabular}{ccc}
		\toprule
		\textbf{Training $\rightarrow$ Inference:} 
		&High $\rightarrow$ Low 
		& Low $\rightarrow$ High \\
		\midrule
		
		GCN      & $20.3\pm{\scriptstyle 1.6}$    & $18.0\pm{\scriptstyle 0.6}$ \\
	GATv2    & $76.8\pm{\scriptstyle 52.4}$   & $42.5\pm{\scriptstyle 24.3}$ \\
	ChebNet  & $123.0\pm{\scriptstyle 3.8}$   & $158.5\pm{\scriptstyle 57.40}$ \\
	GIN      & $881.2\pm{\scriptstyle 268.7}$ & $2875.8\pm{\scriptstyle 1527.7}$ \\
	SAG      & $113.5\pm{\scriptstyle 10.3}$  & $43.3\pm{\scriptstyle 5.4}$ \\
	SAG-M    & $75.0\pm{\scriptstyle 15.1}$   & $61.7\pm{\scriptstyle 4.6}$ \\
	UFGNet   & $24.5\pm{\scriptstyle 0.3}$    & $26.8\pm{\scriptstyle 2.1}$ \\
	Lanczos  & $1286.1\pm{\scriptstyle 54.4}$ & $728.4\pm{\scriptstyle 230.0}$ \\
		\bottomrule
	\end{tabular}
\end{table}

As evident from Table \ref{tab:embed_diff}, latent embeddings 
of graphs describing the same object on varying resolutions are significantly different. This in turn explains the inability to generalize across scales.
\noindent It should be noted that in practice, this failure \emph{cannot} be overcome by augmenting the training set, as we have no way of generating faithful high-resolution descriptions given only lower resolution training data.

\section{The Underlying Problem: Discontinuity
}\label{sec:non-cont}

Our next goal then is to understand \textit{why} the latent embeddings $F, \underline{F} $ generated for  original ($G$) and coarse ($\underline{G}$) graph descriptions of the same underlying object
 are so vastly different. 
Counter to conventional wisdom and consensus within the graph learning community we find below that 
\begin{center}
  \textbf{GNNs are not continuous across scales}.
\end{center}
Hence, they may produce vastly different outputs \((F, \underline{F})\)
for highly similar inputs \((G, \underline{G})\).

To see this discontinuity in action, note that in the coarse graphs \(\{\underline{G}\}\) of QM7\(_{\text{coarse}}\), hydrogen atoms have been fused into their nearest heavy atoms.
Such a coarse graph hence can be viewed as the limit of a process in which hydrogen atoms are displaced from equilibrium toward their nearest heavy atom. The limit graph is obtained once they are fully merged into the corresponding heavy nodes.

\begin{figure}[H]
	\vspace{-2mm}
	\includegraphics[scale=0.35]{figures/Ethan_Original}(a)\hfill
	\includegraphics[scale=0.35, trim= 0 -15 0 0]{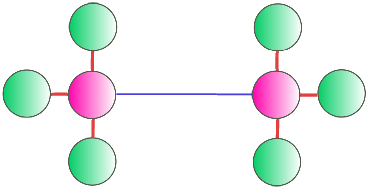}(b)\hfill
	\includegraphics[scale=0.35, trim= 0 -45 0 0]{figures/Ethan_Collapsed}(c)\hfill
	\caption{Collapsing Procedure visualized} 
	\label{fig:collapse_weighted}
\end{figure}

If standard GNN architectures were continuous,
 the convergence of this
graph modification process in Fig. \ref{fig:collapse_weighted} towards a limit graph should be reflected
also in the latent space:
Latent embeddings of modified graphs should converge to the latent embedding of the limit graph.
 
Figure \ref{collapse_graph_qm7_w_exp} compares embeddings $\{\underline{F}\}$ generated for
coarsified graphs $\{\underline{G}\}$, with 
embeddings $\{ F_\omega\}$ of
graphs $\{G_\omega\}$ where hydrogen atoms
have been moved to reduce the distance towards their nearest heavy atoms by a factor of $\omega \geq 1$ (i.e. $\text{dist}_\text{new} = \text{dist}_\text{equilib.}/\omega$),
but have not yet
completely arrived  at the positions of nearest  heavy atoms.
      \begin{figure}[H]
      	\vspace{-3mm}
      		\begin{overpic}[scale=0.4]{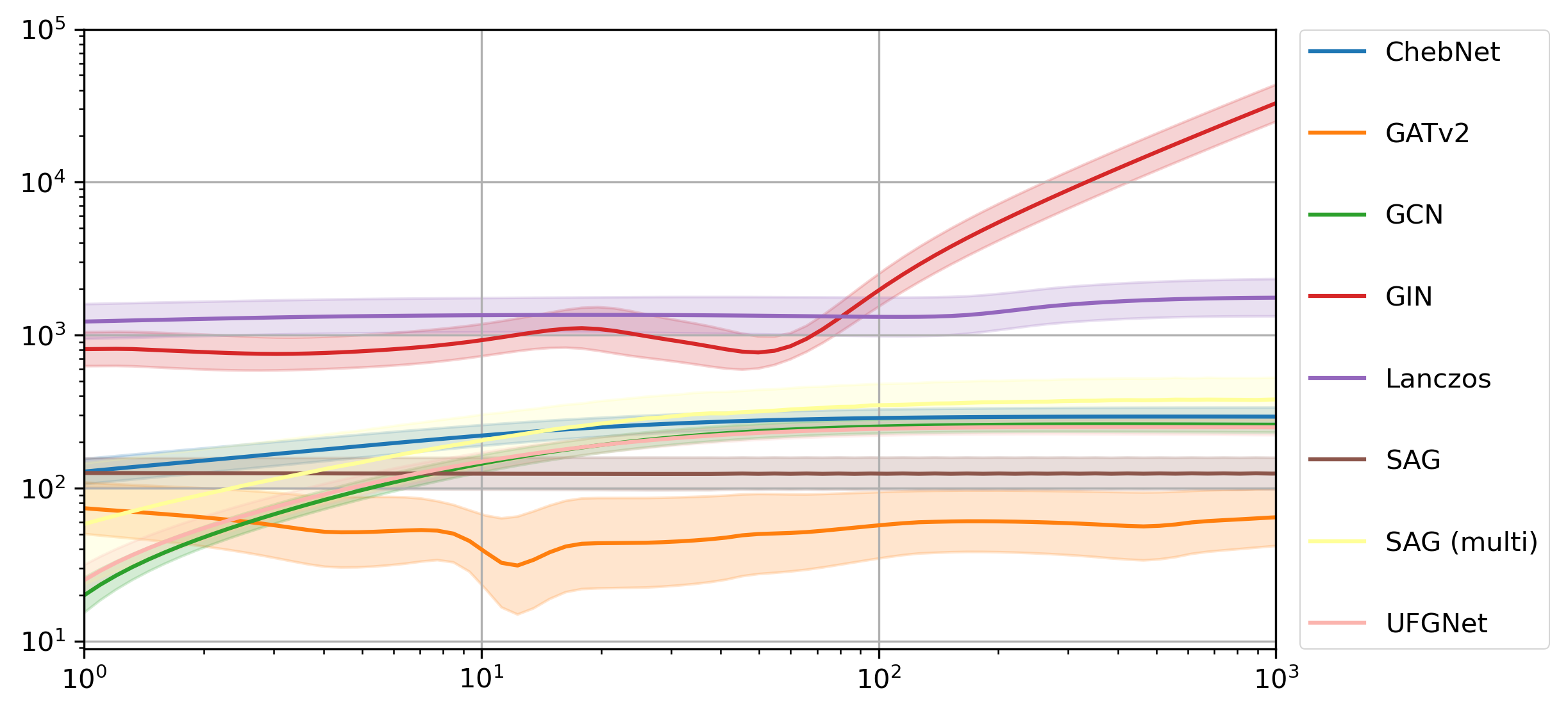}
      		\put(43,0){\small $\omega$}
      	\end{overpic}
				\caption{Latent distance $\|F_\omega-\underline{F}\|$ 
				}
				\vspace{-7mm}
				\noindent
				\label{collapse_graph_qm7_w_exp}
			\end{figure}
From Fig. \ref{collapse_graph_qm7_w_exp}, we observe that latent embeddings do \emph{not} converge
($\|F_\omega-\underline{F}\|\nrightarrow 0$). Since the sequence
$\{G_w\}_w$ of input graphs converges\footnote{Section \ref{subsec:graph_to_kernel} discusses the precise sense in which this convergence is formalized
mathematically.} but GNN outputs evidently do not, they indeed
cannot be considered continuous.
Hence they may map similar graphs to
vastly different latent embeddings.

\section{The Obstruction: GNN Propagation 
}\label{sec:effective_prop_section}

We can understand the underlying reason for this discontinuity by
exemplarily investigating the prototypical graph neural network GCN
\citep{Kipf}:\footnote{Appendix \ref{app:limitprop} contains derivations
for all standard GNN types.}
Inside a GCN-layer, a node feature matrix $X\in \mathbb{R}^{N \times F}$  (with number of nodes $N$ and feature dimension $F$) is updated  as
\begin{equation}\label{eq:gcn_prop}
	X\mapsto \hat{A} X W.
\end{equation}
Here $W \in \mathbb{R}^{F \times F}$ mixes channels. Information flow
over the graph is implemented via the so called \textit{renormalized}
adjacency matrix  $\hat{A} \in \mathbb{R}^{N \times N}$, with
$\hat{A}_{ij}  \sim A_{ij}/\sqrt{d_id_j}$. 
When moving hydrogen (H) atoms towards heavy atoms
($|\vec{x}_{\text{H}}-\vec{x}_{\text{heavy}}| \rightarrow 0$),
corresponding edge weights  $A_{\text{H},\text{heavy}}	 = 1\cdot
Z_{\text{heavy}}\cdot|\vec{x}_{\text{H}}-\vec{x}_{\text{heavy}}|^{-1}
$ of the \textit{original} adjacency matrix $A$ tend to infinity. Thus
also node-degrees  associated to heavy atoms tend to infinity. Since
distances (and hence weights in the adjacency matrix  $A$) between heavy atoms remain constant, the
``\textit{renormalized}" entries $\hat{A}_{\text{heavy},\text{heavy}}$ in
$\hat{A}$, however, tend to zero instead.

Hence, as hydrogen atoms are moved out of equilibrium towards their final positions,
			the communication between heavy atoms in the modified graphs $G_\omega$ becomes severely disrupted ($\hat{A}_{\text{heavy},\text{heavy}} \rightarrow 0$). Information is only propagated along a severely disconnected effective limit graph (dissected into distinct connected components (Fig \ref{collapse_weighted_II} (a));  and not along the true coarse limit graph $\underline{G}$ (Fig. \ref{collapse_weighted_II} (b)).
			\begin{figure}[H]
				\centering
				(a)	\includegraphics[scale=0.5, trim= 0 -10 0 0]{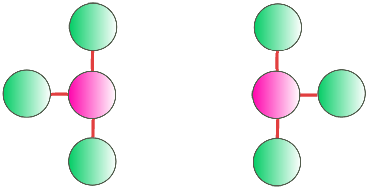} \ \ \  \ \ \ 
				(b) \includegraphics[scale=0.5, trim= 0 -45 0 0]{figures/Ethan_Collapsed}
				\caption{
					(a) Effective propagation vs. (b) true coarse graph  $\underline{G}$}
				\label{collapse_weighted_II}
			\end{figure}
		As a consequence of the information flows over the graphs $G_\omega, \underline{G}$
    being vastly different, the latent embeddings $F_\omega, \underline{F}$ that are being generated for the two graphs differ greatly.
  At first glance, it may seem that the observations above apply only to
  the QM7 dataset. There edge weights scale inversely with distance, so
  sequences of graphs with diverging weights arise naturally. In
  Section~\ref{solution} immediately below, however, we discuss that this
  is merely a proxy for a general fact in graph theory: the natural
  interpolation between an original graph \(G\) and its coarsified
  version \(\underline{G}\) proceeds by letting the edge weights within
  those clusters of \(G\) that collapse to the super-nodes of
  \(\underline{G}\) tend to infinity \citep{koke2026a, koke2026b}. 
  Thus the observations above are indeed applicable to \emph{generic} datasets.

		\section{Solution: Laplace Transform Propagation
			}\label{solution}
		Having established that standard GNNs are not continuous across scales due to the way they propagate information over a given graph $G$, let us now work towards building architectures that instead \emph{are}  continuous. 
		
		\subsection{From Graphs to Heat Kernels}\label{subsec:graph_to_kernel}
		To this end, let us begin by  formalizing rigorously, in which sense the sequence of graphs $G_\omega$ considered in Section \ref{sec:effective_prop_section} approaches the coarsified limit graph $\underline{G}$:
		
		\paragraph{Heat Kernel convergence:}
		When moving hydrogen atoms out of equilibrium, we are significantly increasing certain weights ($A_{\text{H},\text{heavy}}	 \sim |\vec{x}_{\text{H}}-\vec{x}_{\text{heavy}}|^{-1}  \sim \omega \rightarrow \infty$).
		From a heat diffusion perspective, information in a graph equalizes
		much 
		faster along edges with very large weights.
		In the limit where 
		edge-weights within certain  sub-graphs tend to infinity, information  within these clusters  equalizes immediately. Each such sub-graph thus effectively behaves as a single node in a coarse grained effective graph $\underline{G}$.
		
		To quantify 
		this,
		we recall that the diffusion equation on a graph  is given by
		$dX(t)/dt = -L \cdot X(t)$ with solution
		$	X(t) = e^{-L t} \cdot X(0)$. 
		As was established rigorously in \citet{koke2026b}  (cf. Appendix \ref{coarse_grain_proofs} for a discussion) the convergence
		\begin{equation}\label{eq:orig_exp_conv_result_II}
				\|e^{-tL_\omega} - J^\uparrow e^{- t\underline{L}} J^\downarrow\| \rightarrow 0
		\end{equation}
		then holds for any fixed $t>0$ as $\omega \rightarrow \infty$.
		Here $L_\omega, \underline{L}$ are the Laplacians of the respective graphs $G_\omega, \underline{G}$. 
		The matrices $J^{\downarrow,\uparrow}$ appear naturally and allow to compare graphs of different sizes: 	Projection $J^\downarrow$ assigns the average over strongly connected clusters to the super-node representing this cluster in $\underline{G}$. 
		Interpolation $J^\uparrow$ 
		assigns information at a super-node in $\underline{G}$ to each node in the respective cluster in $G$.

It is important to note that in order to detect the convergence $G_\omega \rightarrow \underline{G}$ as in (\ref{eq:orig_exp_conv_result_II}) one needs to make use of the (non-normalized) Laplacian $L$. Using e.g. the normalized Laplacian $\mathcal{L}$ would not lead to convergence  ($e^{-t\mathcal{L}_\omega} \nrightarrow J^\uparrow e^{-t\underline{\mathcal{L}}} J^\downarrow$). Also directly comparing the Laplacians $L, L_\omega$ would not detect the convergence: Clearly $\|L_\omega\| \rightarrow \infty$, so that the sequence of Laplacians $L_\omega$ alone does not converge.

			\paragraph{Interpolating between Graphs on different scales:}
			
			We can use the convergence guarantee of (\ref{eq:orig_exp_conv_result_II}) to interpolate ("in edge weight space") between an original graph $G$ and its coarse grained version $\underline{G}$ via a family of intermediate graphs $\{G_\omega\}_\omega$ with ever increasing intra cluster weights.
			\begin{figure}[H]
				\vspace{-2mm}
				\includegraphics[scale=0.2]{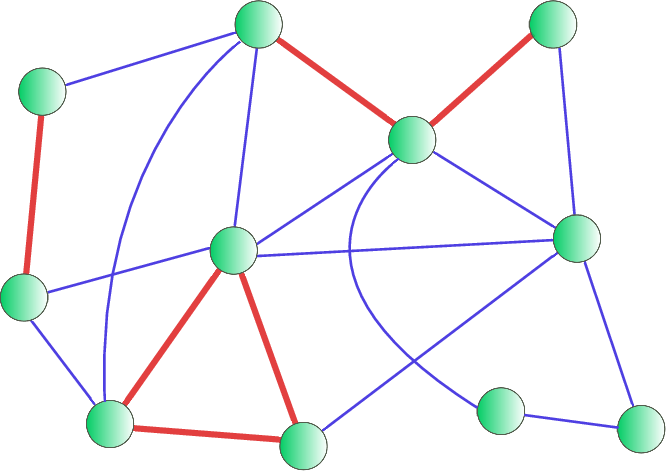}(a)\hfill
				\includegraphics[scale=0.2]{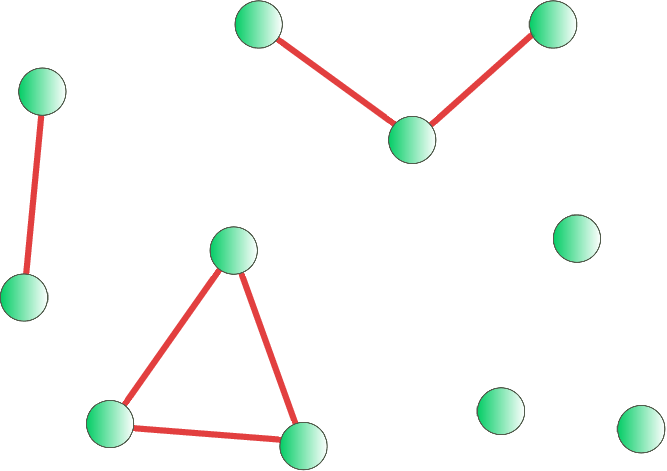}(b)\hfill
				\includegraphics[scale=0.2]{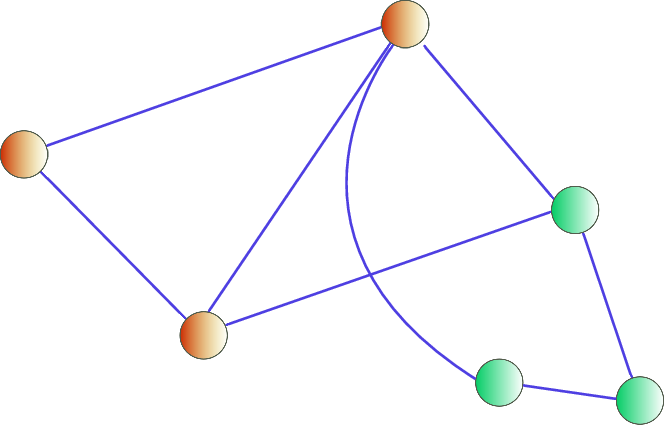}(c)
				\captionof{figure}{(a) $G_\omega$.\ \ (b) Scaled edges of $G_\omega$   in \textcolor{red}{red}.\ \ (c)  $\underline{G}$.} 
				\label{collapse_weighted_oo}
			\end{figure}
			
As we let $\omega \rightarrow \infty$, the heat kernel on $G$ then more and more resembles the one on $\underline{G}$. 	To visualize this fact, we	exemplarily  plot in Fig.~\ref{fig:the_next_generation} (c)  the difference	$\| e^{-L_\omega t}   -   J^\uparrow e^{-\underline{L} t}J^\downarrow \|$ for the coarse graining setting of Figure  \ref{fig:the_next_generation} (a, b).

			\begin{figure}[H]
				\centering
				\includegraphics[scale = 0.5, trim= 0 -25 0 0]{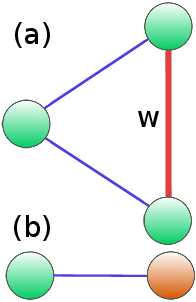}\ \ \ \ 
				\includegraphics[scale=0.27]{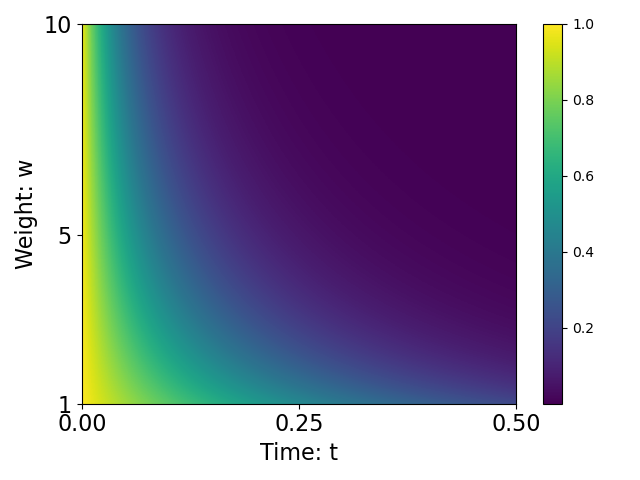}(c)
				\captionof{figure}{	$\|e^{-tL_\omega} - J^\uparrow e^{- t\underline{L}} J^\downarrow\|$-plot for graphs (a) \& (b)} 
				\label{fig:the_next_generation}
			\end{figure}

		For fixed $t > 0$ we see that $\|e^{-tL_\omega} - J^\uparrow e^{- t\underline{L}} J^\downarrow\| \rightarrow 0$ as 
		$\omega$ increases.
		Additionally, the decay $	\|e^{-tL_\omega} - J^\uparrow e^{- t\underline{L}} J^\downarrow\| \rightarrow 0$ for increasing $t$ is faster, the larger $w$ is chosen. This is congruent with our intuition: The stronger two nodes are connected, the more they act as a single entity.

		\subsection{From Heat Kernels to Propagation Schemes}
				
		From (\ref{eq:orig_exp_conv_result_II}) and Fig. \ref{fig:the_next_generation} we conclude that applying
		 the matrix $e^{-tL_\omega}$ on $G$ is more and more the same as projecting to $\underline G$ via $J^\downarrow$, applying 
		  $e^{-t\underline{L}}$ there and interpolating back via  $J^\uparrow$. 
		
			For a network to be continuous across scales, we need exactly this property:  We want the propagation scheme on $G_\omega$ to more and more resemble the one on $\underline{G}$ as we increase $\omega \rightarrow \infty$.
		As we have seen in Section \ref{sec:non-cont}, the propagation rule $X \mapsto \hat{A} X W$ of GCN does not achieve this, as it leads to disconnected limit graphs instead. However,  propagating as
		\begin{equation}
		X \mapsto e^{-tL} X W,
		\end{equation}
		\emph{does} facilitate contact and similarity between information flows over $G_\omega$ and $\underline{G}$ as $\omega \rightarrow \infty$.

		More generally, suppose we have for each time  $t > 0$ individually that $\|e^{-L_{\omega}t} - J^\uparrow e^{- t\underline{L}t} J^\downarrow\| < \delta$. If we build up the propagation matrix $\psi(L_\omega)$ as a weighted sum of such diffusion flows $e^{-tL_\omega}$ that have progressed to various times ($\psi(L_\omega) \sim \sum_{k} a_k e^{-t_k L_\omega}$) and the coefficients $\{a_k\}_k$ are not too large, then we can estimate $\| \psi(L_\omega) - J^\uparrow \psi(\underline{L}) J^\downarrow\| \leq  \left( \sum_k |a_k|\right)\cdot \delta =: \epsilon$ by a triangle-inequality argument.
		Thus we can still guarantee that for large $\omega$ the propagation implemented by the layer-wise update rule
		\begin{equation}\label{eq:global_laplacian_proagation}
			X \mapsto \psi(L_\omega) XW 
		\end{equation}
		over $G_\omega$ is approximately the same as the effective propagation $X \mapsto [J^\uparrow \psi(\underline{L}) J^\downarrow] XW $ over $\underline{G}$.

		Generalizing the weighted sum to an integral, we 
define:

		\begin{Def}\label{def:LTF_def}
			Let $\hat{\psi}$ be a bounded (generalized) function 
			on $[0,\infty)$. 
			The corresponding \textbf{Laplace-transform propagation matrix} is the matrix  $\psi(L) \in \mathbb{R}^{N \times N}$  arising as the Laplace transform (cf. Appendix \ref{app:ltf_matrices} for details) of  $\hat{\psi}$:
			\begin{equation}\label{eq:LT_integral}
				\psi(L) := \int_0^\infty e^{-tL}\hat{\psi}(t) dt
			\end{equation}
		\end{Def}
		Allowing \textit{generalized} functions means we e.g. allow:
		\begin{itemize}
			\item  
		Dirac distributions $\hat{\psi}_{\delta_{s}}(t) := \delta(t-s)$; yielding back \textbf{heat kernels} $\psi_s(L) = \int_{0}^\infty \delta(t-s) e^{-tL}dt = e^{-s L}$.
		\item  Decaying functions  $\hat{\psi}_k := (-t)^{k-1}e^{-\lambda t}$ yielding $k^{\text{th}}$-powers of \textbf{resolvents} $\psi_k(L) = [(\lambda I+L)^{-1}]^k$.
		\end{itemize}
		
		For such matrices $\Psi(L)$ we note (details in Appendix \ref{app:graph_level_coarse_fine}):
		\begin{lemma}[informal]\label{lem:simple_cty}
						Using Laplace-transform propagation matrices as in (\ref{eq:LT_integral}) together with the propagation rule (\ref{eq:global_laplacian_proagation}) in each layer leads to scale continuos networks: If $G_\omega \rightarrow\underline{G}$ in the sense of heat kernels, then $F_\omega \rightarrow \underline{F}$ for the corresponding generated latent representations $F_\omega, \underline{F}$.
		\end{lemma}

\section{The Result: Scale-Continuous GNNs}\label{sec:scale_ct_gnns}

	Lemma \ref{lem:simple_cty} above guarantees networks built on the propagation rule (\ref{eq:global_laplacian_proagation}) are continuous across scales. We may think of this propagation rule  as a scale continuous modification to the propagation rule (\ref{eq:gcn_prop}) of GCN.
GCN in turn is simultaneously both the simplest spectral- as well as message-passing based graph neural network. 
Our next goal is therefore to leverage the insights obtained from making GCN scale continuous to extend scale continuity to general spectral- and message-passing based GNNs.

		\subsection{Laplace Transform Propagation}\label{subsec:sc_gnns}
	
		\paragraph{Spectral Networks:}
In order to turn spectral networks continuous across scales, we replace the polynomials $T_\ell(\mathcal{L})$ of the normalized Laplacian $\mathcal{L}$ used in their standard layer-wise update (\ref{eq:spectral_prop}) with functions $\psi_\ell(L)$ of the graph Laplacian $L$ as defined in (\ref{eq:LT_integral}). The resulting layer wise update reads	
		\begin{align}\label{eq:spectral_scale_ct}
			X^{(k+1)} = 
			\sum_{\ell=0}^{K} \psi_\ell(L)\, X^{(k)}\, W_\ell.
		\end{align}

		\paragraph{Message Passing Networks:}
To determine which modifications are needed in order to turn message passing based networks scale continuous, we gain intuition from the the binary edge setting $(A_{uv} \in \{0,1\})$: There we may write the definition of the message function in (\ref{eq:mf_form}) equivalently simply as $m_{uv} = A_{uv} \cdot \phi(h_v^{(k)}, h_u^{(k)})$.
As we saw in Sections~\ref{sec:non-cont}~\&~\ref{sec:effective_prop_section}, such adjacency-focused propagation schemes result in networks that are not continuous across scales. We hence make the replacement $A \mapsto \psi(L)$ in our propagation scheme, and choose \(\operatorname{sum}\) as the aggregation function:
		\begin{align}
			m_{vu}^{(k)} &= [\psi(L)]_{vu} \cdot   \phi(h_v^{(k)}, h_u^{(k)}) \label{eq:sc_m}\\
			h_v^{(k+1)} &= \sum_{u \in G}  m_{vu} \label{eq:sc_agg}.
		\end{align}
	Additionally one initial pass $h_0 =  \Psi(L) X$ on the initial node features $X$ is performed, before feeding them into  our scale-continuous message passing framework (\ref{eq:sc_m})~\&~(\ref{eq:sc_agg}).

\paragraph{Scalability:}
At first glance, propagation along Laplace-transform matrices as in 
(\ref{eq:LT_integral}) may appear difficult to scale, since matrices such as 
$\psi(L) = e^{-tL}$ are typically dense. A naïve implementation would thus 
require sending $N^2$ messages over an $N$-node graph. 
This does however not pose a practical bottleneck: most entries 
are vanishingly small ($[\psi(L)]_{ij} \ll \delta$), decaying exponentially 
with (graph- or resistance-) distance $d(i,j)$ between nodes $i$ and $j$ 
\citep{Bauer2017SharpDGG}. Hence, propagation remains effectively sparse, 
and negligible edges can be safely discarded in practice. We explore this further in Appendix \ref{comp_cost}.

		\subsection{Theoretical Guarentees}
		Modifying spectral- and message-passing based graph neural networks exactly as outlined then indeed yields scale continuous networks, as we show next.
	Our main Theorem directly below establishes scale-continuity of the modified GNNs outlined in Section~\ref{subsec:sc_gnns}:
		\begin{theorem}\label{thm:quant_norm_est}
		Let $F$ and $\underline{F}$ be the latent embeddings generated for a graph $G$ its coarsified version $\underline{G}$ by a (spectral or message passing) network employing Laplace transform propagation, as outlined in Section~\ref{subsec:sc_gnns}. With $\{\Psi_i(L) = \int_0^\infty \hat{\psi}_i(t)e^{-tL}dt\}_{i \in I}$ the collection of all Laplace-transform propagation matrices used 
		within the network, we have $			\|F - \underline{F}\| \leq\label{eq:latent_conv} C_\theta \cdot \max\limits_{i} \left\{  \int_0^\infty   |\hat{\psi}_i(t)| \cdot \|e^{-tL_\omega} - J^\uparrow e^{- t\underline{L}} J^\downarrow\|  \right\}$.
			The constant $C_\theta$ depends on the learned parameters.
			 	\end{theorem}
	The proof of this Theorem is provided in Appendix \ref{app:graph_level_coarse_fine}. We note that whenever $G_\omega$ converges to $\underline{G}$ in the sense of heat kernels, by definition we have $\|e^{-tL_\omega} - J^\uparrow e^{- t\underline{L}} J^\downarrow\| \rightarrow 0$ for all $t>0$. Hence the integrand in Theorem \ref{thm:quant_norm_est} converges to zero.  Hence also $\|F_\omega - \underline{F}\|$ converges to zero. 
	
	Using the same  proof strategy, one may also show:
	\begin{Cor}\label{cor:main}
		In the setting of Theorem \ref{thm:quant_norm_est}, for two graphs $G_1, G_2$ on the same node set, one has $\|F_1 - F_2\| \leq C_\theta \cdot \max_i \{\int_0^\infty \hat{\psi}_i(t) \|e^{-tL_1} - e^{t L_2}\| dt \} $.
	\end{Cor}
The proof is provided in Appendix \ref{app:graph_level_stab}. 
Similar guarantees  also hold at the node level, (cf. Appendix \ref{app:node_level_stab}).

		\section{Numerical Investigation of Scale Continuity}

		\paragraph{Continuity across scales:}\label{finite_numerics}
		
Theorem \ref{thm:quant_norm_est} implies that networks employing
Laplace-transform propagation schemes are continuous as maps from the
space of graphs into their latent spaces. To numerically verify this, we
repeat the experiment of Section \ref{sec:non-cont} for models belonging
to this category (using resolvent and exponential matrices; cf.
Definition~\ref{def:LTF_def} in a spectral and a spatial setting). As is
evident from Fig. \ref{fig:collapse_graph_qm7_w_exp_orig}, latent
embeddings generated by models employing Laplace transform propagation
\textit{do} indeed converge ($\|F_\omega-\underline{F}\| \rightarrow
0$). Hence for the models constructed in Section \ref{subsec:sc_gnns}, we
have indeed verified the desired continuity across scales.
			\begin{figure}[H]
				\vspace{-2mm}
				\centering
						\vspace{-0mm}{
							\hspace*{-2mm}{
									\begin{overpic}[scale=0.41]{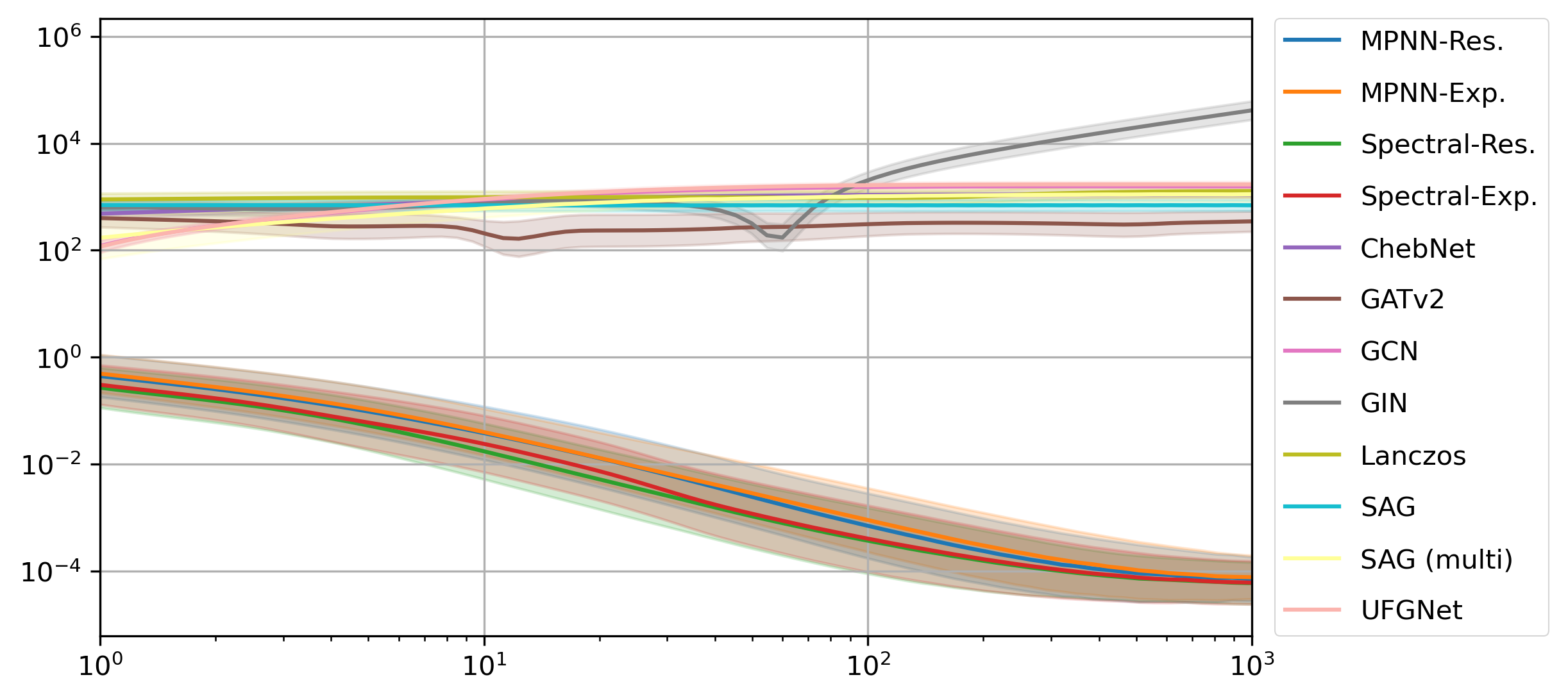}	
									\put(43,0){\small $\omega$}
								\end{overpic}
							}
						}
						\vspace{-4mm}
						\caption{Latent distance $\|F_\omega-\underline{F}\|$. Latent embeddings generated by Laplace transform based GNNs converge; others do not.
						}
						\vspace{-7mm}
						\noindent
						\label{fig:collapse_graph_qm7_w_exp_orig}
							\vspace{-1mm}
					\end{figure}
	\paragraph{Resulting Generalization Ability:}	
In Section \ref{sec:non-cont} we had identified lack of continuity as the obstruction to generalizing across scales. As verified above, graph neural networks based on Laplace-transform propagation \textit{are} continuous. Hence we expect them to map  similar graphs to similar latent embeddings. 
To verify this, we here repeat the generalization experiment of Tables~\ref{tab:qm7_baseline_model_results}~\&~\ref{tab:embed_diff} in  Section \ref{sec:failure_mode}.

	\begin{table}[h!]
			\vspace{-1mm}
		\small
		\centering
		\caption{
				Embedding difference $\|F - \underline{F}\|$ across resolution scales averaged over 5 runs (mean$\pm$std). Lower is better ($\downarrow$).
			}
		\label{tab:qm7_latent_diff}
		
		\vspace{-2mm}
		\setlength{\tabcolsep}{3pt}
		\renewcommand{\arraystretch}{1.1}

		\begin{tabular}{ccc}
			\toprule
			\textbf{Training $\rightarrow$ Inference:} 
			& High $\rightarrow$ Low 
			& Low $\rightarrow$ High \\
			\midrule
			
			GCN      & $20.3\pm{\scriptstyle 1.6}$    & $18.0\pm{\scriptstyle 0.6}$ \\
			GATv2    & $76.8\pm{\scriptstyle 52.4}$   & $42.5\pm{\scriptstyle 24.3}$ \\
			ChebNet  & $123.0\pm{\scriptstyle 3.8}$   & $158.5\pm{\scriptstyle 57.40}$ \\
			GIN      & $881.2\pm{\scriptstyle 268.7}$ & $2875.8\pm{\scriptstyle 1527.7}$ \\
			SAG      & $113.5\pm{\scriptstyle 10.3}$  & $43.3\pm{\scriptstyle 5.4}$ \\
			SAG-M    & $75.0\pm{\scriptstyle 15.1}$   & $61.7\pm{\scriptstyle 4.6}$ \\
			UFGNet   & $24.5\pm{\scriptstyle 0.3}$    & $26.8\pm{\scriptstyle 2.1}$ \\
			Lanczos  & $1286.1\pm{\scriptstyle 54.4}$ & $728.4\pm{\scriptstyle 230.0}$ \\
			\midrule
			
			MPNN$_{\text{exp.}}$        & $0.2\pm{\scriptstyle 0.1}$ & $0.2\pm{\scriptstyle 0.1}$ \\
			MPNN$_{\text{res.}}$        & $0.2\pm{\scriptstyle 0.1}$ & $0.2\pm{\scriptstyle 0.2}$ \\
			Spectral$_{\text{exp.}}$    & $0.3\pm{\scriptstyle 0.1}$ & $0.3\pm{\scriptstyle 0.1}$ \\
			Spectral$_{\text{res.}}$    & $0.3\pm{\scriptstyle 0.1}$ & $0.3\pm{\scriptstyle 0.3}$ \\
			\bottomrule
		\end{tabular}

	\end{table}
	
Comparing with standard GNNs 
we see that in cross-resolution settings the difference $\|F - \underline{F}\|$ of latent embeddings generated by  
methods employing global Laplacian propagation schemes are \emph{lower}
than those of standard graph learning methods by factors of order $10^2$
to $10^4$, indicating that the methods developed in Section
\ref{subsec:sc_gnns} indeed \emph{generalize}.

The small-to-negligible variations in latent representations then
translate to small-to-negligible variations in prediction performance: As we infer from Table \ref{tab:qm7_resolvent_results} below, MAEs generated by Laplace-Transform based GNNs in the cross resolution setting are essentially the same as those corresponding to same-resolution settings.
Hence GNNs based on Laplace Transform propagation schemes are indeed able to
\emph{generalize} across scales.

\begin{table}[h!]
	\small
	\centering
	\caption{QM7 regression. Mean Absolute Error (MAE $\downarrow$) in kcal/mol for training and inference at different resolutions scales.}
	\label{tab:qm7_resolvent_results}
	
	\vspace{-2mm}
	\setlength{\tabcolsep}{1pt}
	\renewcommand{\arraystretch}{1.1}

\begin{tabular}{lcccc}
	\toprule
	\textbf{Resolution} & \multicolumn{4}{c}{\textbf{MAE ($\downarrow$) on QM7 [kcal/mol] }} \\
	\cmidrule(lr){2-5}
	
	Training: & \multicolumn{2}{c}{\textbf{High}} & \multicolumn{2}{c}{\textbf{Low}} \\
	\cmidrule(lr){2-3} \cmidrule(lr){4-5}
	
	Inference: & \textbf{Low} & High & Low & \textbf{High} \\
	\midrule
	
	GCN  
	& \textbf{136.7$\pm{\scriptstyle 6.6}$} 
	& (63.6$\pm{\scriptstyle 1.3}$) 
	& (63.6$\pm{\scriptstyle 1.3}$)
	& \textbf{138.1$\pm{\scriptstyle 2.4}$} \\
	
	GATv2  
	& \textbf{423.5$\pm{\scriptstyle 337.1}$} 
	& (67.4$\pm{\scriptstyle 8.2}$) 
	& (59.7$\pm{\scriptstyle 2.7}$) 
	& \textbf{257.4$\pm{\scriptstyle 139.1}$} \\
	
	ChebNet   
	& \textbf{447.8$\pm{\scriptstyle 6.0}$} 
	& (66.7$\pm{\scriptstyle 1.4}$) 
	& (71.5$\pm{\scriptstyle 2.1}$) 
	& \textbf{158.7$\pm{\scriptstyle 57.4}$} \\
	
	GIN
	& \textbf{658.4$\pm{\scriptstyle 85.8}$} 
	& (17.$\pm{\scriptstyle 2.8}$) 
	& (38.3$\pm{\scriptstyle 21.6}$) 
	& \textbf{1835.4$\pm{\scriptstyle 925.8}$} \\
	
	SAG  
	& \textbf{589.7$\pm{\scriptstyle 44.9}$} 
	& (68.2$\pm{\scriptstyle 2.6}$) 
	& (107.9$\pm{\scriptstyle 1.1}$) 
	& \textbf{283.6$\pm{\scriptstyle 39.3}$} \\
	
	SAG-M  
	& \textbf{194.4$\pm{\scriptstyle 29.9}$}
	& (66.6$\pm{\scriptstyle 1.9}$) 
	& (77.8$\pm{\scriptstyle 6.0}$) 
	& \textbf{219.5$\pm{\scriptstyle 11.7}$} \\
	
	UFGNet  
	& \textbf{131.5$\pm{\scriptstyle 6.1}$} 
	& (62.4$\pm{\scriptstyle 0.7}$) 
	& (69.4$\pm{\scriptstyle 0.7}$) 
	& \textbf{148.1$\pm{\scriptstyle 6.3}$} \\
	
	Lanczos  
	& \textbf{938.4$\pm{\scriptstyle 2.5}$} 
	& (9.9$\pm{\scriptstyle 2.5}$) 
	& (88.2$\pm{\scriptstyle 2.7}$) 
	& \textbf{658.6$\pm{\scriptstyle 199.2}$} \\
	\midrule
	
	MPNN$_{\text{exp.}}$  
	& \textbf{17.4$\pm{\scriptstyle 3.4}$} 
	& (17.4$\pm{\scriptstyle 3.4}$) 
	& (17.9$\pm{\scriptstyle 5.2}$) 
	& \textbf{17.9$\pm{\scriptstyle 5.2}$} \\
	
	MPNN$_{\text{res.}}$  
	& \textbf{18.0$\pm{\scriptstyle 4.3}$} 
	& (18.0$\pm{\scriptstyle 4.3}$) 
	& (17.6$\pm{\scriptstyle 4.9}$) 
	& \textbf{17.6$\pm{\scriptstyle 4.9}$} \\
	
	Spectral$_{\text{exp.}}$  
	& \textbf{15.9$\pm{\scriptstyle 1.1}$} 
	& (15.9$\pm{\scriptstyle 1.1}$) 
	& (16.0$\pm{\scriptstyle 1.5}$) 
	& \textbf{16.0$\pm{\scriptstyle 1.5}$} \\
	
	Spectral$_{\text{res.}}$  
	& \textbf{17.2$\pm{\scriptstyle 3.1}$} 
	& (17.2$\pm{\scriptstyle 3.1}$) 
	& (15.8$\pm{\scriptstyle 1.4}$) 
	& \textbf{15.8$\pm{\scriptstyle 1.4}$} \\
	
	\bottomrule
\end{tabular}

\end{table}

\paragraph{Scale Continuity at the Node Level:}
We next consider  popular citation networks (cf. Appendix \ref{imba_geom} for details) where nodes correspond to scientific articles. Labels correspond to the academic discipline of the paper and an edge implies a citation. We then
					expand individual nodes into connected $k$-cliques (cf. Fig. \ref{fig:node_blowup}).  We might interpret this as dissecting each article into subsections, which reference each other. 					
						Both typical models and the Laplace-transform propagation based methods introduced in Section \ref{subsec:sc_gnns} were then trained on the same ($k$-fold expanded) train-set and asked to classify nodes in the ($k$-fold expanded) test-partition.

\begin{figure}[h!]
	\includegraphics[scale=.18, trim=0 -60 40 0]{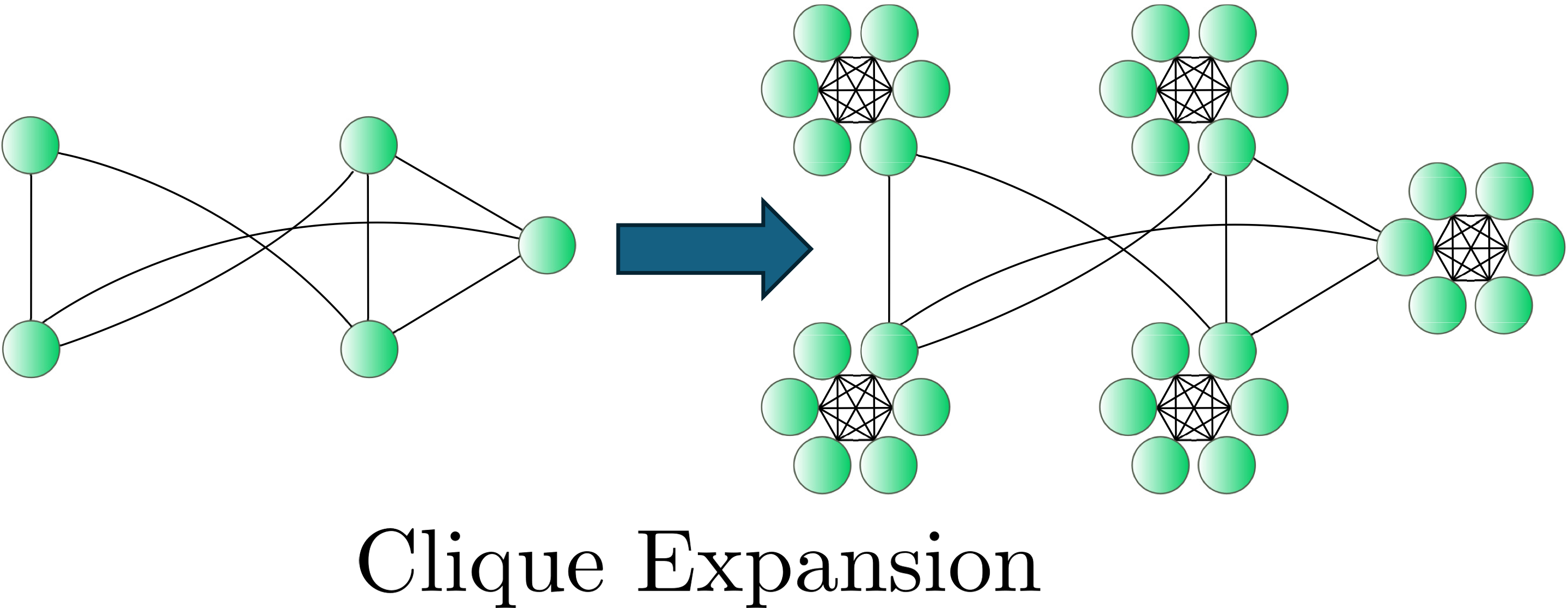}\hfill
	\includegraphics[scale=.15, trim=0 -60 -50 0]{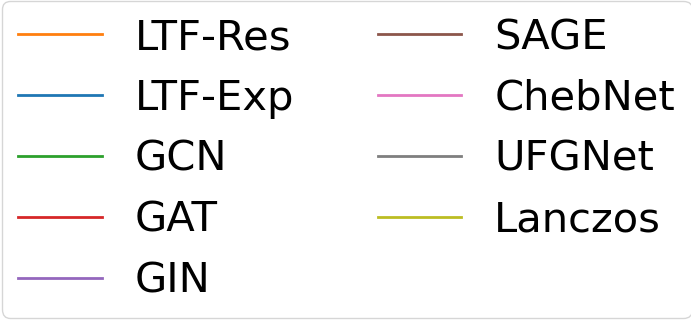}
	\includegraphics[scale=.15]{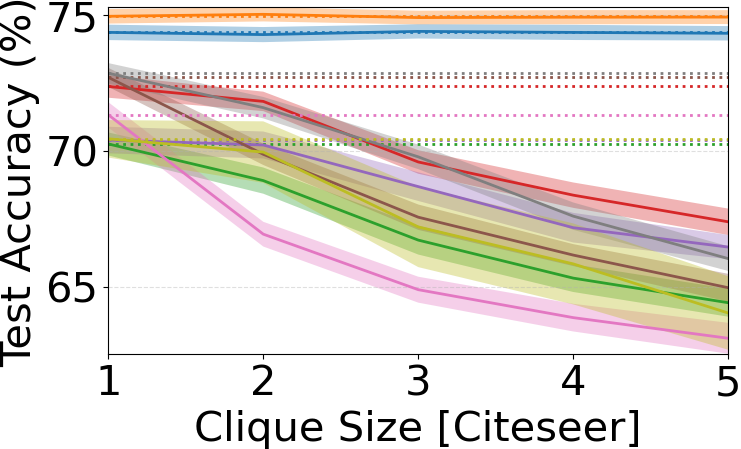}\hfill
	\includegraphics[scale=.15]{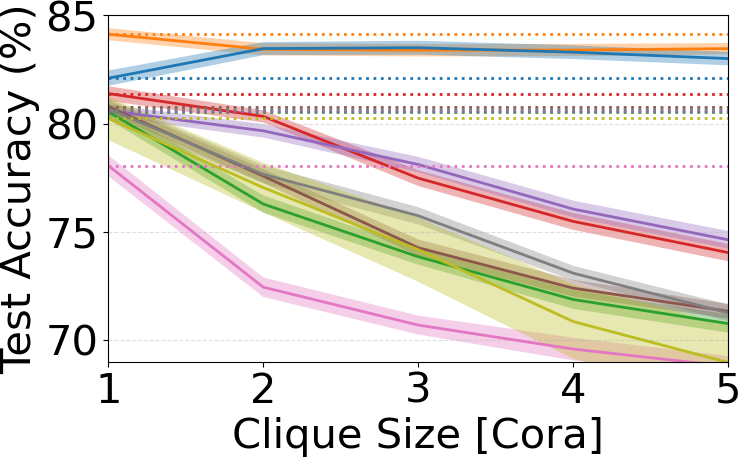}
	\vspace{-1mm}
	\captionof{figure}{Node-Classification-Accuracy ($\uparrow$) and uncertainty (for $100$ runs) vs. clique size. 
		}
	\label{fig:node_blowup}%
\end{figure}

The classification accuracies of methods not employing Laplace-transform propagation decrease significantly with increasing clique size (cf.  Fig. \ref{fig:node_blowup}). We can understand the underlying reason for this using GCN as an Example (cf. Appendix \ref{app:limitprop} for discussions on other methods): 
Inside a GCN-layer, a node feature matrix $X$ is updated   as $X\mapsto \hat{A} X W$, with the renormalized  adjacency matrix  $\hat{A}$ 
given as $	\hat{A}_{ij} \sim A_{ij}/\sqrt{d_id_j}$.
As the degree $d_i$ of each node increases (linearly) with increasing clique-size $k$, 
the message-strength $\hat{A}_{ij}$  between the respective cliques decreases as $\hat{A}_{ij} \sim 1/k$. Hence information propagation between the cliques becomes
disrupted as $k$ increases: GCNs propagation is more and more similar to a modified version 
of the original graph graph 
where edges \textit{between} cliques are removed ($G_2$ in Fig. \ref{fig:dumbbell_plot} (a) bottom).
\begin{figure}[h!]
	\vspace{-1mm}
	\centering
	
\begin{subfigure}{0.45\columnwidth}
	\centering
	\begin{overpic}[scale=.09, trim=0 0 0 0, clip]{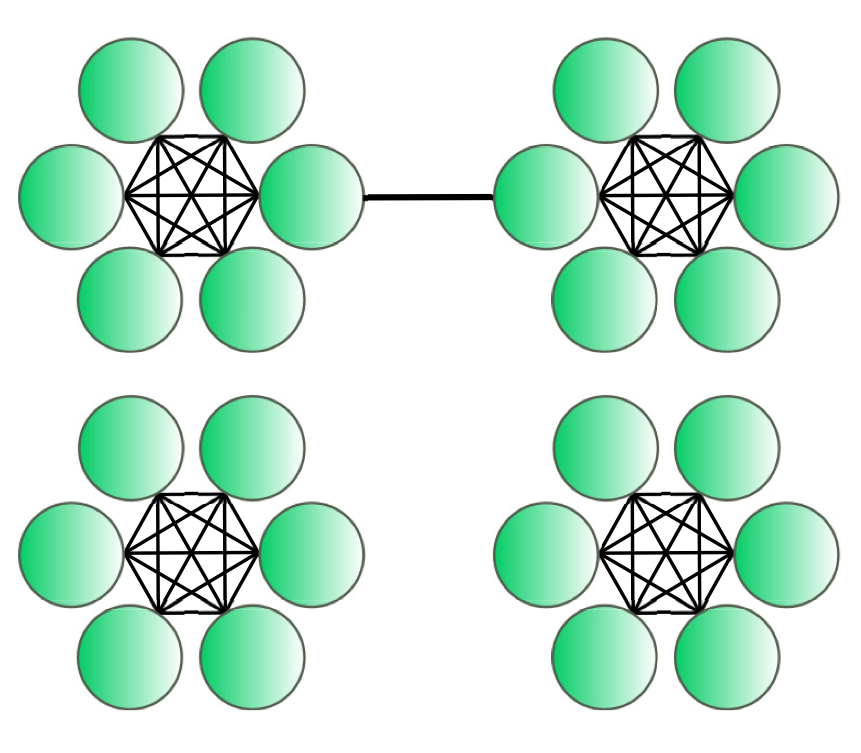}
		\put(105,60){\small $G_1$}
		\put(105,10){\small $G_2$}
	\end{overpic}
	\caption{Graph with ($G_1$) and without ($G_2$) bridge edge}
\end{subfigure}
	\hfill
	\begin{subfigure}{0.49\columnwidth}
		\centering
		\includegraphics[scale=.2, trim=0 30 0 0]{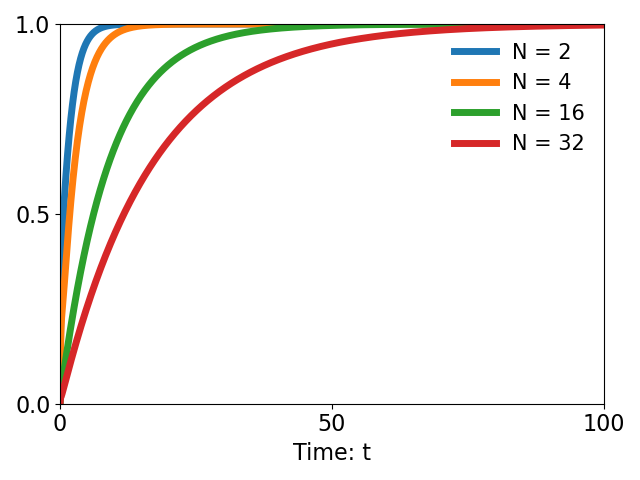}
		\caption{$\|e^{- t L_1} - e^{-t L_2}\|$ corresp. to graphs in (a) for clique-size $N$ }
	\end{subfigure}

	\caption{Graphs $G_1, G_2$ whose renormalized adjacency matrices are similar ($\hat{A}_1 \approx \hat{A}_2$), but whose heat kernels differ significantly.}
	\label{fig:dumbbell_plot}
		\vspace{-2mm}
\end{figure}

This is not the case for Laplace-transform propagation based networks (using either resolvent or exponential propagation matrices): As Corrolary \ref{cor:main} elucidates, similarity of graphs for such models is not determined through the renormalized adjacency as for GCN. Rather it is governed by the similarity the heat kernels corresponding to the graphs, (i.e. the size of $\|e^{-t L_1} - e^{-t L_2}\|$ for any $t > 0$). As we can deduce from  Fig. \ref{fig:dumbbell_plot} (b) this quantity does not tend to zero with increasing clique size.
Hence for the  models of Section \ref{subsec:sc_gnns}, propagation \emph{does not} happen along a disconnected effective graph. Thus they are able to propagate information also \textit{between} 
high connectivity areas and therefore are able to retain a high classification accuracy as the resolution scale is increased.

\paragraph{Scale continuity in the binary edge weight setting:}
In Section \ref{subsec:graph_to_kernel} we argued that from a diffusion perspective, as 
edge-weights within certain  sub-graphs of a graph $G$ tend to infinity, 
each such sub-graph effectively behaves as a single node in a coarse grained effective graph $\underline{G}$. This allows to continuously interpolate between $G$ and its coarse grained version $\underline{G}$ by scaling up edge weights within clusters. In order to be able to generalize across scales, we then modified standard GNNs to turn them continuous with respect to this limiting procedure in Section \ref{subsec:sc_gnns}. 

We next test continuity of 
GNNs under a related but distinct graph-limiting procedure: Conceptually we can think of increasing the weights of edges within clusters as a proxy for increasing the \emph{connectivity} within these subgraphs. 
To increase connectivity, rather than increasing the weights of a fixed \emph{number} of edges, we might instead also increase the number of edges within a cluster. To do this, we make use of stochastic block models (c.f. Appendix \ref{sbm_experiments}) (SBM) and vary the intra-cluster probability from $p = 0$ to $p = 1$.

\begin{figure}[h!]
	\includegraphics[scale=.18]{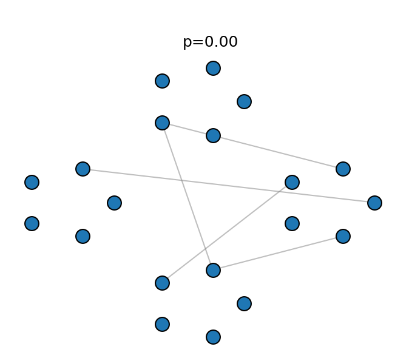}\hfill
	\includegraphics[scale=.18]{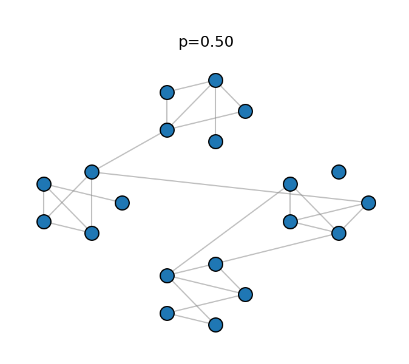}\hfill
	\includegraphics[scale=.18]{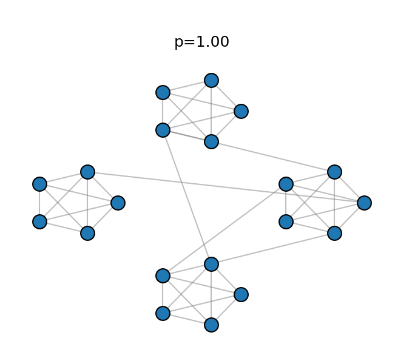}\hfill
	\captionof{figure}{Graphs drawn from an SBM as the intra-cluster connectivity is varied from $p=0$ to $p=1$.}
\label{fig:sbm_limit}%
	\vspace{-1mm}
\end{figure}

We then compare latent embeddings $F_p$ generated for graphs drawn from an SBM at intra-cluster connectivity $p$ with the latent embeddings $\underline{F}$ generated for a coarse grained version of this graph where clusters are aggregated to single nodes.
As is evident from Fig. \ref{fig:sbm_decay} below, latent embeddings generated by Laplace transform based models increasingly approximate the latent embedding $\underline{F}$ as $p \rightarrow 1$, while the same is not true for standard GNNs. In latent space, this continuity allows GNNs of Section \ref{subsec:sc_gnns} to embedd clique graphs close to graphs where each such clique is represented as a supernode (cf. Table \ref{tab:sbm_lat_diff})).\\

\begin{figure}[h!]
	\vspace{-1mm}
	\centering
	\begin{minipage}[t]{0.52\columnwidth}
		\centering
    \begin{overpic}[scale=.325, trim=35 0 0 0]{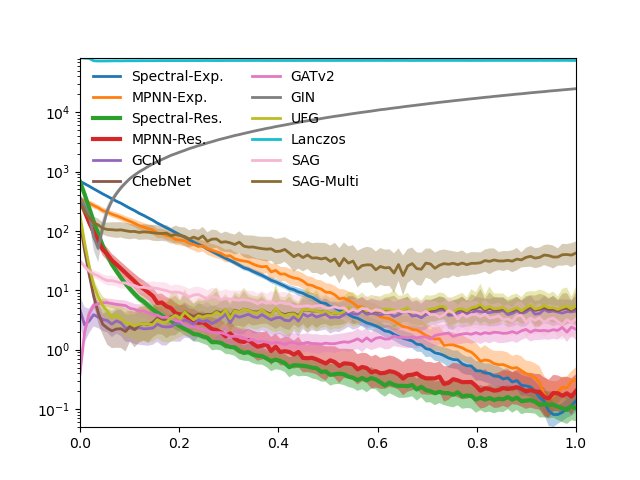}
	\put(46,2){\tiny $p$}  %
\end{overpic}
		\captionof{figure}{$\mathbb{E}[\| F_p   -  \underline{F}\|]$ }
		\label{fig:sbm_decay}
	\end{minipage}\hfill
	\begin{minipage}[t]{0.44\columnwidth}
		\centering
		\tiny
		\vspace{-40mm}
		\captionof{table}{Latent distances}
		\label{tab:sbm_lat_diff}
		
		\vspace{1mm}
		\setlength{\tabcolsep}{3pt}
		\renewcommand{\arraystretch}{1.1}

\begin{tabular}{cc}
	\toprule
	\textbf{Model} & \textbf{$\|F_{p=1} - \underline{F}\|$} \\
	\midrule
	
	GCN      & $4.39 \pm{\scriptstyle 2.95}$ \\
	GATv2    & $2.15 \pm{\scriptstyle 1.02}$ \\
	ChebNet  & $4.67 \pm{\scriptstyle 2.76}$ \\
	GIN      & $24877.22 \pm{\scriptstyle 279.41}$ \\
	SAG      & $3.20 \pm{\scriptstyle 2.51}$ \\
	SAG-M    & $43.46 \pm{\scriptstyle 22.56}$ \\
	UFGNet   & $5.02 \pm{\scriptstyle 3.75}$ \\
	Lanczos  & $74203.45 \pm{\scriptstyle 576.27}$ \\
	\midrule
	
	Spectral$_{\text{res.}}$ & $0.12 \pm{\scriptstyle 0.07}$ \\
	Spectral$_{\text{exp.}}$ & $0.13 \pm{\scriptstyle 0.08}$ \\
	MPNN$_{\text{res.}}$.    & $0.22 \pm{\scriptstyle 0.18}$ \\
	MPNN$_{\text{exp.}}$.    & $0.30 \pm{\scriptstyle 0.17}$ \\
	
	\bottomrule
\end{tabular}

	\end{minipage}
		\vspace{-1mm}
\end{figure}

\paragraph{Scale Continuity in the Continuum Limit:}
Finally, we consider the setting where graphs discretize a continuous manifold. In this regime, we require that as we increase the resolution scale, model outputs converge to those that would be produced by a model operating directly on the underlying continuous manifold.
In particular we want 
 the correlation between two nodes representing (the same two) points on the manifold to stabilize, as the resolution is increased.

To show that this requirement is not fulfilled by standard GNNs, we feed standard models with a dirac impulse at a fixed point (located at $\phi = 0$, $\theta = 0$)  on the torus $\mathbb{T} = [0,2\pi)^2$.
We then investigate the output the GNNs generate at a test node, which is positioned at angles $\phi = \pi/2$, $\theta = 0$.

	\begin{figure}[h!]
	\centering
	\includegraphics[scale=0.17]{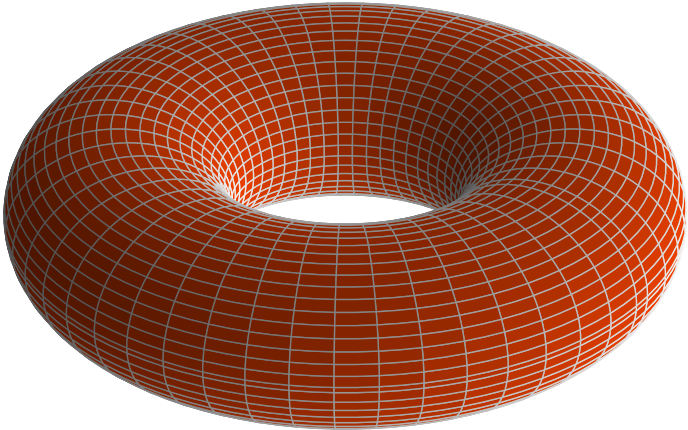}
	\ \ \ \ \
	\includegraphics[scale=0.17]{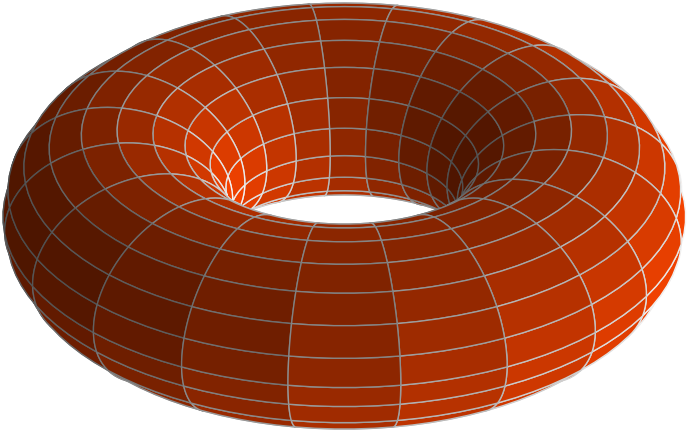}
	\captionof{figure}{Graphs approximating torus-manifold
		(two resolutions).} 
	\label{fig:torus_plot}
	\vspace{-2mm}
\end{figure}

 As evident from Fig. \ref{fig:impulse_response} (a), this requirement is not fulfilled for standard GNNs, as  as neithers models' response stabilizes: For most GNNs the impuls response drops to zero, while for Lanczos it fluctuates as the mesh resolution (i.e. number of nodes) is increased. 

\begin{figure}[h!]
	\centering
	
	\begin{subfigure}{0.45\columnwidth}
		\centering
		 \begin{overpic}[scale=.25, trim= 20 0 0 0]{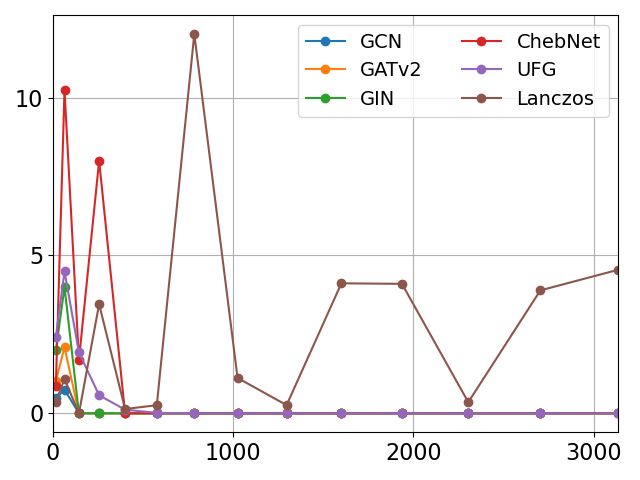}
			\put(32,-1){\tiny Number of Nodes}  %
		\end{overpic}
		\caption{Standard GNN impulse\\ response }
	\end{subfigure}
	\hfill
	\begin{subfigure}{0.45\columnwidth}
		\centering
		 \begin{overpic}[scale=.25, trim= 22 0 0 0]{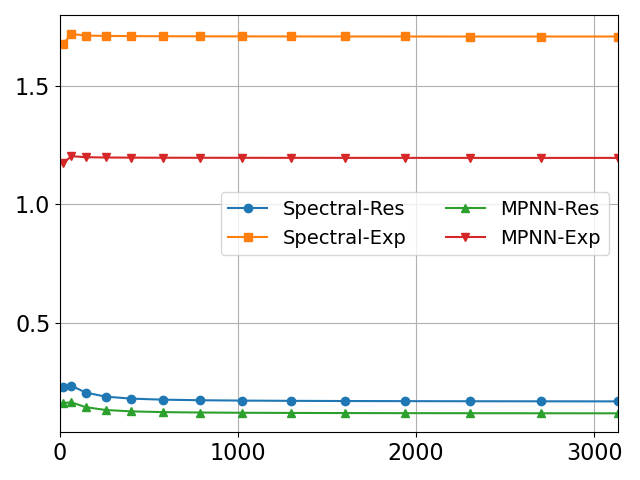}
			\put(33,-1){\tiny Number of Nodes}  %
		\end{overpic}
		\caption{Laplace-Transform GNN\\ impulse response}
	\end{subfigure}
	
	\caption{Impulse response}
	\label{fig:impulse_response}
\end{figure}

In contrast to that, the impulse response for the Laplace-transform based methods of Section \ref{subsec:sc_gnns} stays consistent as the mesh resolution is varied.
To show that this persists also at graph level, we would ideally want to show that the graph level latent embeddings $F_N$ generated for a mesh on $N$ nodes converge to the latent embedding corresponding to the continuous manifold as $N \rightarrow \infty$. 
In Appendix \ref{app:torus_exp_description} we detail that we indeed have actual convergence $F_N \rightarrow F$ towards a limit embedding $F$ corresponding to the underlying continuous manifold. Since $F$ is impossible to compute numerically, we verify the condition  $\|F_N - F_{2N}\| \xrightarrow{N \rightarrow \infty} 0$ (necessary for convergence) 
in  Fig. \ref{fig:torus_graph_seq} below instead.

	\begin{figure}[h!]
	\centering
	\begin{overpic}[scale=.25]{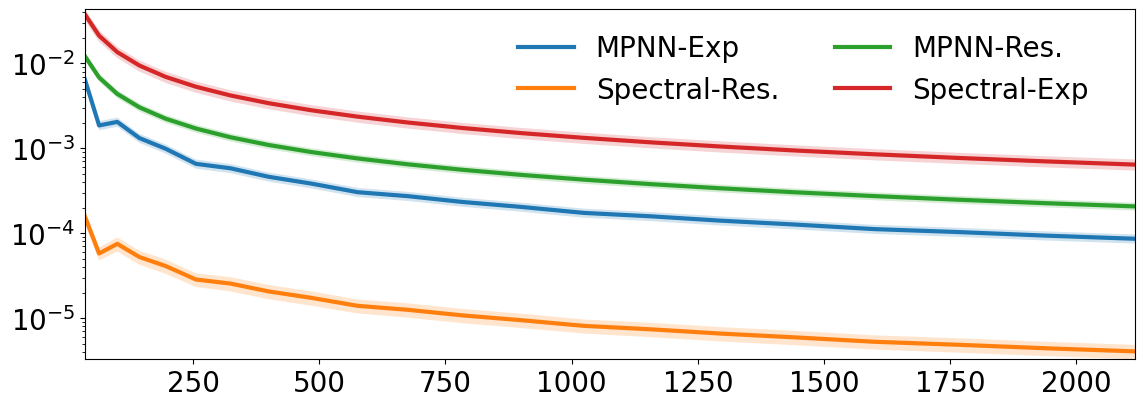}
		\put(45,-2){\tiny Number of Nodes $N$}  %
	\end{overpic}
			\caption{Graph level latent distance  $\|F_{N} - F_{2N}\| $ as $N$ increases.}
	\label{fig:torus_graph_seq}
	\vspace{-1mm}
\end{figure}

\section{Discussion}

Our paper analyzed the discontinuity of existing GNNs across scales. We
found the underlying obstruction to originate in commonly-used propagation
schemes. We derived modifications to turn GNNs continuous, and
showed that these modified models based on Laplace-Transform propagation can
indeed \emph{consistently} incorporate varying scales.

\section*{Acknowledgements}
This work has received funding from the Swiss State Secretariat for Education, Research, and Innovation~(SERI).
C.K. thanks Madeline Navarro for helpful discussions during the rebuttal phase. M.B. is partially supported by the EPSRC Turing AI World-Leading Research Fellowship No. EP/X040062/1 and EPSRC AI Hub No. EP/Y028872/1.

\section*{Impact Statement}
This paper presents work whose goal is to advance the field of Machine
Learning. There are many potential societal consequences of our work, none
which we feel must be specifically highlighted here.

\bibliography{bibfile}
\bibliographystyle{icml2026}

\newpage
\appendix
\onecolumn

\section{Discussion of Tasks and Graph Structures for which Scale Continuity is important}
When GNNs are tasked to strictly distinguish non-isomorphic graphs the exact articulation of a graph matters (e.g. program graphs, electrical circuit graphs, etc.).
In other settings, multiple graphs may describe the same underlying object (graphs discretizing the same manifold, sampled from the same distribution, or describing the same underlying noisy (social-, or natural-) science dataset.
 Here we want models to effectively embed the underlying object and not overfit on exact fine-print articulation of graphs describing said object.
 
 A notion of distance respecting this intuition arises from comparing heat kernels as $\|e^{-t L_1 - e^{-t L_2}}\|$. 
 Using exponentials leads to an (exponential) suppression of large eigenvalues of the Laplacian $L$. The suppressed information encoded into these large eigenvalues corresponds to fine structure details of the graph (c.f. e.g.\citep{Chung_1997}) which the model is supposed to discount. 
 When comparing graphs on different numbers of nodes, we have to align them via projection $J^\downarrow$ and interpolation $J^\uparrow$ operators determining how to traverse between these graphs.
 
The quantity $\|e^{-t L_1} - J^\uparrow e^{-t L_2} J_\downarrow \|$ (c.f. e.g. eq. (\ref{eq:orig_exp_conv_result_II}) \& Fig. \ref{fig:the_next_generation}) measures similarity of coarse structures in $G_1, G_2$, after alignment. If this difference is small, coarse structures are both aligned and similar. Only then should models generate similar embeddings.

\section{Discussion of Related Work}\label{app:related work}

In this section we provide further details on existing approaches to estimating variatons in latent embeddings generated by graph neural networks:

\paragraph{Graphon Neural Networks and the Transferability of Graph Neural Networks \citep{DBLP:conf/nips/RuizCR20}:}

This seminal work explores the theoretical underpinnings of Graph Neural Networks (GNNs) in the context of graphons, a mathematical generalization of graphs to large-scale, continuous structures. The paper establishes a connection between GNNs and graphons, providing insights into the behavior of GNNs on large, dense graphs  ($|\mathcal{E}|$ is of $\mathcal{O}(N^2)$, with $N$ the number of nodes \citep{le2023limits}) by modeling these graphs as graphons. This framework helps understand how GNNs operate in the limit of large graphs and illuminates their potential to generalize across different graph structures in this realm. A central focus of the paper is the transferability of GNNs—specifically, their ability to perform well on large graphs that may differ in size or topology from those seen during training. Transferability errors between graphs discretizing the same graphon are established to be of $\mathcal{O}(N^{-\frac12})$, with $N$ the minimum number of nodes.
Since this approach is aplicable only when graphs are large and $|\mathcal{E}| = \mathcal{O}(N^2)$, it e.g. does not cover the setting of Table \ref{tab:qm7_baseline_model_results}.

\paragraph{Transferability of Graph Neural Networks: an Extended Graphon Approach \citep{DBLP:journals/corr/abs-2109-10096}:}

This work is in spirit similar to \citep{DBLP:conf/nips/RuizCR20} whose results it extends from considering the adjacency matrix  as the graph shift operator to more general graph shift operators and from considering only polynomial filters to allowing for general continuous filter functions. Similar limitations as for \citep{DBLP:conf/nips/RuizCR20} apply.

\paragraph{Limits, approximation and size transferability for {GNN}s on sparse graphs via graphops  \citep{le2023limits}:}

In contrast to approaches using graphons, which focus on large \textit{dense} graphs, this paper instead focuses on transferability on \textit{sparse} graphs ($ |\mathcal{E}| = \mathcal{O}(N)$) making use of the concept of Graphops, mathematical operators that can be used to model how GNNs behave on large sparse graphs.  The authors also explore how GNNs can transfer learned representations from smaller, sparse graphs to larger ones, and vice versa. By leveraging the Graphop framework, the paper formalizes conditions for successful transferability between graphs of varying sizes.
Due to the necessity for graphs to be large (with edges $ |\mathcal{E}| = \mathcal{O}(N)$) it also does not apply e.g. to the setting considered in Section \ref{sec:failure_mode}.

\paragraph{On Local Distributions in Graph Signal Processing  \citep{RoddenberryGBS22}:}

This work is rooted in the field of graph signal procesing (GSP) and puts a particular emphasis on the transferability of GSP techniques across different graph structures. The paper focuses on the concept of graphings, which are a probabilistic framework for representing large  sparse graphs and their underlying structures.
By modeling large graphs through graphings, the authors provide a framework that makes it possible to generalize local distributions and signal processing tasks across different graphs. This work similarly is only applicable to large sparse graphs.

\paragraph{Graph Convolutional Neural Networks via Scattering \citep{Zou_2020}}

This work provides a different perspective on spectral GNNs by connecting them to scattering transforms \citep{Mallat2012group, Understanding, GraphScatteringBeyond}, a concept from signal processing. The authors demonstrate that spectral GNNs can be interpreted as a discrete graph counterpart of scattering transforms, which involve multi-scale wavelet-like operations that capture hierarchical information across different levels of graph structure. 
A key focus of the paper is the stability of networks when viewed through the scattering framework. The authors argue that scattering transforms offer a more stable approach to graph signal processing compared to traditional networks, especially in the presence of noisy or incomplete graph data. 
Derived single filter transferability results  depend on spectral properties of the utilized Laplacians on the respective graphs.
The focus of this work is on the setting where $\|L - \tilde{L}\|$ is small. This is distinct from the setting of similar heat kernels, which is more general.

\paragraph{Stability to Deformations of Manifold Filters and Manifold Neural Networks \citep{DBLP:journals/tsp/WangRR24}}:

This work explores the theoretical foundation of manifold filters and manifold neural networks (MNNs). 
Similarly to the filters analyzed in the present work, manifold filters are defined in terms of Laplace transforms.
By framing spectral graph neural networks  as discrete approximations of MNNs, the authors analyze conditions under which MNNs remain stable under smooth deformations of the manifold.
Stability is shown to depend on specific spectral properties of the filter functions, including Lipschitz continuity and integral Lipschitz continuity, which control the trade-off between robustness and frequency discriminability. The paper establishes that filters meeting these conditions can generalize effectively to new manifolds by adapting to changes in the Laplace-Beltrami operator's spectrum.
More techically,  filters are bounded as $|\psi(L) - \psi(\tilde{L})| \leq K \|L -  \tilde{L}\|$. In Theorem 2 absolute perturbations are considered ($\tilde{L} = L + A$), in  Theorem 3 relative perturbations are considered ($\tilde{L} = L + EL $). In both cases the conditions on spectrum and filter functions stem from the fact that Lipschitz-ness does not directly translate to operator Lipschitz-ness when measured in spectral norm (see e.g. \citet{Wihler} for a discussion). In contrast, our work does not primarily focus on (large) graphs approximating a common manifold. We only consider this setting as \emph{one} example. In comparison to \citep{DBLP:journals/tsp/WangRR24}, the results we develop for this particular example (c.f. Appendix \ref{app:torus_exp_description}) need no assumptions on band-limitedness. Additionally our results also apply to message passing based networks, in contrast to those of \citep{DBLP:journals/tsp/WangRR24} which apply to spectral networks.

\paragraph{Geometric Graph Filters and Neural Networks: Limit Properties and Discriminability Trade-offs \citep{DBLP:journals/tsp/WangRR24a}:} 
Here instead of measuring the linear norm difference $\|LP -  \mathcal{L}P\|$ between a graph Laplacian $L$ and a manifold Laplacian $\mathcal{L}$ (which generically would be infinite as $\mathcal{L}$ is an unbounded operator), the difference of the action of these operators on eigenfunctions ($\| LP\phi -  \mathcal{L}P\phi  \|$) is considered. After a triangle inequality argument, one term that has to be bounded in order to bound the difference in filter outputs is $\|\phi_i^n - \phi_i\|$ of the $i^{th}$ eigenfunction and eigenvector respectively. The fidelity of this approximation depends on spectral separation properties (cf. Theorem 4 ibid.), which hence leads to the requirement that the spectrum be $\alpha$-separated. This requirement can thus be considered an artifact of considering the linear approximation $\|\phi_i^n - \phi_i\|$ for each eigenfunction. In contrast, in our approach when applied to the specific example of discretized manifolds (cf. Appendix \ref{app:torus_exp_description}) the notion of approximation of the Laplacian on the underlying manifold is different. We bound the quantity $\| J^\uparrow e^{-tL} J^\downarrow- e^{-t\Delta}  \|$ instead. Hence we do not need to bound differences between individual eigenfunctions and eigenvectors and hence avoid dependencies on spectral separations. Additionally, our results also cover message passing based networks, and are not restricted to spectral networks.

\paragraph{Transferability of Spectral Graph Convolutional Neural Networks \citep{DBLP:journals/corr/abs-1907-12972}:}

As one of the earliest works challenging the then prevailing belief that spectral methods are not transferable, this work was among the first to present theoretical proofs and experimental evidence to demonstrate that these methods can generalize effectively under certain conditions.
The key contribution is a theoretical framework in which transferability depends on how well graphs approximate a shared underlying continuous domain, such as a topological space or metric-measure space. 
The study also develops sufficient conditions for achieving low transferability errors, demonstrating that spectral gnns can perform consistently across graphs with varying sizes, topologies, and dimensions, provided the graphs discretize the same continuous domain.
As in our work, filters here are only required  to be bounded and Lipschitz continuous (cf. Theorem 17 ibid.). However, signals are assumed to be bandlimted.  In the specific example setting of graphs discretizing manifolds, we avoid Levie's growth of the stability constant with the number of considered eigenvalues (cf. the discussion towards the end of page 12 ibid.) by avoiding approximations of individual eigenfunctions and instead approximating the bounded operator $e^{-t\Delta} $ directly. Additionally, our results also cover message passing based networks, and are applicable also outside the setting of graphs discretizing ambient spaces.

\paragraph{Diffusion Scattering Transforms on Graphs \citep{gama2018diffusion}:}

This work emphasizes the stability of scattering-based representations against perturbations in graph topology and reindexing. The framework introduces diffusion scattering transforms that leverage diffusion operators to capture multi-scale hierarchical features of graph signals.
The authors focus on ensuring that the transforms are robust to changes in graph structure, such as modifications to edge weights or topology. Stability is achieved through the use of diffusion wavelets, which provide a principled way to construct graph filters that are invariant to local perturbations while retaining sensitivity to meaningful global graph features. The stability analysis demonstrates that the scattering transform bounds the impact of graph perturbations in terms of the changes they induce in the graph Laplacian's spectrum, ensuring reliable performance across varied graph inputs.
Results derived in this work do not apply to our settings, as detecting graph variations via the difference $\|\mathcal{L} - \tilde{\mathcal{L}}\|$ as considered in \citep{gama2018diffusion} is insensitive to the notion of variation and convergence we consider here (c.f. Section \ref{subsec:graph_to_kernel}). Additionally the use of the normalized Laplacian in  \cite{gama2018diffusion} precludes detection of the convergence behaviour of Section \ref{subsec:graph_to_kernel}.

\paragraph{Stability Properties of Graph Neural Networks \citep{DBLP:journals/tsp/GamaBR20}:}

This paper investigates the stability properties of GNNs to perturbations in the underlying graph structure. The authors analyze how small changes in graph topology— such as modifications to edge weights, addition or deletion of edges, or reindexing of nodes—affect the outputs of spectral GNNs. 
The paper develops a rigorous mathematical framework to assess the stability of GNNs using tools from spectral graph theory. It establishes that GNNs are stable to localized perturbations in the graph topology, with the degree of stability depending on the spectral properties of the graph filters used within the network. Specifically, it is shown that GNNs exhibit a trade-off between stability and discriminability: filters that are more stable to perturbations may sacrifice sensitivity to high-frequency information, which can limit their ability to differentiate fine-grained graph structures.
Here as well, Lipschitz type arguments are  being used (see e.g. the assumptions of Theorem 1) to establish single filter transferability. Since scalar Lipschitzness does not translate to operator Lipschitzness under spectral norm,  additional restrictions on spectrum and filter functions need to be hence imposed.
Following this, the authors  highlight the importance of filter design in achieving a balance between robustness and expressivity. Filters that adhere to conditions such as Lipschitz continuity or integral Lipschitz continuity are particularly effective in maintaining stability while preserving key graph features.
Results derived in this work do not apply to our settings, as detecting graph variations via the difference $\|L - \tilde{L}\|$  is insensitive to the notion of variation and convergence we consider here (c.f. Section \ref{subsec:graph_to_kernel}).

\paragraph{Limitless transferability for graph convolutional Networks \citep{limitless}:}

This work studies stability- and transferability properties of spectral graph neural networks, with a particular focus on directed graphs. In spirit, it is the closest to our work here, as one of the main class of filters it investigates is the class of resolvent based spectral filter functions which constitute an example of the more  general class of Laplace transform filters considered in this present work.
Our work here is however far more general, considering general Laplace transform based propagations schemes, including not only spectral-, but also message passing based networks. Additionally we develop a rich and detailed theory and provide numerical investigations and experimental evaluations in diversex settings.

\paragraph{Previous versions of this present work:}
Earlier versions and preliminary aspects of this research were presented at various non-archival workshops \citep{koke2023resolvnet, koke2024transferability, koke2025on, koke2025graphscale,  koke2025incorporating, koke2026scale}. The various aspects discussed there are subsumed and extended by this present work.

\section{Laplace Transform Matrices}\label{app:ltf_matrices}
In this section we provide an overview of the concept of Laplace transforms. We begin with a recapitulation of complex measures.
\subsection{Complex measures on $\mathds{R}_{\geq 0}$ and their Theory of Integration}

As reference for this section \cite{tao2013introduction} might serve. 

In mathematics, a measure is a formal generalization of concepts such as length, area and volume.
We are interested in assigning a generalized notion of length (or mass) to subsets of the real half-line 
\begin{equation}
	\mathds{R}_{\geq 0} = [0,\infty).
\end{equation}
The set will turn out to be a so called $\sigma$-Algebra; i.e. a set $\Sigma$ of sets for which
\begin{itemize}
	\item $\emptyset, \mathds{R}_{\geq 0} \in \Sigma$
	\item $A,B \in \sigma \Rightarrow A\cap B \in \Sigma$
	\item $A,B\in \Sigma \Rightarrow A \setminus B \in \Sigma $
	\item $A,B\in \Sigma \Rightarrow A \cup B \in \Sigma $.
\end{itemize}
We now take $\Sigma_{\mathds{R}_{\geq 0}}$ to be the smallest such set of sets $\Sigma$ that contains all open intervals.

A complex measure then is a set-function that assigns to each set in $\Sigma_{\mathds{R}_{\geq 0}}$ a complex number in a certain way:

\begin{Def}
	A complex measure $\mu$ on $\mathds{R}_{\geq 0}$ is a complex valued function $\mu: \Sigma_{\mathds{R}_{\geq 0}} \rightarrow \mathds{C}$ satisfying
	\begin{equation}
		\mu\left( \bigcup\limits_n  A_n \right) = \sum\limits_n 	\mu\left(   A_n \right)
	\end{equation}
	for any countable (potentially infinite) collection of sets in $\Sigma_{\mathds{R}_{\geq 0}}$ which are pairwise disjoint.
\end{Def}

Let us provide some examples:

\begin{Ex}
	The prototypical example of a measure is the standard Lebesgue measure that assigns to any interval $(a,b)$ the length $\mu_{\text{Leb}}((a,b)) = |a-b|$ ($a,b \in \mathds{R}_{\geq 0} $).
\end{Ex}
\begin{Ex}
	Alternatively, we might consider the Dirac measure $\mu_{\delta_{t_0}}$, which assigns the value $\mu_{\delta_{t_0}}((a,b)) = 1$ to any interval $(a,b)$ containing $t_0$ (i.e. $t_0 \in (a,b)$). Otherwise it assigns the value $\mu_{\delta_{t_0}}((a,b)) = 0$ if $t_0 \notin (a,b)$.
\end{Ex}
\begin{Ex}
	Every integrable function $\hat{\psi}: \mathds{R}_{\geq 0} \rightarrow \mathds{C}$ defines a complex measure via $\mu_{\hat{\psi}}((a,b)) = \int_a^b \hat{\psi}(t) dt$.
\end{Ex}

Any given measure on $\mathds{R}_{\geq 0}$ defines a unique way of integrating (known as Lebesgue integration) a function $f$ defined on $\mathds{R}_{\geq 0}$.
This proceeds by approximating any function $f$ via a weighted sequence of indicator functions (with $A \in \Sigma_{\mathds{R}_{\geq 0}}$ a set)

\begin{equation}
	\chi_A(t) = \begin{cases} 
		1 &; t \in A\\
		0 &; t \notin A
	\end{cases}.
\end{equation}

as 
\begin{equation}
	f(t) \approx f_n(t) := \sum_k a^n_k \chi_{A_k}(t).
\end{equation}
with $a_k \in \mathds{C}$. For these functions, one then sets
\begin{equation}
	\int_{\mathds{R}_{\geq 0}} f_n d\mu \equiv \sum_k a^n_k\cdot \mu(A_k).
\end{equation}
Since we have $\lim_{n \rightarrow \infty} f_n = f$, one then simply sets
\begin{equation}
	\int_{\mathds{R}_{\geq 0}} f d\mu \equiv \lim_{n\rightarrow\infty} \int_{\mathds{R}_{\geq 0}} f_n d\mu.
\end{equation}

\begin{Ex}
	For the prototypical example of the standard Lebesgue measure, this process simply yields
	\begin{equation}
		\int_{\mathds{R}_{\geq 0}} f(t) d\mu_{\text{Leb}}(t) =	\int_0^\infty f(t) dt.
	\end{equation}
	
\end{Ex}
\begin{Ex}
	For the Dirac measure $\mu_{\delta_{t_0}}$, the above process yields
	\begin{equation}
		\int_{\mathds{R}_{\geq 0}} f(t) d\mu_{\delta_{t_0}}(t) = f(t_0)
	\end{equation}
\end{Ex}
\begin{Ex}
	For measures arising from  integrable functions  $\hat{\psi}: \mathds{R}_{\geq 0} \rightarrow \mathds{C}$ as $\mu_{\hat{\psi}}((a,b)) = \int_a^b \hat{\psi}(t) dt$, we find
	\begin{equation}
		\int_{\mathds{R}_{\geq 0}} f(t) d\mu_{\hat{\psi}} = \int_0^\infty \hat{\psi}(t) f(t) dt.
	\end{equation}
\end{Ex}

\subsection{Laplace Transforms}\label{laplace_transforms_app}
We say  complex valued measure $\mu$ is finite if we have
\begin{equation}
	\int_{\mathds{R}_{\geq 0}} d|\mu|(t) < \infty.
\end{equation}
Here the measure $|\mu|$ arises from the original measure $\mu$ via 
\begin{equation}
	|\mu|((a,b)) \equiv |\mu((a,b))|.
\end{equation}
For any such finite measure $\mu$  we may define its Laplace transform as 
\begin{equation}
	\psi_\mu(z) := \int_{\mathds{R}_{\geq 0}} e^{-tz} d\mu(t).
\end{equation}
This function $f_\mu$ is well defined for $z$ in the right hemisphere
\begin{equation}
	\mathds{C}_R := \{z \in \mathds{C} : \text{Re}(z) \geq 0\}.
\end{equation}
of the complex plane $\mathds{C}$, since there we have
\begin{align}
	|\psi_\mu(z)| & = \left|\int_{\mathds{R}_{\geq 0}} e^{-tz} d\mu(t)\right|\\
	& \leq \int_{\mathds{R}_{\geq 0}} |e^{-tz}|d|\mu|(t)\\
	& \leq \int_{\mathds{R}_{\geq 0}} d|\mu|(t) < \infty.
\end{align}

\begin{Ex}
	For the  Dirac measure $\mu_{\delta_{t_0}}$, we have
	\begin{equation}
		\psi_{\mu_{\delta_{t_0}}}(z) = e^{-t_0 z}.
	\end{equation}
\end{Ex}

\begin{Ex}
	For any integrable function $\hat{\psi}$, we have
	\begin{equation}
		\psi(z) \equiv 	\int_{\mathds{R}_{\geq 0}} e^{-tz} d\mu_{\hat{\psi}} = \int_0^\infty \hat{\psi}(t) e^{-tz} dt.
	\end{equation}
\end{Ex}
More specifically, if the integrable function  is given as 	$\hat{\psi}_k := (-t)^{k-1}e^{-\lambda t}$ (with $\text{Re}(\lambda) > 0$), then $\psi_k(z) = (z+\lambda)^{-k}$:
\begin{Ex}
	If		$\hat{\psi}_k := (-t)^{k-1}e^{-\lambda t}$ yields $\psi_k(z) = (z+\lambda)^{-k}$, then
	\begin{equation}
		\psi_k(z) = (z + \lambda)^{-k}.
	\end{equation}
	For $k=1$, this can be seen from
	\begin{equation}
		\int_0^\infty e^{-tz}e^{-\lambda t} dt = -\frac{1}{z+\lambda} e^{-(z+\lambda)} \bigg|_{0}^{\infty}.
	\end{equation}
	For $k > 1$, the claim follows from differentiating the above expression with respect to $z$
	Note that the functions $\psi_k(z) = (z+\lambda)^{-k}$ are also defined if $\text{Re}(z)\leq 0$, as long as $z \neq -\lambda$.
\end{Ex}

Using the function $\psi_k$ of the examples above, a wide class of functions may be parametrized
\begin{Th}
	Let $f:\mathds{R}_{\geq 0} \rightarrow 0$ be any function with $\lim\limits_{x\rightarrow\infty} f(x) = 0$. Then for any $\epsilon > 0$, there is a function 
	\begin{equation}
		h(x) = \sum_{k} \theta_k \psi_k(x) 
	\end{equation}
	for which 
	\begin{equation}
		\sup\limits_{x \in [0,\infty)} |f(x) - h(x)| < \epsilon.
	\end{equation}
	Here the basis functions $\{\psi_k\}$ may either be chosen as $\psi_k(z) = (z+\lambda)^{-k}$ or $\psi_k(x) = e^{-(kt_0)x}$ for any $t_0 > 0$.
	
\end{Th}

\begin{proof}
	This is a direct consequence of the Weierstrass approximation theorem.
\end{proof}

\section{Effective Propagation Schemes}\label{app:limitprop}

For definiteness, we here discuss limit-propagation schemes in the setting where \textbf{edge-weights} are large. The discussion for high-connectivity in the Sense of large cliques proceeds analogously.

In this section, we then take up again the setting of Section \ref{sec:effective_prop_section}. We reformulate this setting here in a slightly modified language, that is more adapted to discussing effective propagation schemes of standard architectures:\\

We partition edges on  a weighted graph $G$, into two disjoint sets $\mathcal{E} = \mathcal{E}_{\text{reg.}} \dot{\cup} \mathcal{E}_{\text{high}} $, where the set of edges with large weights is given by: 
\begin{equation}
	\mathcal{E}_{\text{high}} := \{(i,j)\in \mathcal{E}:   w_{ij} \geq S_{\text{high}}\}
\end{equation}
and the set with small weights is given by:
\begin{equation}
	\mathcal{E}_{\text{reg.}} := \{(i,j)\in \mathcal{E}: w_{ij} \leq S_{\text{reg.}}\}
\end{equation}
for weight scales $S_\text{high} > S_\text{reg.} > 0$. Without loss of generality, assume $S_\text{reg.}$ to be as low as possible (i.e. $S_\text{reg.} = \max_{(i,j)\in \mathcal{E}_\text{reg.}} w_{ij}$) and $S_\text{high}$ to be as high as possible (i.e. $S_\text{large} = \min_{(i,j)\in \mathcal{E}_\text{high}}$) and no weights in between the scales.
\begin{figure}[H]
	(a)\includegraphics[scale=0.27]{figures/graph}\hfill
	(b)\includegraphics[scale=0.27]{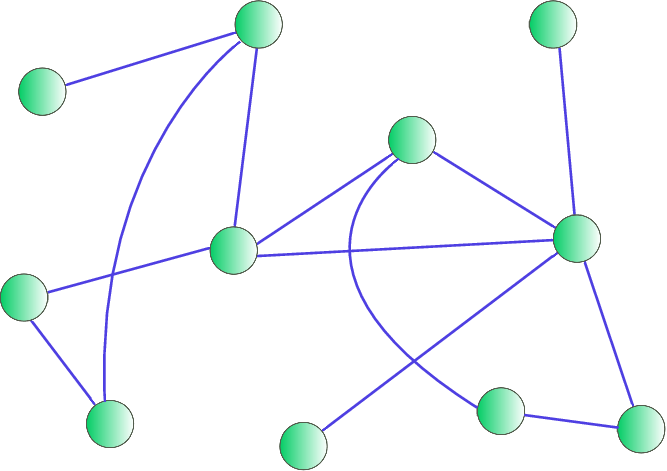}\hfill
	(c)\includegraphics[scale=0.27]{figures/graphhigh}\hfill
	(d)\includegraphics[scale=0.27]{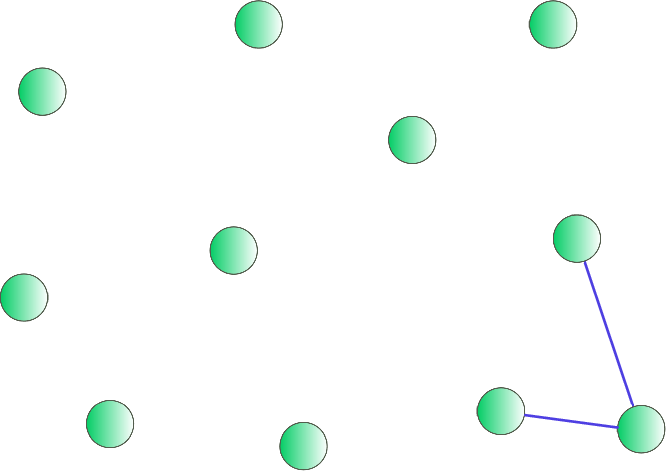}
	\captionof{figure}{(a) Graph $G$ with \textcolor{blue}{$\mathcal{E}_{\text{reg.}}$ (blue)} \& \textcolor{red}{$\mathcal{E}_{\text{high}}$ (red)};\ (b)  $G_{\text{reg.}}$; \ (c) $G_{\text{high}}$;\ (d) $G_{\text{reg., exclusive}}$    } 
	\label{graph_decomp_app}
\end{figure}
This decomposition induces two graph structures corresponding to the disjoint edge sets on the node set $\mathcal{G}$: We set $G_{\text{reg.}} := (\mathcal{G},\mathcal{E}_{\text{reg.}})$ and $G_{\text{high}}:=(\mathcal{G},\mathcal{E}_{\text{high}})$ c.f. Fig.  \ref{graph_decomp_app}).\\
We also introduce the set of edges 
$\mathcal{E}_{\text{reg., exclusive}} := \{(i,j)\in \mathcal{E}_{\text{reg.}}| \ \forall k \in \mathcal{G}:\ (i,k) \notin \mathcal{E}_{\text{high}}\ \&\ (k,j) \notin \mathcal{E}_{\text{high}} \}$
connecting nodes that do not have an incident edge in $\mathcal{E}_{\text{high}}$. A corresponding example-graph $G_{\text{reg., exclusive}}$ is depicted in Fig. \ref{graph_decomp_app} (d).\\
\ \\
We are now interested in the behaviour of graph convolution schemes if the scales are well separated:
\begin{equation}
	S_{\text{high}} \gg S_{\text{reg.}}
\end{equation}

\subsection{Spectral Convolutional Filters}
We first discuss resulting limit-propagation schemes for spectral convolutional networks. Such networks implement convolutional filters as a mapping
\begin{equation}
	x \longmapsto g_\theta(T)x
\end{equation}
for a node feature $x$, a learnable function $g_\theta$ and a graph shift operator $T$.

\subsubsection{Need for Normalization}\label{spectral_need_f_normalization}
The graph shift operator $T$ facilitating the graph convolutions needs to be normalized for established spectral graph convolutional architectures:

For \cite{ARMA}, this e.g. arises as a necessity for convergence of the proposed implementation scheme for the rational filters introduced there (c.f. eq. (10) in \cite{ARMA}).

The work \cite{Bresson} needs its graph shift operator to be normalized, as it approximates generic filters via a Chebyshev expansion. As argued in \cite{Bresson}, such Chebyshev polynomials form an orthogonal basis for the space $L^2([-1,1], dx/\sqrt{1-x^2})$. Hence, the spectrum of the operator $T$ to which the (approximated and learned) function $g_\theta$ is applied needs to be contained in the interval $[-1,1]$.

In \cite{Kipf}, it has been noted that for the architecture proposed there, choosing $T$ to have eigenvalues in the range $[0,2]$ (as opposed to the normalized ranges $[0,1]$ or $[-1,1]$) has the potential to lead to vanishing- or exploding gradients as well as numerical instabilities. To alleviate this, \cite{Kipf} introduces a "renormalization trick" (c.f. Section 2.2. of \cite{Kipf} to produce a normalized graph shift operator on which the network is then based.

We can understand the relationship between normalization of graph shift operator  $T$ and the stability of corresponding convolutional filters explicitly: Assume that we have 
\begin{equation}
	\|T\| \gg 1.
\end{equation}
This might e.g. happen when basing networks on the un-normalized graph Laplacian $\Delta$ or the weight-matrix $W$ if edge weights are potentially large (such as in the setting $S_{\text{high}} \gg S_{\text{reg.}}$ that we are considering).

By the spectral mapping theorem (see e.g. \cite{Teschl}), we have
\begin{equation}\label{specctral_mapping}
	\sigma\left( g_\theta(T)\right) = \left\{g_\theta(\lambda) : \lambda \in \sigma(T)\right\},
\end{equation}
with $\sigma(T)$ denoting the spectrum (i.e. the set of eigenvalues) of $T$. For the largest (in absolute value) eigenvalue $\lambda_{\max}$ of $T$, we have
\begin{equation}\label{sp_norm_a_l_ev}
	|\lambda_{\max}| = \|T\|.
\end{equation}
Since learned functions are implemented directly as a polynomial (as e.g. in \cite{Bresson}), we have 
\begin{equation}
	\lim\limits_{\lambda \rightarrow \pm \infty}  |g_\theta(\lambda)| = \infty.
\end{equation}
Thus in view of (\ref{specctral_mapping}) and (\ref{sp_norm_a_l_ev})  we have for $\|T\|$ sufficiently large, that
\begin{equation}
	\|g_\theta(T)\| =  |g_\theta(\pm\|T\|)|   
\end{equation}
with the sign $\pm$ determined by $\lambda_{\max}\gtrless 0$. Since non-constant polynomials behave at least linearly for large inputs, there is a constant $C>0$ such that 
\begin{equation}
	C\cdot\|T\| \leq\|g_\theta(T)\|
\end{equation}
for all sufficiently large $\|T\|$.	We thus have the estimate
\begin{equation}
	\|x\|\cdot C\cdot \|T\| \leq  \|g_\theta(T)x\|
\end{equation}
for at least one input signal $x$ (more precisely all $x$ in the eigen-space corresponding to the largest (in absolute value) eigenvalue $\lambda_{\max}$). Thus if $T$ is not normalized (i.e. $\|T\|$ is not sufficiently bounded), the norm of 
(hidden) features might increase drastically when moving from one (hidden) layer to the next. This behaviour persists for all input signals $x$  have components in eigenspaces corresponding to large (in absolute value) eigenvalues of $T$.

\subsubsection{Spectral Normalizations}\label{spectral_n_edges}
\begin{minipage}{0.6\textwidth}
	As discussed in the previous Section \ref{spectral_need_f_normalization}, instabilities arising from non-normalized graph shift operators can be traced back to the problem of such operators having large eigenvalues. It was thus -- among other considerations -- suggested in \cite{Bresson} to  base convolutional filters on the spectrally normalized graph shift operator
	\begin{equation}
		T = \frac{1}{\lambda_{\max}(\Delta)}\Delta,
	\end{equation}
	
\end{minipage}\hfill
\begin{minipage}{0.36\textwidth}
	\begin{figure}[H]
		\includegraphics[scale=0.3]{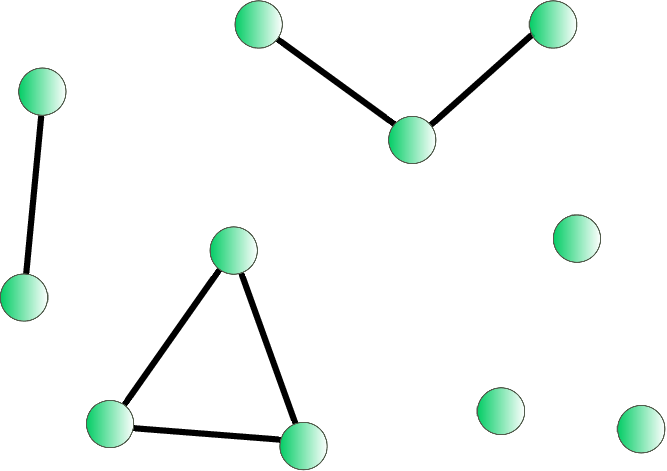}\hfill
		
		\captionof{figure}{Limit graph corresponding to Fig \ref{graph_decomp_app} for spectral  normalization 
		} 
		\label{spectral_norm_limit_app}
	\end{figure}
\end{minipage}

with $\Delta$ the un-normalized graph Laplacian. In the setting $S_{\text{high}} \gg S_{\text{reg.}}$ we are considering, this leads to an effective feature propagation along $G_{\text{high}}$ (c.f. also Fig. \ref{spectral_norm_limit_app}) only, as Theorem \ref{spectr_norm_limit_prop} below establishes:
\ \\

\begin{Thm}\label{spectr_norm_limit_prop}
	With 
	\begin{equation}
		T = \frac{1}{\lambda_{\max}(\Delta)}\Delta,
	\end{equation}
	and the scale decomposition as above
	we have that
	\begin{equation}\label{sp_norm_to_establish}
		\left\|T -  \frac{1}{\lambda_{\max}(\Delta_{\text{high}})}\Delta_{\text{high}} \right\| = \mathcal{O}\left(\frac{S_{\text{reg.}}}{S_{\text{high}}} \right)
	\end{equation}
	for  $S_{\text{high}} \gg S_{\text{reg.}}$.
\end{Thm}
\begin{proof}
	For convenience in notation, let us write 
	\begin{equation}
		T_{\text{high}} = \frac{1}{\lambda_{\max}(\Delta_{\text{high}})}\Delta_{\text{high}}
	\end{equation}	
	and similarly 
	\begin{equation}
		T_{\text{reg.}} = \frac{1}{\lambda_{\max}(\Delta_{\text{reg.}})}\Delta_{\text{reg.}}.
	\end{equation}	
	We may write
	\begin{equation}
		\Delta = \Delta_{\text{high}} + \Delta_{\text{reg.}},
	\end{equation}
	which we may rewrite as
	\begin{equation}\label{too_many_sp_norm_eqs}
		\Delta = \lambda_{\max}(\Delta_{\text{high}})\cdot \left(T_{\text{high}} + \frac{\lambda_{\max}(\Delta_{\text{reg.}})}{\lambda_{\max}(\Delta_{\text{high}})}\cdot T_{\text{reg.}} \right).
	\end{equation}
	Let us consider the equivalent expression 
	\begin{equation}\label{to_apply_Kato}
		\frac{1}{\lambda_{\max}(\Delta_{\text{high}})}\cdot \Delta = T_{\text{high}} + \frac{\lambda_{\max}(\Delta_{\text{reg.}})}{\lambda_{\max}(\Delta_{\text{high}})}\cdot T_{\text{reg.}}.
	\end{equation}
	We next note that
	\begin{equation}\label{sp_n_ev_division}
		\lambda_{\max}\left(  \frac{1}{\lambda_{\max}(\Delta_{\text{high}})}\cdot \Delta \right) = \frac{\lambda_{\max}(\Delta)}{\lambda_{\max}(\Delta_{\text{high}})}.
	\end{equation} 
	and
	\begin{equation}
		\lambda_{\max}\left(T_{\text{high}}\right) = 1
	\end{equation}
	since the operation of taking eigenvalues of operators is multiplicative in the sense of
	\begin{equation}
		\lambda_{\max}(|a|\cdot T) = |a|\cdot \lambda_{\max}(T)
	\end{equation}
	for non-negative $|a|\geq 0$.
	
	Since the right-hand-side of (\ref{to_apply_Kato}) constitutes an analytic  perturbation of $T_{\text{high}}$, we may apply analytic perturbation theory (c.f. e.g. \cite{Kato} for an extensive discussion) to this problem.  With this (together with $\|T_{\text{high}}\| = 1$) we find
	\begin{equation}\label{kato_sp_norm}
		\lambda_{\max}\left( \frac{1}{\lambda_{\max}(\Delta_{\text{high}})}\cdot\Delta \right) = 1 + \mathcal{O}\left( \frac{\lambda_{\max}(\Delta_{\text{reg.}})}{\lambda_{\max}(\Delta_{\text{high}})}\right).
	\end{equation}
	Using (\ref{sp_n_ev_division}) and the fact that 
	\begin{equation}\label{spec_norm_ev_rel}
		\frac{\lambda_{\max}(\Delta_{\text{reg.}})}{\lambda_{\max}(\Delta_{\text{high}})} \propto  \frac{S_{\text{reg.}}}{S_{\text{high}}},
	\end{equation}
	we thus have
	\begin{equation}
		\frac{\lambda_{\max}\left(\Delta \right) }{\lambda_{\max}(\Delta_{\text{high}})}  =1+ \mathcal{O}\left( \frac{S_{\text{reg.}}}{S_{\text{high}}}\right).
	\end{equation} 
	Since for small $\epsilon$, we also have
	\begin{equation}
		\frac{1}{1+\epsilon} = 1 + \mathcal{O}(\epsilon),
	\end{equation}
	the relation (\ref{spec_norm_ev_rel}) also implies
	\begin{equation}
		\frac{\lambda_{\max}(\Delta_{\text{high}})}{\lambda_{\max}\left(\Delta \right) }  =1+ \mathcal{O}\left( \frac{S_{\text{reg.}}}{S_{\text{high}}}\right).
	\end{equation} 
	Multiplying (\ref{too_many_sp_norm_eqs}) with $1/\lambda_{\max}(\Delta)$ yields
	\begin{align}\label{sp_norm_another_one}
		T = \frac{\lambda_{\max}(\Delta_{\text{high}})}{\lambda_{\max}(\Delta)}\cdot \left(T_{\text{high}} + \frac{\lambda_{\max}(\Delta_{\text{reg.}})}{\lambda_{\max}(\Delta_{\text{high}})}\cdot T_{\text{reg.}} \right).
	\end{align}
	Since $\|T_{\text{high}}\|,\|T_{\text{reg.}}\|=1$ and 
	\begin{equation}
		\frac{\lambda_{\max}(\Delta_{\text{reg.}})}{\lambda_{\max}(\Delta_{\text{high}})} \propto \frac{S_{\text{reg.}}}{S_{\text{high}}} < 1
	\end{equation}
	for sufficiently large $S_{\text{high}}$, relation (\ref{sp_norm_another_one}) implies 
	\begin{equation}
		\left\|T -  \frac{1}{\lambda_{\max}(\Delta_{\text{high}})}\Delta_{\text{high}} \right\| = \mathcal{O}\left(\frac{S_{\text{reg.}}}{S_{\text{high}}} \right)
	\end{equation}
	as desired.
	
	Note that we might in principle also make use of Lemma \ref{evs_lipschitz} below,  to provide quantitative bounds:
	Lemma \ref{evs_lipschitz}  states that
	\begin{equation}
		|\lambda_k(A) - \lambda_k(B)| \leq \|A-B\|
	\end{equation}
	for self-adjoint operators $A$ and $B$ and their respective $k^{\text{th}}$ eigenvalues ordered by magnitude. On a graph with $N$ nodes, we clearly have $\lambda_{\max} = \lambda_N$ for eigenvalues of (rescaled) graph Laplacians, since all such eigenvalues are non-negative.
	This implies for the difference $|1 - \lambda_{\max}(\Delta)/\lambda_{\max}(\Delta_{\text{high}})|$ arising in (\ref{kato_sp_norm}) that explicitly
	\begin{align}
		\left|1 - \frac{\lambda_{\max}(\Delta)}{\lambda_{\max}(\Delta_{\text{high}})}\right| & \leq \frac{\lambda_{\max}(\Delta_{\text{reg.}})}{\lambda_{\max}(\Delta_{\text{high}})}.
	\end{align}
	This in turn can then be used to provide a quantitative bound in (\ref{sp_norm_to_establish}). Since we are only interested in the qualitative behaviour for $S_{\text{high}} \gg S_{\text{reg.}}$, we shall however not pursue this further.

\end{proof}

It remains to state and establish Lemma \ref{evs_lipschitz} referenced at the end of the proof of Theorem \ref{spectr_norm_limit_prop}:\\

\begin{Lem}\label{evs_lipschitz}
	Let $A$ and $B$ be two hermitian $n \times n$ dimensional matrices. Denote by
	$\{\lambda_k(M)\}_{k=1}^n$ the eigenvalues of a hermitian matrix in increasing order.\\
	With this we have:
	
	\begin{equation}
		|\lambda_k(A) - \lambda_k(B)| \leq ||A-B||.
	\end{equation}
\end{Lem}
\begin{proof}
	After the redefinition $B\mapsto(-B)$, what we need to prove is
	\begin{equation}
		|\lambda_i(A+B) - \lambda_i(A)| \leq ||B||
	\end{equation}	
	for Hermitian $A,B$. 
	Since we have
	\begin{equation}
		\lambda_i(A) - \lambda_i (A+B) = \lambda_i((A+B) +(-B)) - \lambda_i(A+B)
	\end{equation}
	and $||-B|| = ||B||$ it follows that it suffices to prove
	\begin{equation}
		\lambda_i(A + B) - \lambda_i(A) \leq ||B||
	\end{equation}
	for arbitrary hermitian $A,B$.

	We note that the Courant-Fischer	 $\min-\max$ theorem tells us that if $A$ is an $n \times n$ Hermitian matrix, we have
	\begin{equation}
		\lambda_i(M) = \sup\limits_{\dim(V)=i}\inf\limits_{v\in V, ||v|| =1} v^*Mv.
	\end{equation}
	
	With this we find
	\begin{align}
		\lambda_i(A + B) - \lambda_i(A) &=  \sup\limits_{\dim(V)=i}\inf\limits_{v\in V, ||v|| =1} v^*(A+B)v - \sup\limits_{\dim(V)=i}\inf\limits_{v\in V, ||v|| =1} v^*Av \\
		&\leq  \sup\limits_{\dim(V)=i}\inf\limits_{v\in V, ||v|| =1} v^*Av + \sup\limits_{\dim(V)=i}\inf\limits_{v\in V, ||v|| =1} v^*Bv\\
		& - \sup\limits_{\dim(V)=i}\inf\limits_{v\in V, ||v|| =1} v^*Av \\
		&=\sup\limits_{\dim(V)=i}\inf\limits_{v\in V, ||v|| =1} v^*Bv\\
		&= \sup\limits_{\dim(V)=i}\inf\limits_{v\in V, ||v|| =1} v^*Bv\\
		&\leq \max\limits_{1\leq k \leq n}\{|\lambda_k(B)|\}\\
		&= ||B||.
	\end{align}
\end{proof}

\subsubsection{Symmetric Normalizations}\label{sym_n_edges}
\begin{minipage}{0.6\textwidth}
	Most common spectral graph convolutional networks (such as e.g. \cite{bernnet, ARMA, Bresson}) base the learnable filters that they propose on the symmetrically normalized graph Laplacian
	\begin{equation}
		\mathscr{L} = Id - D^{-\frac12}WD^{-\frac12}.
	\end{equation}
	In the setting $S_{\text{high}} \gg S_{\text{reg.}}$ we are considering, this leads to an effective feature propagation along edges in $\mathcal{E}_{\text{high}}$ and $\mathcal{E}_{\text{low, exclusive}}$ (c.f. also Fig. \ref{sym_norm_limit_app}) only, as Theorem \ref{sym_norm_limit_prop} below establishes:

\end{minipage}\hfill
\begin{minipage}{0.36\textwidth}
	\begin{figure}[H]
		\includegraphics[scale=0.3]{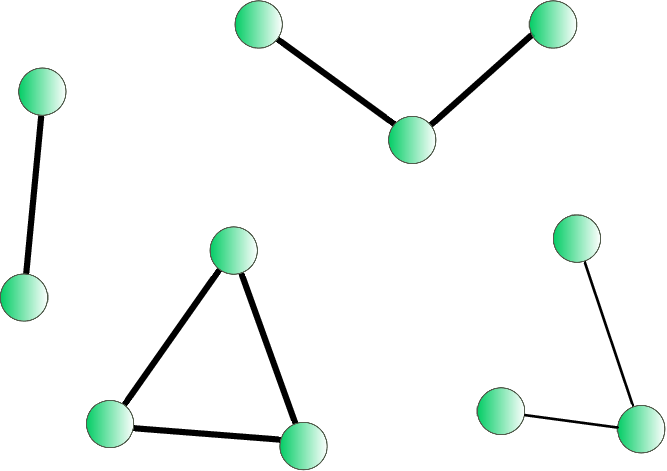}\hfill
		
		\captionof{figure}{Limit graph corresponding to Fig \ref{graph_decomp_app} for symmetric  normalization 
		} 
		\label{sym_norm_limit_app}
	\end{figure}
\end{minipage}

\begin{Thm}\label{sym_norm_limit_prop}
	With 
	\begin{equation}
		T = Id -  D^{-\frac12}WD^{-\frac12},
	\end{equation}
	and the scale decomposition as introduced 
	above,
	we have that
	\begin{equation}\label{sym_norm_to_establish}
		\left\|T - \left(Id - D^{-\frac12}_{\text{high}} W_{\text{high}} D^{-\frac12}_{\text{high}} - D^{-\frac12}_{\text{reg.}} W_{\text{low, exclusive}} D^{-\frac12}_{\text{reg.}}\right)\right\| =
		\mathcal{O}\left(\sqrt{\frac{S_{\text{reg.}}}{S_{\text{high}}}}\right)
	\end{equation}
	for  $S_{\text{high}} \gg S_{\text{reg.}}$.
\end{Thm}
\begin{proof}
	We first note that instead of (\ref{sym_norm_to_establish}), we may equivalently establish 
	\begin{equation}\label{sym_norm_to_establish_actual}
		\left\|D^{-\frac12} W D^{-\frac12}- \left(  D^{-\frac12}_{\text{high}} W_{\text{high}} D^{-\frac12}_{\text{high}} + D^{-\frac12}_{\text{reg.}} W_{\text{low, exclusive}} D^{-\frac12}_{\text{reg.}}\right)\right\| = \mathcal{O}\left(\sqrt{\frac{S_{\text{reg.}}}{S_{\text{high}}}}\right).
	\end{equation}
	We have
	\begin{equation}
		W = W_{\text{high}} + W_{\text{reg.}}.
	\end{equation}		
	With this, we may write
	\begin{equation}\label{sym_norm_deco_second}
		D^{-\frac12} W D^{-\frac12} = D^{-\frac12} W_{\text{high}} D^{-\frac12} + D^{-\frac12} W_{\text{reg.}} D^{-\frac12}.
	\end{equation}
	Let us first examine the term $D^{-\frac12} W_{\text{high}} D^{-\frac12}$. We note for the corresponding matrix entries that
	\begin{equation}
		\left(D^{-\frac12} W_{\text{high}} D^{-\frac12}\right)_{ij} = \frac{1}{\sqrt{d_i}}\cdot (W_{\text{high}})_{ij}\cdot \frac{1}{\sqrt{d_j}}
	\end{equation}
	Let us use the notation
	\begin{equation}
		d_i^\text{high} = \sum\limits_{j=1}^N  (W_{\text{high}})_{ij}, \ \ \ \ d_i^\text{reg.} = \sum\limits_{j=1}^N (W_{\text{reg.}})_{ij}\ \ \text{and} \ \ d_i^\text{low,exclusive} = \sum\limits_{j=1}^N (W_{\text{low,exclusive}})_{ij}.
	\end{equation}
	We then find
	\begin{equation}
		\frac{1}{\sqrt{d_i}} = \frac{1}{\sqrt{d^\text{high}_i}}\cdot \frac{1}{\sqrt{1+\frac{d^\text{reg.}_i}{d^\text{high}_i}}}
	\end{equation}
	Using the Taylor expansion
	\begin{equation}
		\frac{1}{\sqrt{1+\epsilon}} = 1 - \frac{1}{2}\epsilon + \mathcal{O}(\epsilon^2),
	\end{equation}
	we thus have
	\begin{equation}
		\left(D^{-\frac12} W_{\text{high}} D^{-\frac12}\right)_{ij} = \frac{1}{\sqrt{d^{\text{high}}_i}}\cdot (W_{\text{high}})_{ij}\cdot \frac{1}{\sqrt{d^{\text{high}}_j}} + \mathcal{O}\left(\frac{d^\text{reg.}_i}{d^\text{high}_i}\right).
	\end{equation}
	Since we have
	\begin{equation}
		\frac{d^\text{reg.}_i}{d^\text{high}_i} \propto \frac{S_\text{reg.}}{S_\text{high}},
	\end{equation}
	this yields 
	\begin{equation}
		D^{-\frac12} W_{\text{high}} D^{-\frac12} =  D^{-\frac12}_{\text{high}} W_{\text{high}} D^{-\frac12}_{\text{high}} + \mathcal{O}\left(\frac{S_\text{reg.}}{S_\text{high}}\right).
	\end{equation}
	Thus let us turn towards the second summand on the right-hand-side of (\ref{sym_norm_deco_second}).
	We have
	\begin{equation}
		\left( D^{-\frac12} W_{\text{reg.}} D^{-\frac12}\right)_{ij} = \frac{1}{\sqrt{d_i}}\cdot (W_{\text{reg.}})_{ij}. \frac{1}{\sqrt{d_j}}.
	\end{equation}
	Suppose that either $i$ or $j$ is not in $G_{\text{low, exclusive}}$. Without loss of generality (since the matrix under consideration is symmetric), assume $i \notin G_{\text{low, exclusive}}$, but $(W_{\text{reg.}})_{ij} \neq 0$. We may again write
	\begin{equation}
		\frac{1}{\sqrt{ d_j}} = \frac{1}{\sqrt{ d^{\text{high}}_j}} \cdot \frac{1}{\sqrt{1+\frac{d^\text{reg.}_i}{d^\text{high}_i}}}.
	\end{equation}
	Since 
	\begin{equation}
		\frac{1}{\sqrt{1+\frac{d^\text{reg.}_i}{d^\text{high}_i}}} \leq 1,
	\end{equation}
	we have
	\begin{equation}
		\left|\left( D^{-\frac12} W_{\text{reg.}} D^{-\frac12}\right)_{ij}\right| \leq \left|\frac{1}{\sqrt{d_i}}\cdot (W_{\text{reg.}})_{ij}\right|\cdot \frac{1}{\sqrt{d_j^{\text{high}}}} = \mathcal{O}\left(\sqrt{\frac{S_{\text{reg.}}}{S_{\text{high}}}}\right).
	\end{equation}
	If instead we have $i,j \in G_{\text{low, exclusive}}$, then clearly
	\begin{equation}
		\left( D^{-\frac12} W_{\text{reg.}} D^{-\frac12}\right)_{ij} 
		= \left( D^{-\frac12}_{\text{reg.}} W_{\text{low,exclusive}} D^{-\frac12}_{\text{reg.}}\right)_{ij}.
	\end{equation}
	Thus in total we have established
	\begin{equation}
		D^{-\frac12} W D^{-\frac12}  = \left(  D^{-\frac12}_{\text{high}} W_{\text{high}} D^{-\frac12}_{\text{high}} + D^{-\frac12}_{\text{reg.}} W_{\text{low, exclusive}} D^{-\frac12}_{\text{reg.}}\right) + \mathcal{O}\left(\frac{S_{\text{reg.}}}{S_{\text{high}}} \right)
	\end{equation}
	which was to be established.

\end{proof}

Apart from networks that make use of the symmetrically normalized graph Laplacian $\mathscr{L}$, some methods, such as most notably \cite{Kipf}, instead base their filters on the operator	  
\begin{equation}
	T = \tilde{D}^{-\frac12}\tilde{W}\tilde{D}^{-\frac12},
\end{equation}
with 
\begin{equation}
	\tilde{W} = (W+Id)
\end{equation}
and
\begin{equation}
	\tilde{D} = D + Id.
\end{equation}
In analogy to Theorem \ref{sym_norm_limit_prop}, we here establish the limit propagation scheme determined by such operators:

\begin{Thm}\label{sym_norm_limit_prop_GCN}
	With 
	\begin{equation}
		T = \tilde{D}^{-\frac12}\tilde{W}\tilde{D}^{-\frac12},
	\end{equation}
	where $\tilde{W} = (W+Id)$ and $\tilde{D} = D + Id$ as well as the scale decomposition introduced above, 	we have that
	\begin{equation}\label{sym_norm_to_establish_GCN}
		\left\|T - \left( D^{-\frac12}_{\text{high}} W_{\text{high}} D^{-\frac12}_{\text{high}} + D^{-\frac12}_{\text{reg.}} \tilde{W}_{\text{low, exclusive}} D^{-\frac12}_{\text{reg.}}\right)\right\| =
		\mathcal{O}\left(\sqrt{\frac{S_{\text{reg.}}+1}{S_{\text{high}}}}\right)
	\end{equation}
	for  $S_{\text{high}} \gg S_{\text{reg.}}$. Here $\tilde{W}_{\text{low, exclusive}}$ is given as
	\begin{equation}
		\tilde{W}_{\text{low, exclusive}} := W_{\text{low, exclusive}} + \text{diag}\left( \mathds{1}_{G_{\text{low, exclusive}}}\right)
	\end{equation}
	and $\mathds{1}_{G_{\text{low, exclusive}}}$ denotes the vector whose entries are one for nodes in $G_{\text{low, exclusive}}$ and zero for all other nodes.
\end{Thm}
The difference to the result of Theorem \ref{sym_norm_limit_prop} is thus that applicability of the limit propagation scheme of Fig. \ref{sym_norm_limit_app}  for the GCN \cite{Kipf}   is not only contingent upon $S_{\text{high}}\gg S_{\text{reg.}}$ but also  $S_{\text{high}}\gg1$.

\begin{proof}
	To establish this -- as in the proof of Theorem \ref{sym_norm_limit_prop} -- we first decompose $T$:
	\begin{align}\label{kipfs_mode_deco}
		\tilde{D}^{-\frac12}\tilde{W}\tilde{D}^{-\frac12} &= \tilde{D}^{-\frac12}W_{\text{high}}\tilde{D}^{-\frac12} + \tilde{D}^{-\frac12}W_{\text{reg.}}\tilde{D}^{-\frac12} + \tilde{D}^{-\frac12}Id\tilde{D}^{-\frac12}\\
		&= \tilde{D}^{-\frac12}W_{\text{high}}\tilde{D}^{-\frac12} + \tilde{D}^{-\frac12}W_{\text{reg.}}\tilde{D}^{-\frac12} + \tilde{D}^{-1}
	\end{align}
	For the first term, we note
	\begin{equation}
		\left(\tilde{D}^{-\frac12}W_{\text{high}}\tilde{D}^{-\frac12}\right)_{ij} = \frac{1}{\sqrt{d_i+1}}\cdot (W_{\text{high}})_{ij}\cdot \frac{1}{\sqrt{d_j+1}}.
	\end{equation}	
	
	We then find
	\begin{equation}
		\frac{1}{\sqrt{d_i+1}} = \frac{1}{\sqrt{d^\text{high}_i}}\cdot \frac{1}{\sqrt{1+\frac{d^\text{reg.}_i+1}{d^\text{high}_i}}}.
	\end{equation}
	Analogously to the proof of Theorem \ref{sym_norm_limit_prop}, this yields 	
	\begin{equation}
		\left(\tilde D^{-\frac12} W_{\text{high}} \tilde D^{-\frac12}\right)_{ij} = \frac{1}{\sqrt{d^{\text{high}}_i}}\cdot (W_{\text{high}})_{ij}\cdot \frac{1}{\sqrt{d^{\text{high}}_j}} + \mathcal{O}\left(\frac{1+d^\text{reg.}_i}{d^\text{high}_i}\right).
	\end{equation}
	This implies 
	\begin{equation}
		\tilde{D}^{-\frac12} W_{\text{high}} \tilde{D}^{-\frac12} =  D^{-\frac12}_{\text{high}} W_{\text{high}} D^{-\frac12}_{\text{high}} + \mathcal{O}\left(\frac{S_\text{reg.}+1}{S_\text{high}}\right).
	\end{equation}
	Next we turn to the second summand in (\ref{kipfs_mode_deco}):
	
	\begin{equation}
		\left(\tilde D^{-\frac12} W_{\text{reg.}}\tilde D^{-\frac12}\right)_{ij} = \frac{1}{\sqrt{d_i+1}}\cdot (W_{\text{reg.}})_{ij}. \frac{1}{\sqrt{d_j+1}}.
	\end{equation}

	Suppose that either $i$ or $j$ is not in $G_{\text{low, exclusive}}$. Without loss of generality (since the matrix under consideration is symmetric), assume $i \notin G_{\text{low, exclusive}}$, but $(W_{\text{reg.}})_{ij} \neq 0$. We may again write
	\begin{equation}
		\frac{1}{\sqrt{ d_j+1}} = \frac{1}{\sqrt{ d^{\text{high}}_j}} \cdot \frac{1}{\sqrt{1+\frac{d^\text{reg.}_i+1}{d^\text{high}_i}}}.
	\end{equation}
	Since 
	\begin{equation}
		\frac{1}{\sqrt{1+\frac{d^\text{reg.}_i+1}{d^\text{high}_i}}} \leq 1,
	\end{equation}
	we have
	\begin{align}
		\left|\left( D^{-\frac12} W_{\text{reg.}} D^{-\frac12}\right)_{ij}\right| &\leq \left|\frac{1}{\sqrt{1+d_i}}\cdot (W_{\text{reg.}})_{ij}\right|\cdot \frac{1}{\sqrt{d_j^{\text{high}}}}\\
		&\leq \left|\frac{1}{\sqrt{d^{\text{reg.}}_i}}\cdot (W_{\text{reg.}})_{ij}\right|\cdot \frac{1}{\sqrt{d_j^{\text{high}}}}\\
		&= \mathcal{O}\left(\sqrt{\frac{S_{\text{reg.}}}{S_{\text{high}}}}\right).
	\end{align}

	If instead we have $i,j \in G_{\text{low, exclusive}}$, then clearly
	\begin{equation}
		\left( \tilde D^{-\frac12} W_{\text{reg.}} \tilde D^{-\frac12}\right)_{ij} 
		= \left( \tilde D^{-\frac12}_{\text{reg.}} W_{\text{low,exclusive}} \tilde D^{-\frac12}_{\text{reg.}}\right)_{ij}.
	\end{equation}

	Finally we note for the third term on the right-hand-side of (\ref{kipfs_mode_deco}) that
	\begin{equation}
		\frac{1}{d_i}\leq \frac{1}{d^{\text{high}}_i} = \mathcal{O}\left(\frac{1}{S_{\text{high}}}\right)
	\end{equation}
	if $i\notin G_{\text{low, exclusive}}$.

	In total we thus have found
	\begin{equation}
		\tilde{D}^{-\frac12}\tilde{W}\tilde{D}^{-\frac12} = \left( D^{-\frac12}_{\text{high}} W_{\text{high}} D^{-\frac12}_{\text{high}} + D^{-\frac12}_{\text{reg.}} \tilde{W}_{\text{low, exclusive}} D^{-\frac12}_{\text{reg.}}\right) + \mathcal{O}\left(\sqrt{\frac{S_{\text{reg.}}+1}{S_{\text{high}}}}\right);
	\end{equation}
	which was to be proved.
\end{proof}

\subsection{Spatial Convolutional Filters}\label{mpnns}

Apart from spectral methods, there of course also exist methods that purely operate in the spatial domain of the graph. Such methods most often fall into the paradigm of message passing neural networks (MPNNs) \cite{mpnncm,pyg}: 
With  $X^{\ell}_i\in\mathds{R}^F$ denoting the features of node $i$ in layer $\ell$ and $w_{ij}$ denoting edge features, a message passing neural network may be described by the update rule (c.f. \cite{mpnncm})
\begin{equation}\label{mpnn_def_eq}
	X^{\ell+1}_i = \gamma\left(X^{\ell}_i, \coprod\limits_{j \in \mathcal{N}(i)} \phi\left(X_i^\ell, X_j^\ell, w_{ij}\right)\right).
\end{equation}
Here $ \mathcal{N}(i)$ denotes the neighbourhood of node $i$,   $\coprod$ denotes a differentiable and permutation invariant function (typically "sum", "mean" or "max") while $\gamma$ and $\phi$ denote differentiable functions such as multi-layer-perceptrons (MLPs) which might not be the same in each layer. 
\cite{pyg}.

Before we discuss corresponding limit-propagation schemes, we first establish that MPNNs are not able to reproduce the limit propagation scheme ofFigure \ref{collapse_weighted_II} (b) and are thus not stable to scale transitions and topological perturbations.

\subsubsection{Scale-Sensitivity of Message Passing Neural Networks}\label{spatial_scale_sensitivity}

Here we establish that message passing networks (as defined in (\ref{mpnn_def_eq}) above) are unable to emulate a limit propagation scheme similar to the one in Figure \ref{collapse_weighted_II} (b). Hence such architectures are also not stable to scale-changing topological perturbations such as coarse-graining procedures.
\noindent

\begin{minipage}{0.6\textwidth}	
	To this end, we consider a simple, fully connected graph $G$ on three nodes labeled $1$, $2$ and $3$ (c.f. Fig. \ref{three_graph}). We assume all node-weights to be equal to one ($\mu_i = 1$ for $i = 1,2,3$) and  edge weights 
	\begin{equation}
		w_{13}, w_{23} \leq S_{\text{reg.}}
	\end{equation}
	as well as 
	\begin{equation}
		w_{12} = S_{\text{high}}.
	\end{equation}
	We now assume $ S_{\text{high}} \gg S_{\text{reg.}}$.
\end{minipage}\hfill
\begin{minipage}{0.36\textwidth}
	\begin{figure}[H]
		\includegraphics[scale=0.6]{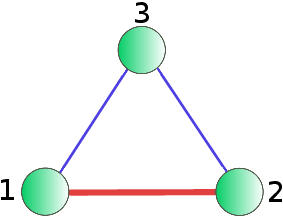}\hfill
		
		\captionof{figure}{Three node Graph $G$ with on large weight $w_{12}\gg1$.
		} 
		\label{three_graph}
	\end{figure}
\end{minipage}
Given states $\{X_1^\ell, X_2^\ell, X_3^\ell\}$ in layer $\ell$, a limit propagation scheme as in Figure \ref{collapse_weighted_II} (b)
would require the updated feature vector of node $3$ to be given by 
\begin{equation}
	X_{3, \text{desired}}^{\ell +1} := \gamma\left(X^{\ell}_3, \phi\left(X_3^\ell, \frac{ X_1^\ell+X_2^\ell}{2}, (w_{31}+w_{32})\right) \right)
\end{equation}

However, the actual updated feature at node $3$ is  given as (c.f. (\ref{mpnn_def_eq})):
\begin{equation}\label{node_three_no_scale_dep_rem}
	X_{3, \text{actual}}^{\ell +1} := \gamma\left(X^{\ell}_3, \phi\left(X_3^\ell, X_1^\ell, w_{31}\right) \coprod \phi\left(X_3^\ell, X_2^\ell, w_{32}\right)\right)
\end{equation}
Since there is no dependence on $S_{\text{high}}$ in equation (\ref{node_three_no_scale_dep_rem}) -- which defines  $X_{3, \text{actual}}^{\ell +1}$ -- the desired propagation scheme can not arise, unless it is paradoxically already present at all scales $S_{\text{high}}$. If it is present at all scales, there is however only propagation along edges in $\underline{G}$, even if $S_{\text{high}}\approx S_{\text{reg.}}$, which would imply that the message passing network would not respect the graph structure of $G$. 	Hence $X_{3, \text{actual}}^{\ell +1} \nrightarrow X_{3, \text{desired}}^{\ell +1}$ does not converge as $S_{\text{high}}$ increases.

\subsubsection{Limit Propagation Schemes}

The number of possible choices of message functions $\phi$, aggregation functions $\coprod$ and update functions $\gamma$ is clearly endless. Here we shall exemplarily discuss limit propagation schemes for two popular architectures: We first  discuss the most general case where the message function $\phi$ is given as a learnable perceptron. Subsequently we assume that node features are updated with an attention-type mechanism.

\paragraph{Generic message functions:}
We first consider the possibility that the message function $\phi$ in (\ref{node_three_no_scale_dep_rem}) 
is implemented via an MLP using ReLU-activations: Assuming (for simplicity in notation) a one-hidden-layer MLP mapping features $X_i^\ell \in \mathds{R}^{F_\ell}$ to features $X_i^{\ell+1} \in \mathds{R}^{F_\ell+1}$ 
we have
\begin{equation}
	\phi(X^\ell_i,X^\ell_j,w_{ij}) = \text{ReLU}\left(W^\ell_1\cdot X^\ell_i+W^\ell_2\cdot X^\ell_2+ W^\ell_3\cdot w_{ij} +B^\ell    \right)
\end{equation}
with bias term $B^{\ell+1} \in \mathds{R}^{F_{\ell+1}}$ and  weight matrices $W^{\ell+1}_1,W^{\ell+1}_2 \in \mathds{R}^{F_{\ell+1}\times F_{\ell}}$ and $W^\ell_3 \in \mathds{R}^{F_{\ell+1}}$.

We will assume that the weight-vecor $W_3^{\ell+1}$ has no-nonzero entries.
This is not a severe limitation experimentally and in fact generically justified: The complementary event of at-least one entry of $W_3$ being assigned precisely zero during training has probability weight zero (assuming an absolutely continuous probability distribtuion according to which weights are learned).

Let us now assume that the edge $(ij)$ belongs to $\mathcal{E}_{\text{high}}$ and the corresponding weight $w_{ij}$ is large ($w_{ij} \gg 1$). The behaviour of entries $\phi(X^\ell_i,X^\ell_j,w_{ij})_a$ of the message $\phi(X^\ell_i,X^\ell_j,w_{ij}) \in \mathds{R}^{F_{\ell+1}}$ is then determined by the sign of the corresponding entry $\left(W_3^\ell\right)_a$ of the weight vector $W_3^\ell\in \mathds{R}^{F_{\ell+1}}$: 

If we have $\left(W_3^\ell\right)_a<0$, then $\phi(X^\ell_i,X^\ell_j,w_{ij})_a $ approaches zero for larger edge-weights $w_{ij}$:
\begin{equation}\label{uninformative_message}
	\lim\limits_{w_{ij}\rightarrow \infty} \phi(X^\ell_i,X^\ell_j,w_{ij})_a = 0
\end{equation} 

If we have $\left(W_3^\ell\right)_a>0$, then $\phi(X^\ell_i,X^\ell_j,w_{ij})_a $ increasingly diverges for larger edge-weights $w_{ij}$:
\begin{equation}\label{dangerous_message}
	\lim\limits_{w_{ij}\rightarrow \infty} \phi(X^\ell_i,X^\ell_j,w_{ij})_a = \infty
\end{equation} 
For either choice of  aggregation function  $\coprod$ in (\ref{mpnn_def_eq}) among  "max", "sum" or "mean" the behaviour in (\ref{dangerous_message}) leads to unstable networks if the update function $\gamma$ is also given as an MLP with ReLU activations. Apart from instabilities,
we also make the following observation: If  $S_{\text{high}} \gg S_{\text{reg.}}$, then by (\ref{dangerous_message}) and continuity of $\phi$ we can conclude that components $\phi(X^\ell_i,X^\ell_j,w_{ij})_a $ of messages propagated along $\mathcal{E}_{\text{high}}$  for which  $\left(W_3^\ell\right)_a>0$ dominate over messages propagated along edges in $\mathcal{E}_{\text{reg.}}$. By (\ref{uninformative_message}), the former clearly also dominate over components $\phi(X^\ell_i,X^\ell_j,w_{ij})_a $ of messages propagated along $\mathcal{E}_{\text{high}}$ for which $\left(W_3^\ell\right)_a<0$. This behaviour is irrespective of whether  "max", "sum" or "mean" aggregations are employed. Hence the limit propagation scheme essentially only takes into account message channels $\phi(X^\ell_i,X^\ell_j,w_{ij})_a $ for which $(ij)\in \mathcal{E}_{\text{high}}$ and $\left(W_3^\ell\right)_a>0$.

Similar considerations apply, if non-linearities are chosen as leaky ReLU. 
If instead of ReLU activations a sigmoid-nonlinearity $\sigma$ like $\tanh$ is employed,   messages  propagated along $\mathcal{E}_{\text{large}}$ become increasingly uninformative, since they are progressively more independent of features $X_i^\ell$ and weights $w_{ij}$.  
Indeed, for sigmoid activations, the limits (\ref{uninformative_message}) and (\ref{dangerous_message}) are given as follows:

If we have $\left(W_3^\ell\right)_a<0$, then we have for larger edge-weights $w_{ij}$ that
\begin{equation}\label{uninformative_message_sigma}
	\lim\limits_{w_{ij}\rightarrow \infty} \phi(X^\ell_i,X^\ell_j,w_{ij})_a = \lim\limits_{y \rightarrow -\infty} \sigma(y).
\end{equation} 

If we have $\left(W_3^\ell\right)_a>0$, then 
\begin{equation}\label{dangerous_message_sigma}
	\lim\limits_{w_{ij}\rightarrow \infty} \phi(X^\ell_i,X^\ell_j,w_{ij})_a = \lim\limits_{y \rightarrow \infty} \sigma(y).
\end{equation} 
In both cases, the messages $\phi(X^\ell_i,X^\ell_j,w_{ij})$ propagated along $\mathcal{E}_{\text{large}}$ become increasingly constant as the scale $S_{\text{high}}$ increases.

\paragraph{Attention based messages:}
Apart from general learnable message functions as above, we here also discuss an approach where edge weights are re-learned in an attention based manner. For this we modify the method \cite{GAT} to include edge weights. The resulting propagation scheme -- with a single attention head for simplicity and a non-linearity $\rho$ -- is given as
\begin{equation}
	X_i^{\ell +1} = \rho\left(\sum\limits_{j \in \mathcal{N}(i)}\alpha_{ij}(WX_j^{\ell +1})\right).
\end{equation}
Here we have $W \in \mathds{R}^{F_{\ell+1}\times F_\ell}$ and
\begin{equation}\label{att_weight_mat}
	\alpha_{ij}=\frac{\exp\left(\text{LeakyRelu}\left(\vec{a}^\top\left[WX_i^\ell\mathbin{\|}WX_j^\ell\mathbin{\|}w_{ij}\right]\right)\right)}{\sum\limits_{k \in \mathcal{N}(i)}\exp\left(\text{LeakyRelu}\left(\vec{a}^\top\left[WX_i^\ell\mathbin{\|}WX_k^\ell\mathbin{\|}w_{ik}\right]\right)\right)},
\end{equation}
with $\mathbin{\|}$ denoting concatenation. The weight vector $\vec{a}\in \mathds{R}^{2F_{\ell+1}+1}$ is assumed to have a non zero entry in its last component. Otherwise, this attention mechanism would correspond to the one proposed in \cite{GAT}, which does not take into account edge weights. Let us denote this entry of $\vec{a}$ ()determining attention on the weight $w_{ij}$)  by $a_w$. 

If  $a_w <0 $, we have for $(i,j)\in \mathcal{E}_{\text{high}}$ that
\begin{equation}
	\exp\left(\text{LeakyRelu}\left(\vec{a}^\top\left[WX_i^\ell\mathbin{\|}WX_j^\ell\mathbin{\|}w_{ij}\right]\right) \right) \longrightarrow 0
\end{equation} 
as the weight $w_{ij}$ increases. Thus propagation along edges in $\mathcal{E}_{\text{high}}$ is essentially suppressed in this case.

If  $a_w >0 $, we have for $(i,j)\in \mathcal{E}_{\text{high}}$ that
\begin{equation}
	\exp\left(\text{LeakyRelu}\left(\vec{a}^\top\left[WX_i^\ell\mathbin{\|}WX_j^\ell\mathbin{\|}w_{ij}\right]\right) \right) \longrightarrow \infty
\end{equation}
as the weight $w_{ij}$ increases.
Thus for edges $(i,j) \in \mathcal{E}_{\text{reg.}}$ (i.e. those that are \textit{not} in $\mathcal{E}_{\text{high}}$), we have
\begin{equation}
	\alpha_{ij} \rightarrow 0,
\end{equation}
since the denominator in (\ref{att_weight_mat}) diverges. Hence in this case, propagation along $\mathcal{E}_{\text{reg.}}$ is essentially suppressed and features are effectively only propagated along $\mathcal{E}_{\text{high}}$.

\section{Stability of latent representations generated by Laplace-transform based GNNs for graphs on the same node set }\label{app:gogginstheoremII_proof}

\subsection{Node Level Stability}\label{app:node_level_stab}

In this Section, we are concerned with two graphs $G_1$ and $G_2$ that share a common node set.
We assume that their heat kernels are similar in the sense of
\begin{align}
	\|e^{-tL_1} - e^{-t L_2}\| \approx 0.
\end{align}
We are then interested in bounding the variation in latent representations generated by a GNN $\Phi_{\mathscr{W},\mathscr{B},\Psi}$ (as defined in Section \ref{sec:scale_ct_gnns}) with weights $\mathscr{W}$, (potential) biases, and Laplace transform propagation matrices $\Psi$.

\subsubsection{Spectral Laplace-transform methods}

We first consider spectral Laplace Transform based spectral GNNs:

\begin{Thm}\label{gogginstheoremII}
	Let $\Xi_{\mathscr{W},\mathscr{B},\Psi}$ be a $K$-layer deep spectral graph convolutional architecture. Assume in each layer $1\leq \ell \leq K$ that $\sum_i \|W^\ell_i\|\leq W$ and $\|B^\ell\|\leq B$ with weight- and bias-matrices $\{W^\ell,B^\ell\}$. Assume the non-linearity $\rho$ is $1$-Lipschitz. Choose  a constant $C$  such that $C \geq \|\Psi_i(L_1)\|, \|\Psi_i(L_2)\|$ for all Laplacian propagation matrices $\Psi_i(L_1),\Psi_i(L_2)$. W.l.o.g. assume $CW >1$. With this, we have with $\delta = \max_{i \in I}\{\|\Psi_i(L_1) - \Psi_i(L_2)\|\} $ and input node-feature matrix $X$, that 
	\begin{equation}
		\|\Xi_{\mathscr{W},\mathscr{B},\Psi}(L,X ) - \Xi_{\mathscr{W},\mathscr{B},\Psi}(\widetilde{L},X )\| \leq \left[K\cdot C^{K}W^{K-1}\cdot \left( \|X\| + \frac{1}{CW - 1} B\right)\right]\cdot \delta.
	\end{equation}
\end{Thm}

\begin{proof}
	For simplicity in notation, let us denote the hidden representations in the network corresponding to $L_2$ by $X^{\ell}_i$. With this, we note:
	\begin{align}
		\|X^{K}_1 - {X}_2^{K}\| &\leq \sum\limits_{i \in I} \|\psi_i(L_1) - \psi_i({L_2}) \|\cdot \|X^{K-1}_1\|\cdot \|W_i^K\| + \sum\limits_{i \in I} \|\psi_i({L_2})\|\cdot\|{X}_2^{K-1} - X^{K-1}_1\|\cdot \|W^K_i\|\\
		&\leq \delta W  \|X^{K-1}_1\| + CW \|{X}_1^{K-1} - X^{K-1}_2\|\\
		&\leq \delta W  \|X^{K-1}_1\| + CW\delta \|X^{K-2}_1\| + (CW)^2 \|{X}_2^{K-1} - X_1^{K-1 }\|\\
		&\leq \frac{\delta}{C} \cdot \left(\sum\limits_{\ell = 1}^{K} (CW)^{\ell} \|X^{K-\ell}_1\|\right)\\
		& = \frac{\delta}{C} \cdot \left(\sum\limits_{j = 0}^{K-1} (CW)^{K-j} \|X^{j}_1\|\right)
	\end{align}
	Hence we need to bound the quantity $\|X^{j}_1\|$ in terms of $C,W,B$ and $X$.
	
	We have
	\begin{align}
		\|X^j_1\| & \leq \sum_i \|\psi_i(L_1)\| \cdot \|X^{j-1}_1\| \cdot \|W^{j}_i| + \|B^{J}\|\\
		&\leq CW \|X^{j-1}_1\| + B\\
		&\leq (CW)^2\|X^{j-2}_1\| + CWB + B\\
		&\leq B \left(\sum_{k=0}^{j-1} (CW)^k \right) + (CW)^j \|X\|\\
		& = \begin{cases} 
			B\frac{(CW)^j -1 }{CW -1} + (CW)^j \|X_1\| &; CW \neq 1 \\
			jB + \|X\| &; CW = 1
		\end{cases}.
	\end{align}

	For the case $CW = 1$, we thus find
	\begin{align}
		\|X^{K}_1 - {X}_2^{K}\| &\leq \frac{\delta}{C} \cdot \left(\sum\limits_{j = 0}^{K-1} (jB+\|X\|)\right)\\
		& = \frac{\delta}{C}\cdot\left(K\|X\| + B\frac{K(K-1)}{2}\right).
	\end{align}
	
	For the case $CW \neq 1$, we find
	\begin{align}
		\|X^{K}_1 - {X}_2^{K}\| &\leq \frac{\delta}{C} \cdot \left(\sum\limits_{j = 0}^{K-1} (CW)^{K-j} \left[B\frac{(CW)^j -1 }{CW -1} + (CW)^j \|X\|\right]\right)\\
	\end{align}
	For $CW > 1$, we may further estimate this as
	\begin{align}
		\|X^{K}_1 - {X}_2^{K}\| &\leq \frac{\delta}{C} \cdot \left(\sum\limits_{j = 0}^{K-1} (CW)^{K-j} \left[B\frac{(CW)^j -1 }{CW -1} + (CW)^j \|X\|\right]\right)\\
		&\leq \delta \cdot \frac{K(CW)^K}{C} \left[\frac{B}{CW - 1} + \|X\|\right].
	\end{align}
	This proves the claim.
\end{proof}

Next we note that we can estimate $\delta $ in terms of $\max_i \{\int_0^\infty \hat{\psi}_i(t) \|e^{-tL_1} - e^{t L_2}\| dt \}$.

\begin{lemma}\label{lem:delta_est}
	We have
	\begin{equation}
		\delta \leq \max_i \left\{\int_0^\infty \hat{\psi}_i(t) \|e^{-tL_1} - e^{t L_2}\| dt \right\}.
	\end{equation}
\end{lemma}
\begin{proof}
We note that by definition (c.f. \ref{eq:LT_integral}), we have
	\begin{equation}
	\psi_i(L) := \int_0^\infty e^{-tL}\hat{\psi}_i(t) dt.
\end{equation}
From this we may infer
\begin{align}\label{eq:key_drawout}
	\|\psi_i(L_1) - \psi_i(L_2)\| = \left\| \int_0^\infty \hat{\psi}_i(t) \left[e^{-tL_1} - e^{-tL_2}\right]dt\right\| \leq \int_0^\infty \hat{\psi}_i(t) \left\| e^{-tL_1} - e^{-tL_2}\right\| dt .
\end{align}
Taking the maximum over $i$ indexing all propagation matrices. yields the claim.

\end{proof}

Combining Lemma \ref{lem:delta_est} with Theorem \ref{gogginstheoremII} then establishes a bound on node-wise latent embeddings in terms of the heat kernel distance $\left\| e^{-tL_1} - e^{-tL_2}\right\|$ for spectral Laplace transform based methods.

\subsubsection{Message passing based Laplace Transform Methods:}
We next establish similar results for Message passing networks. To this end, we first introduce some notation. We continue to denote the node-feature matrix by $X$. As discussed in (\ref{eq:sc_m})\&(\ref{eq:sc_agg}), we have as layer-wise update the implemented the scheme
	\begin{align}
	m_{vu}^{(\ell)} &= [\psi(L)]_{vu} \cdot   \phi(h_v^{(k)}, h_u^{(k)}) \label{eq:sc_m_app}\\
	h_v^{(\ell+1)} &= \sum_{u \in G}  m_{vu} \label{eq:sc_agg_app}.
\end{align}

Denote the feature dimension of layer $\ell$ by  $d_\ell$. Hence we have  $h^{(\ell)}_u \in \mathbb{R}^{\ell}$ and $\phi(h^{(\ell)}_v, h^{(\ell)}_u) \in \mathbb{R}^{d_{\ell +1}}$.

As shorthand notation, we collect the latent representations in layer $(\ell)$ into a feature matrix $X^\ell$ whose columns are given as
\begin{align}
	[X^\ell]_{u:} = h^{(\ell)}_u.
	\end{align}

We then define the map $\Phi: \mathbb{R}^{N\times d_\ell} \rightarrow  \mathbb{R}^{N\times N \times d_{\ell +1}}$ componentwise as 
\begin{align}
	[\Phi(X^\ell)]_{vu:} =\phi([X^{\ell}]_{v:},[X^{\ell}]_{u:}) \equiv  \phi(h^{\ell}_{v},h^{\ell}_{u}).
\end{align}
With this we find
\begin{align}
	[X^{\ell + 1}]_{v:} =  \sum_u [\psi(L)]_{vu}	\ [\Phi(X^\ell)]_{vu:}.
\end{align}
This now allows us to prove:

\begin{theorem}\label{thm:node_mpnn_stab} 
	Suppose we are given two graphs $G_1, G_2$ defined on a common node set. Denote the collection of node-wise latent representations generated by a $K$-layer deep Laplace transform based message passing network $\Xi_{\psi, \phi}$ by $X^K_1, X^K_2$ respectively. 
	
	Suppose there is a constant 
	$C$ so that $C \geq \|\psi(L_1)\|, \|\psi(L_2)\| $.
	Assume that the message function $\phi $ is Lipschitz contiuous:
	\begin{align}
		\|\phi([X_1]_{v:},[X_1]_{u:} ) - \phi([X_2]_{v:},[X_1]_{u:} )\| \leq \frac{L}{4} \left(\| [X_1]_{v:}, - [X_2]_{v:} \| + 
		\| [X_1]_{u:}, - [X_2]_{u:} \| \right).
	\end{align}
Further asume that $\Phi $ is bounded, with $\|\Phi(\cdot)\| \leq D$. Then there is a constant $C_\Xi$ so that 
	\begin{align}
		\|X_1^K - X_2^K\| \leq C_\Xi \cdot \|\psi(L_1) - \psi(L_2)\|.
	\end{align}
\end{theorem}
\begin{proof}
We note:
\begin{align}
	&[X^K_1]_{v:}-[X^K_2]_{v:}
	= \sum_u [\psi(L_1)]_{vu}	\ [\Phi(X^{K-1}_1)]_{vu:} -  \sum_u [\psi(L_2)]_{vu}	\ [\Phi(X^{K-1}_2)]_{vu:}\\
	=& \left[  \sum_u \left( [\psi(L_1)]_{vu} - \psi(L_2)]_{vu} \right)	\ [\Phi(X^{K-1}_1)]_{vu:}   \right]	
	   +\left[\sum_u [\psi(L_2)]_{vu} \left(  [\Phi(X^{K-1}_1)]_{vu:}      - [\Phi(X^{K-1}_2)]_{vu:}  \right) \right].
\end{align}

With this, we find
\begin{align}
	&\|X^K_1-X^K_2\|\\
	 \leq 
	 &\left( \sum_v    \left\|  \sum_u \left( [\psi(L_1)]_{vu} - \psi(L_2)]_{vu} \right)	\ [\Phi(X^{K-1}_1)]_{vu:}   \right\|^2	   \mu_v   \right)^\frac12\\
	 +&
	 \left(\sum_v \left\|\sum_u [\psi(L_2)]_{vu} \left(  [\Phi(X^{K-1}_1)]_{vu:}      - [\Phi(X^{K-1}_2)]_{vu:}  \right) \right\|^2\mu_v \right)^\frac12
\end{align}
with the node weights  (c.f. Section \ref{sec:graph_description}) denoted by $\{\mu_v\}_v$.\\
\ \\
To bound the first term, we note:
\begin{align}
	&\left( \sum_v    \left\|  \sum_u \left( [\psi(L_1)]_{vu} - \psi(L_2)]_{vu} \right)	\ [\Phi(X^{K-1}_1)]_{vu:}   \right\|^2	   \mu_v   \right)^\frac12\\
\leq &D\cdot  \left(\sum_{u,v} \left|[\psi(L_1)]_{vu} - \psi(L_2)]_{vu}\right|^2 \mu_v\right)^\frac12\\
\lesssim & D \|\psi(L_1) - \psi(L_2)\| \cdot \left( \sum_v \mu_v\right)^\frac12\\
= & D\ \sqrt{\mu(G)} \ \|\psi(L_1) - \psi(L_2)\|,
\end{align}
Here we have bounded $\left|[\psi(L_1)]_{vu} - \psi(L_2)]_{vu}\right| \lesssim \|\psi(L_1) - \psi(L_2)\|$.\footnote{We are using the standard notation '$\lesssim $' to denote smaller than, up to a fixed 'universal' constant.}

Let us next bound the second term. We have

\begin{align}
	 &\left(\sum_v \left\|\sum_u [\psi(L_2)]_{vu} \left(  [\Phi(X^{K-1}_1)]_{vu:}      - [\Phi(X^{K-1}_2)]_{vu:}  \right) \right\|^2\mu_v \right)^\frac12\\
\leq&C \cdot\left(\sum_{u,v} \left\| \left(  [\Phi(X^{K-1}_1)]_{vu:}      - [\Phi(X^{K-1}_2)]_{vu:}  \right) \right\|^2\mu_v \mu_u \right)^\frac12\\
\leq& C L \left(\sum_u \mu_u \right)^\frac12 \cdot \left(\sum_u 	\| [X^{k-1}_1]_{u:}, - [X^{K-1}_2]_{u:} \|^2 \mu_u \right)^\frac12. \\
=& C L \sqrt{\mu(G)} \cdot \|X^{K-1}_1-X^{K-1}_2\|. 
\end{align}

Putting the two terms together yields the iteration step 
\begin{align}
	\|X^K_1-X^K_2\| \leq C_A \|\psi(L_1) - \psi(L_2)\| + C_B  \|X^{K-1}_1-X^{K-1}_2\|,
\end{align}
with constants $C_A, C_B$ as determined above. Iterating this through to $K = 0$ yields the following:

Let \(d_K := \|X_1^K - X_2^K\|\) and \(\Delta := \|\psi(L_1)-\psi(L_2)\|\).  
Assume the recursive bound
\[
d_K \le C_A \Delta + C_B\, d_{K-1}, \qquad K \ge 1.
\]
Furthermore we have the intial condition \(d_0 = \|X_1^0 - X_2^0\| =  \|X - X\|= 0\).

Iterating the inequality yields
\[
\begin{aligned}
	d_K 
	&\le C_A \Delta + C_B d_{K-1} \\
	&\le C_A \Delta + C_B\big(C_A \Delta + C_B d_{K-2}\big) \\
	&= C_A \Delta (1 + C_B) + C_B^2 d_{K-2} \\
	&\;\;\vdots \\
	&\le C_A \Delta \sum_{j=0}^{K-1} C_B^{\,j} + C_B^{\,K} d_0.
\end{aligned}
\]

Using \(d_0 = 0\), we obtain
\[
\|X_1^K - X_2^K\|
\le C_A \,\|\psi(L_1)-\psi(L_2)\| \sum_{j=0}^{K-1} C_B^{\,j}.
\]

Evaluating the geometric series gives the closed-form bound
\[
\|X_1^K - X_2^K\|
\le
\begin{cases}
	C_A \,\|\psi(L_1)-\psi(L_2)\| \dfrac{1 - C_B^{\,K}}{1 - C_B}, & C_B \neq 1, \\[1.2em]
	K\, C_A \,\|\psi(L_1)-\psi(L_2)\|, & C_B = 1.
\end{cases}
\]

In particular, if \(0 \le C_B < 1\), then
\[
\|X_1^K - X_2^K\| \le \dfrac{C_A}{1-C_B}\,\|\psi(L_1)-\psi(L_2)\|,
\]
which provides a uniform bound independent of \(K\).

\end{proof}

Combining the above Theorem with (\ref{eq:key_drawout}) then yields

\begin{corollary}
	In the setting of Theorem \ref{thm:node_mpnn_stab} we have
	\begin{align}
		\|X_1^K - X_2^K\| \leq C_\Xi \cdot \int_0^\infty \hat{\psi}(t) \left\| e^{-tL_1} - e^{-tL_2}\right\| dt = C_\Xi \max_i\{\int_0^\infty \hat{\psi}_i(t) \left\| e^{-tL_1} - e^{-tL_2}\right\| dt\}.
	\end{align}
\end{corollary}

Here the last equality follows since we are taking the maximum over all used propagation matries, and only ever use a single one.

Hence we have established a bound on variations of node-wise latent embeddings in terms of the heat kernel distance $\left\| e^{-tL_1} - e^{-tL_2}\right\|$ for spectral Laplace transform based methods.

\subsection{Graph Level Stability and Proof of Corrolary~\ref{cor:main}}\label{app:graph_level_stab}

Next we promote the node-level stability results we derived in Subsection \ref{app:node_level_stab} above to graph level stability results:

We may summarize Theorems \ref{gogginstheoremII}~\&~\ref{thm:node_mpnn_stab} as: Given an input node-feature matrix $X$ on graphs $G_1, G_2$ sharing a common node set, there exists cosntants $C$ so that the output variation $\|X_1^K - X_2^K\|$ of $K$-layer deep Laplace transform based GNNs may be bounded as
\begin{align}\label{eq:fund_ineq}
	\|X_1^K - X_2^K\| \leq C \max_i \left\{\int_0^\infty \hat{\psi}_i(t) \|e^{-tL_1} - e^{t L_2}\| dt \right\}.
\end{align}

We are now interested in how these stability results translate to graph level features. To this end we first specify our graph level aggregation method $\Omega$:

	\begin{Def}\label{def:gl_agg}
At the final ($K^{\text{th}}$-)layer	We aggregate  node-level embeddings  $X^K \in \mathds{R}^{N \times d_K}$ of individual nodes to graph-embeddings $\Omega(X) \in \mathds{R}^{d_K}$  as $\Omega(X^K)_j = \sum_{i = 1}^N |X^K_{ij}|\cdot\mu_i$.	Here $\{\mu_i\}_i$ is the set of node-weights (c.f. Section \ref{sec:graph_description}).
\end{Def}

Graph level features $F_1, F_2$ are then generates as $F_1 = \Omega(X^K_1)$ and $F_2 = \Omega(X^K_2)$.

With this we find:

\begin{theorem}
	In the setting of (\ref{eq:fund_ineq}), with graph level aggregation as in Definition \ref{def:gl_agg}, we find
	\begin{align}
		\|F_1 - F_2\| \leq C \max_i \left\{\int_0^\infty \hat{\psi}_i(t) \|e^{-tL_1} - e^{t L_2}\| dt \right\}.
	\end{align}
\end{theorem}
\begin{proof}
	We only have to note that the aggregation method $\Omega$ is $1$-Lipschitz (as a consequence of the reverse triangle inequality). Hence
	\begin{align}
		\|F_1 - F_2\| = 	\|\Omega(X_1^K) - \Omega(X_2^K)\|  \leq 	\|X_1^K - X_2^K\|
	\end{align}
	and the claim follows from \ref{eq:fund_ineq}.
\end{proof}
This then proves  Corrolary~\ref{cor:main}.

\section{Stability of latent representations generated by Laplace-transform based GNNs for graphs on differing resolution scales}\label{app:coarsification_prooc}

In this section, we establish bounds on differences in latent representations generated for Laplace-transform based GNNs when confronted with graphs describing the same underlying object at multiple resolutions. We briefly recall the setting:

\begin{figure}[H]
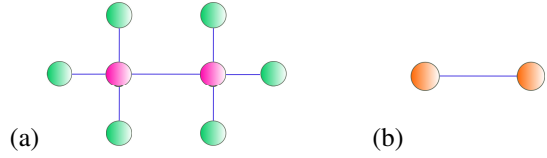
\centering
	(a) \includegraphics[scale=0.4]{figures/Ethan_Original} \ \ \ \ \ \ \ \ \ \ \ \
	(b) \includegraphics[scale=0.45, trim= 0 -45 0 0]{figures/Ethan_Collapsed}
	\caption{(a) 
		Original graph
		$G$ (b) Coarse grained $\underline{G}$ } 
	\label{fig:collapse_weighted_0_app}
	\vspace{-2mm}
\end{figure}

We have a high resolution graph $G$ with associated Laplacian $L$ and node-feature matrix $X$. We also have a lower resolution graph $\underline{G}$, with associated Laplacian $\underline{L}$ and node-feature matrix $\underline{X}  := J^\downarrow X$ arising from the original node feature matrix $X$ on $G$ via projection to $\underline{G}$. Here we have made use of the projection opreator $J^\downarrow$ introduced in Section \ref{subsec:graph_to_kernel} which averages node information over clusters that are condensed into supernodes. Interpolation $J^\uparrow$ instead assings information at each supernode in $\underline{G}$ to every node making up the corresponding cluster in $G$.

We have \citep{koke2026b} (c.f. also Appendix \ref{coarse_grain_proofs}) that $J^\downarrow J^\uparrow = I_{\underline{G}}$ (i.e. equal to the identity matrix on $\underline{G}$).

\subsection{Node Level Results }

We are then initially interested in the following question: Suppose we have a node feature matrix $X$ on the graph $G$. We can generate node-level latent embeddings $\Xi(X)$ by feeding this node-wise information into a (Laplace-transform based) node-level GNN $\Xi$. How different is this outcome from projecting $X$ onto the coarser graph $\underline{G}$, running the GNN there to generate node-level embeddings $\Xi(J^\downarrow X)$ and then interpolating the generated embeddings back up to $G$ via the interpolation operator $J^\uparrow$. That is to say, we are interested in the difference $\|\Xi(X) - J^\uparrow \Xi(J^\downarrow X)\|$.

\subsubsection{Spectral Laplace-transform methods}\label{app:spectral_coarse_fine_results}
We begin by establishing $\|\Xi(X) - J^\uparrow \Xi(J^\downarrow X)\| \leq C \cdot \max\limits_{i} \left\{  \int_0^\infty   |\hat{\psi}_i(t)| \cdot \|e^{-tL_\omega} - J^\uparrow e^{- t\underline{L}} J^\downarrow\|  \right\}$ for spectral methods:

\begin{Thm}\label{transferability_thoerem_app}
Let $\Xi_{\mathscr{W},\mathscr{B},\Psi}$ be a $K$-layer deep 
Laplace transform based network.
Assume that $\sum_{i \in I} \|W^\ell_i\|\leq W$ and $\|B^k\|\leq B$. Choose  $C \geq \|\Psi_i(L)\|,\|\Psi_i(\tilde{L})\| $ ($i \in I$) and w.l.o.g. assume $CW >1$.
Set 	$\max_{i \in I}\{\|\psi_i(L) - {J^\uparrow}\psi_i(\underline{L})J^\downarrow\|\} = \delta  $. 			With this, we have that
\begin{equation}
	\|\Xi_{\mathscr{W},\mathscr{B},\Psi}(X ) - {J^\uparrow}\Xi_{\mathscr{W},\mathscr{B},\Psi}(J^\downarrow X )\| \leq \left[K\cdot C^{K}W^{K-1}\cdot \left( \|X\| + \frac{1}{CW - 1} B\right)\right]\cdot\delta.
\end{equation}

\end{Thm}

\begin{proof}
Let us define 
\begin{equation}
	\underline{X} := J^\downarrow X.
\end{equation}
Let us further use the notation $\underline\psi_i:= \psi_i(\underline{L})$ and $\psi_i:= \psi_i(L)$.

Denote by $X^\ell$ and $\underline{X}^\ell$ the (hidden) feature matrices generated in layer $\ell$ for networks based on  $\psi_i$ and $\underline{\psi}_i$ respectively: I.e. we have 
\begin{equation}
	X^\ell = \rho\left( \sum\limits_{i\in I} \psi_i X^{\ell-1}W^\ell_i + B^\ell\right)
\end{equation}
and
\begin{equation}
	\underline{X}^\ell =\rho\left( \sum\limits_{i \in I} \underline{\psi}_i \underline{X}^{\ell-1}W^\ell_i + \underline{B}^\ell\right).
\end{equation}
Since bias terms are proportional to constant vectors on the graphs, we have
	\begin{equation}
		J^\downarrow B = \underline{B}
		\end{equation}
	and 
	\begin{equation}\label{bias_up}
		J^\uparrow  \underline{B} = B
		\end{equation}
	for bias matrices $B$ and $\underline{B}$ in networks deployed on $G$ and $\underline{G}$ respectively.

We then have

\begin{align}
	&\|\Xi_{\mathscr{W},\mathscr{B},\Psi}(X ) - {J^\uparrow}\Xi_{\mathscr{W},\mathscr{B},\Psi}(J^\downarrow X )\|\\
	=&\|X^K - {J^\uparrow} \underline{X}^K\|\\
	=&\left\| \rho\left( \sum\limits_{i \in I} \psi_i X^{K-1}W^K_i + B^K\right) - {J^\uparrow}\rho\left( \sum\limits_{i \in I} \underline{\psi}_i \underline{X}^{K-1}W^K_i + \underline{B}^L\right) \right\|\\
	=&\left\| \rho\left( \sum\limits_{i \in I} \psi_i X^{K-1}W^K_i + B^K\right) - \rho\left({J^\uparrow} \sum\limits_{i \in I} \tilde{\psi}_i \underline{X}^{K-1}W^K_i + B^L\right) \right\|\\
\end{align}
Here we used the fact that $\rho$ and ${J^\uparrow}$ commute (since $J^\uparrow$ is an assignment matrix and has all non-zero entries equal to one).
Using the fact that the non-linearity $\rho(\cdot)$ is $1$-Lipschitz-continuous we can establish
\begin{align}
	&\|\Xi_{\mathscr{W},\mathscr{B},\Psi}(X ) - {J^\uparrow}\Xi_{\mathscr{W},\mathscr{B},\Psi}(J^\downarrow X )\|\\
	\leq &\left\| \rho\left( \sum\limits_{i \in I} \psi_i X^{K-1}W^K_i + B^K\right) - \rho\left({J^\uparrow} \sum\limits_{i \in I} \underline{\psi}_i \underline{X}^{K-1}W^K_i + B^K\right) \right\|\\
	\leq &\left\|  \sum\limits_{i \in I} \psi_i X^{K-1}W^K_i + B^K - {J^\uparrow} \sum\limits_{i \in I} \underline{\psi}_i \underline{X}^{K-1}W^K_i + B^K \right\|.
\end{align}
Since  $J^\downarrow J^\uparrow =  Id_{\underline{G}}] $, we have
\begin{align}
	&\|\Xi_{\mathscr{W},\mathscr{B},\Psi}(X ) - {J^\uparrow}\Xi_{\mathscr{W},\mathscr{B},\Psi}(J^\downarrow X )\|\\
	\leq&\left\|  \sum\limits_{i \in I} \psi_i X^{K-1}W^K_i  -  \sum\limits_{i\in I} ( J^\uparrow \underline{\psi}_iJ^\downarrow)  J^\uparrow \underline{X}^{K-1}W^K_i  \right\| 
\end{align}
From this, we find	(using $\|J^\uparrow\|,\|J^\downarrow\|\leq1$ ), that

\begin{align}
	&\|\Xi_{\mathscr{W},\mathscr{B},\Psi}(X ) - {J^\uparrow}\Xi_{\mathscr{W},\mathscr{B},\Psi}(J^\downarrow X )\|\\
	\leq&\left\|  \sum\limits_{i \in I} \psi_i X^{K-1}W^K_i  -  \sum\limits_{i\in I} ( J^\uparrow \underline{\psi}_iJ^\downarrow)  J^\uparrow \underline{X}^{K-1}W^K_i  \right\|\\
	\leq &\left\|  \sum\limits_{i \in I} (\psi_i -   J^\uparrow \underline{\psi}_iJ^\downarrow) X^{K-1}W^K_i  \right\| + \sum\limits_{i \in I}\|{J^\uparrow}  \underline{\psi}_i J^\downarrow \|\cdot\| J^\uparrow \underline{X}^{K-1} - X^{K-1}\|\cdot\|W^K_i\|\\
	\leq &\left\|  \sum\limits_{i \in I} (\psi_i -  J^\uparrow \underline{\psi}_iJ^\downarrow) X^{K-1}W^K_i  \right\| + CW\cdot\| J^\uparrow \underline{X}^{K-1} - X^{K-1}\\
	\leq &  \sum\limits_{i \in I} \left\| (\psi_i -   J^\uparrow \underline{\psi}_iJ^\downarrow)\right\| \cdot \left\| X^{K-1}\right\| \cdot \left\|W^K_i  \right\| + CW\cdot\|\tilde J^\downarrow \underline{X}^{K-1} - X^{K-1}\|\\
	\leq &  \delta \cdot \left\| X^{K-1}\right\| W + CW\cdot\|\tilde J^\uparrow{\underline{X}}^{K-1} - X^{K-1}\|\\
\end{align}

Arguing as in the proof of Theorem \ref{gogginstheoremII} then yields the claim.

\end{proof}

\subsubsection{Message passing based Laplace Transform Methods:}
We next establish similar results for message passing networks. Sticking with the notation of the preceding Subsection \ref{app:spectral_coarse_fine_results}, we are hence interested in bounding $X^K - J^\uparrow \underline{X}^K$, with $\underline{X}$ being the node feature matrix generated after a $K$-layer Laplace transform based MPNN.

\begin{theorem}\label{thm:node_mpnn_stab_cross}

	Assume there is a constant 
	$C$ so that $C \geq \|\psi(L)\|, \|\psi(\underline{L})\| $.
	Assume that the message function $\phi $ is Lipschitz contiuous:
	\begin{align}
		\|\phi([X_1]_{v:},[X_1]_{u:} ) - \phi([X_2]_{v:},[X_1]_{u:} )\| \leq \frac{L}{4} \left(\| [X_1]_{v:}, - [X_2]_{v:} \| + 
		\| [X_1]_{u:}, - [X_2]_{u:} \| \right).
	\end{align}
	Further asume that $\Phi $ is bounded, with $\|\Phi(\cdot)\| \leq D$. Then there is a constant $C_\Xi$ so that 
	\begin{align}
		\|X^K - J^\uparrow \underline{X}^K\| \leq C_\Xi \cdot\int_0^\infty \hat{\psi}(t) \|e^{-tL} - J^\uparrow e^{-t\underline{L}} J^\downarrow\|.
	\end{align}
\end{theorem}
\begin{proof}
	We note:
\begin{align}
X^K_{v:} - [J^\uparrow \underline{X}^K]_{v:} = 	X^K_{v:} - \sum_{a \in \underline{G}} J^\uparrow_{va}  \underline{X}^K_{a:} 
= 
\sum_{u \in G}[\psi(L)]_{vu} [\Phi(X^{K-1})]_{vu:} - \sum_{a,b \in \underline{G}} J^\uparrow_{va}  [\psi(\underline{L})]_{ab} [\Phi( \underline{X}^{K-1})]_{ab:} 
\end{align}
Creatively adding zero yields

\begin{align}
	&X^K_{v:} - [J^\uparrow \underline{X}^K]_{v:}\\
=& 
\sum_{u \in G}[\psi(L)]_{vu} [\Phi(X^{K-1})]_{vu:}  - \sum_{u \in G}[\psi(L)]_{vu} [\Phi(J^\uparrow \underline{X}^{K-1})]_{vu:}\\
+& \sum_{u \in G}[\psi(L)]_{vu} [\Phi(J^\uparrow \underline{X}^{K-1})]_{vu:} - \sum_{a,b \in \underline{G}} J^\uparrow_{va}  [\psi(\underline{L})]_{ab} [\Phi( \underline{X}^{K-1})]_{ab:} .
\end{align}

\ \\
For the first term, we note 
\begin{align}
&\left(\sum_{v \in G}	\left\|\sum_{u \in G}[\psi(L)]_{vu} [\Phi(X^{K-1})]_{vu:}  - \sum_{u \in G}[\psi(L)]_{vu} [\Phi(J^\uparrow \underline{X}^{K-1})]_{vu:} \right\|^2 \mu_v\right)^\frac12 \\
\leq & C \cdot \left(\sum_{v,u \in G}	\sum_{u \in G}\left\| [\Phi(X^{K-1})]_{vu:}  -[\Phi(J^\uparrow \underline{X}^{K-1})]_{vu:} \right\|^2 \mu_v\right)^\frac12\\
\leq& C L \left(\sum_u \mu_u \right)^\frac12 \cdot \left(\sum_u 	\| [X^{K-1}]_{u:}, - [J^\uparrow \underline{X}]_{u:} \|^2 \mu_u \right)^\frac12. \\
=& C L \sqrt{\mu(G)} \cdot \|X^{K-1}-J^\uparrow \underline{X}^{K-1}\|. 
\end{align}

Hence let us consider the second term.  By Lemma \ref{lem:index_fight} below, we then  have

\begin{align}
&\sum_{v \in G} \mu_g\left\| \sum_{u \in G}[\psi(L)]_{vu} [\Phi(J^\uparrow \underline{X}^{K-1})]_{vu:} - \sum_{a,b \in \underline{G}} J^\uparrow_{va}  [\psi(\underline{L})]_{ab} [\Phi( \underline{X}^{K-1})]_{ab:} \right\|^2\\
=&\sum_{v \in G} \mu_g \left\| \sum_{u \in G}[\psi(L)]_{vu}  \sum_{a,b \in \underline{G}} J^\uparrow_{va} J^\uparrow_{ub} [\Phi( \underline{X}^{K-1})]_{ab:} - \sum_{a,b \in \underline{G}} J^\uparrow_{va}  [\psi(\underline{L})]_{ab} [\Phi( \underline{X}^{K-1})]_{ab:}\right\|^2\\
=&\sum_{v \in G} \mu_g \left\|\sum_{a,b \in \underline{G}} J^\uparrow_{va} [\psi(L)J^\uparrow]_{vb}  [\Phi( \underline{X}^{K-1})]_{ab:} - \sum_{a,b \in \underline{G}} J^\uparrow_{va}  [\psi(\underline{L})]_{ab} [\Phi( \underline{X}^{K-1})]_{ab:}\right\|^2\\
=&\sum_{v \in G} \mu_g \left\|\sum_{a,b \in \underline{G}} J^\uparrow_{va} [\psi(L)J^\uparrow]_{vb}  [\Phi( \underline{X}^{K-1})]_{ab:} - \sum_{a,b \in \underline{G}} J^\uparrow_{va}  [J^\uparrow\psi(\underline{L})]_{vb} [\Phi( \underline{X}^{K-1})]_{ab:}\right\|^2\\
\leq  & D^2 \sum_{v \in G} \mu_g \left(  \sum_{a,b \in \underline{G}} J^\uparrow_{va} |[\psi(L)J^\uparrow]_{vb} - [J^\uparrow\psi(\underline{L})]_{vb}|   \right)^2\\
\lesssim  & D^2 \sum_{v \in G} \mu_g \left(  \sum_{a \in \underline{G}} J^\uparrow_{va} \|\psi(L)J^\uparrow - J^\uparrow\psi(\underline{L})\|   \right)^2\\
\leq& D^2 \mu(G) \|\psi(L)J^\uparrow - J^\uparrow\psi(\underline{L})\| ^2\\
\lesssim & D^2 \mu(G)  \|\psi(L) - J^\uparrow\psi(\underline{L}) J^\downarrow\|^2.
\end{align}
Here the third step follows from the fact that $J^\uparrow_{va}  [J^\uparrow\psi(\underline{L})]_{vb} = J^\uparrow_{va}  [\psi(\underline{L})]_{ab}$. The last step follows from the fact that $J^\downarrow$ is surjective.

Recursively, we thus have established that 
\begin{align}
	\|X^K- J^\uparrow \underline{X}^K\| \leq C_A \|\psi(L) - J^\uparrow \psi(\underline{L})J^\downarrow \| + C_B  	\|X^{K-1}- J^\uparrow \underline{X}^{K -1}\|.
\end{align}

We may unroll the recursion to obtain
\begin{align}
	\|X^K - J^\uparrow \underline{X}^K\|
	&\le C_A \sum_{i=0}^{K-1} C_B^i \, \|\psi(L) - J^\uparrow \psi(\underline{L})J^\downarrow\|
	+ C_B^K \|X^0 - J^\uparrow \underline{X}^0\|.
\end{align}
If $C_B \neq 1$, the geometric sum yields
\begin{equation}\label{eq:to_bound_eq}
	\|X^K - J^\uparrow \underline{X}^K\|
	\le 
	C_A \frac{1 - C_B^K}{1 - C_B} \, \|\psi(L) - J^\uparrow \psi(\underline{L})J^\downarrow\|
	+ C_B^K \|X^0 - J^\uparrow \underline{X}^0\|.
\end{equation}

Finally we recall from Section \ref{subsec:sc_gnns}, that the initial input $X^0$ into the Laplace-transform message passing layers arises  from the input features $X$ via an initial Laplace-transform propagation step as 
\begin{align}
	X^0 = \psi(L) X.
\end{align}
Thus we also find
\begin{align}
	J^\uparrow \underline{X}^0 = J^\uparrow \psi(\underline{L}) [J^\downarrow X].
\end{align}
From this we may infer that 
\begin{align}
	\|X^0 - J^\uparrow \underline{X}^0 \|  = \|\psi(L) X - J^\uparrow \psi(\underline{L}) J^\downarrow X\| 
	\leq 
	\| \psi(L)  - J^\uparrow \psi(\underline{L}) \| \cdot \|X\|.
\end{align}
Hence both summands in (\ref{eq:to_bound_eq}) may be bounded in terms of $\| \psi(L)  - J^\uparrow \psi(\underline{L}) \| $.
Using 
\begin{align}
	\| \psi(L)  - J^\uparrow \psi(\underline{L}) \|  \leq \int_0^\infty \hat{\psi}(t) \|e^{-tL} - J^\uparrow e^{-t\underline{L}} J^\downarrow\|dt,
\end{align}
we have thus established the claim.
\end{proof}

Finally we provide the Lemma used in the proof above:

\begin{lemma}\label{lem:index_fight}
We have 
\begin{align}
	[\Phi(J^\uparrow \underline{X}^\ell)]_{vu:} = \sum_{a,b \in \underline{G}} J^\uparrow_{va} J^\uparrow_{ub} [\Phi( \underline{X}^\ell)]_{ab:}. 
\end{align}
\end{lemma}
\begin{proof} Indeed, this follows because entries of $J^\uparrow$ satisfy $J^\uparrow_{ua} \in \{0,1\}$ and every node in $G$ gets assigned information exactly from only one node in $\underline{G}$. Thus the sum $\sum_{a,b \in \underline{G}} J^\uparrow_{va} J^\uparrow_{ub} [[\Phi( \underline{X}^\ell)]_{ab:}]$ contains precisely \emph{one} non-zero summand. This summand in turn, corresponds to those indices $a,b \in \underline{G}$ which are mapped to $v$ and $u$ respectively under $J^\uparrow$. Since entries in $J^\uparrow$ are binary (and hence do not diminish or blow up signal if they are non-zero) the claim follows. 
\end{proof}

\subsection{Graph Level Results and Proof of Theorem~\ref{thm:quant_norm_est}}\label{app:graph_level_coarse_fine}

It remains to translate the results of Theorems \ref{transferability_thoerem_app}~\&~\ref{thm:node_mpnn_stab_cross} to the graph level. To this end, we first note the following:

\begin{Lemma}
With the graph level aggregation function $\Omega$ as in Definition \ref{def:gl_agg}, we have
\begin{align}
	\Omega(\underline{X}) = \Omega(J^\uparrow \underline{X}).
	\end{align}
\end{Lemma}
\begin{proof}
	This fact follows immediately from the fact that $J^\uparrow$ assigns the information present at the supernode $a \in \underline{G}$ to all nodes $u$  in the cluster within $G$ that correspond to the supernode a, together with the fact that the node-weight $\mu_a$ of the node $a \in \underline{G}$ is the sum over all of the node weights $\mu_u$ of nodes $u$  in the cluster within $G$ that corresponds to the supernode $a$ under coarse graining (c.f. also (\ref{eq:weight_aggr}) in Appendix \ref{coarse_grain_proofs}).
\end{proof}

Using the above lemma, we find
\begin{align}
\|\underline{F} - F\| = \| \Omega(\underline{X}^K) - \Omega(X^K)\| = \| \Omega(J^\uparrow \underline{X}^K) - \Omega(X^K)\|.
\end{align}
Using the $1$-Lipschitzness of the aggregation function $\Omega$ then yields
\begin{align}
\| \Omega(J^\uparrow \underline{X}^K) - \Omega(X^K)\|\leq \| J^\uparrow \underline{X}^K - X^K\|.
\end{align}
Together with Theorems \ref{transferability_thoerem_app}~\&~\ref{thm:node_mpnn_stab_cross} we have thus proved:
\begin{theorem}
		Let $F$ and $\underline{F}$ be the latent embeddings generated for a graph $G$ its coarsified version $\underline{G}$ by a (spectral or message passing) network employing Laplace transform propagation, as outlined in Section~\ref{subsec:sc_gnns}. With $\{\Psi_i(L) = \int_0^\infty \hat{\psi}_i(t)e^{-tL}dt\}_{i \in I}$ the collection of all Laplace-transform propagation matrices used 
	within the network, we have $			\|F - \underline{F}\| \leq C \cdot \max\limits_{i} \left\{  \int_0^\infty   |\hat{\psi}_i(t)| \cdot \|e^{-tL_\omega} - J^\uparrow e^{- t\underline{L}} J^\downarrow\|  \right\}$.
\end{theorem}

\section{Additional Experimental Considerations}\label{app:global_exp_app}

\subsection{Additional details on initial Coarse Graining Experiment}\label{app:qm7_experiment}

\paragraph{Dataset:}
The  dataset we consider is the \textbf{QM}$\mathbf7$ dataset, introduced in \cite{blum, rupp}. This dataset contains descriptions of $7165$ organic molecules, each with up to seven heavy atoms, with all non-hydrogen atoms being considered heavy. A molecule is represented by its Coulomb matrix $C^{\text{Clmb}}$, whose off-diagonal elements
\begin{equation}
	C^{\text{Clmb}}_{ij}	 = \frac{Z_iZ_j}{|R_i-R_j|}
\end{equation}
correspond to the Coulomb-repulsion between atoms $i$ and $j$. We discard	diagonal entries of Coulomb matrices; which would encode a polynomial fit of atomic energies to nuclear charge \cite{rupp}.

For each atom in any given molecular graph, the individual Cartesian coordinates $R_i$ and the atomic charge $Z_i$ are (in principle) also accessible individually. 
To each molecule an atomization energy - calculated via density functional theory - is associated. The objective is to predict this quantity. The performance metric is mean absolute error. Numerically, atomization energies are negative numbers in the range $-600$ to $-2200$. The associated unit is $[\textit{kcal/mol}]$.

\paragraph{Details on collapsing procedure:}
Again, we make use of the QM$7$ dataset \cite{rupp} and its Coulomb matrix description 
\begin{equation}\label{offdiagII}
	C^{\text{Clmb}}_{ij}	 = \frac{Z_iZ_j}{|R_i-R_j|}
\end{equation}
of molecules.
We modify (all) molecular graphs in QM$7$
by deflecting hydrogen atoms (H) out of their equilibrium positions towards the respective nearest heavy atom. This is possible since the QM$7$ dataset also contains the Cartesian coordinates of individual atoms.
Edge weights between heavy atoms then remain the same, while  Coulomb repulsions  between H-atoms and respective nearest heavy atom increasingly diverge; as is evident from (\ref{offdiagII}).

Given an original molecular graph $G$ with node weights $\mu_i=Z_i$, the corresponding limit graph $\underline{G}$ corresponds to a coarse grained description, where heavy atoms and surrounding H-atoms are aggregated into single super-nodes.

Mathematically, $\underline{G}$ is obtained by removing all nodes corresponding to H-atoms from $G$, while adding the corresponding charges $Z_H = 1$ to the node-weights of the respective 
nearest heavy atom.
Charges 
in (\ref{offdiagII}) are modified similarly to generate the weight matrix $\underline{W}$.

On original molecular graphs, atomic charges are provided via one-hot encodings. For the graph of methane -- consisting of one carbon atom with charge $Z_C =6$ and four hydrogen atoms of charges $Z_H =1$ -- the corresponding node-feature-matrix is e.g. given as
\begin{align}
	X = 
	\begin{pmatrix}
		0 & 0 &\cdots &0 &1 & 0 \cdots \\
		1 & 0 & \cdots&0 &0 & 0 \cdots \\
		1 & 0 & \cdots&0 &0 & 0 \cdots \\
		1 & 0 & \cdots&0 &0 & 0 \cdots \\
		1 & 0 & \cdots&0 &0 & 0 \cdots
	\end{pmatrix}
\end{align}
with the non-zero entry in the first row being in the $6^{\text{th}}$ column, in order to encode the charge $Z_C=6$ for carbon.

The feature vector of an aggregated node represents charges of the heavy atom and its neighbouring H-atoms jointly.

Node feature matrices are translated as $\underline{X} = J^\downarrow X$. 
Applying $J^\downarrow$ to one-hot encoded atomic charges yields (normalized) bag-of-word embeddings on $\underline{G}$:  Individual entries of feature vectors encode how much of the total charge of the super-node is contributed by individual atom-types.
In the example of methane, the limit graph $\underline{G}$ consists of a single node with  node-weight
\begin{equation}
	\mu= 6 + 1 + 1 + 1 + 1 = 10.
\end{equation}
The feature matrix 
\begin{equation}
	\underline{X} = J^\downarrow X
\end{equation}
is  a single row-vector  given as
\begin{align}
	\underline{X} = \left(\frac{4}{10},0,\cdots,0,\frac{6}{10},0,\cdots\right).
\end{align}

\paragraph{Experimental Setup:}
We randomly select $1500$ molecules for testing and train on the remaining graphs. On QM$7$ we run experiments for $5$ different random random seeds and report mean and standard deviation.  All experiments were performed on a single NVIDIA Quadro RTX 8000 graphics card.

\paragraph{Additional details on training and models:}

Typical GNN models are
divided into \textbf{standard} architectures (GCN \citep{Kipf},  ChebNet \citep{Bresson}, GATv2 \citep{Brody0Y22}) GIN (\citep{GIN} (together with its edge-weight enabled version \citep{hu2020strategies} where applicable)) and  \textbf{multi-} \textbf{scale} architectures (UFGNet \citep{Framelets}, Lanczos \citep{LanczosNet}). Apart from UFGNet (already acting as a \textbf{pooling} layer) we also consider self-attention-pooling \citep{gpooling}; both acting on the final layer (SAG) and  
as acting on 
the output of 
each indivifual layer,
with resulting layer-wise features concatenated to produce the final embedding 
(SAG-M).
All considered  convolutional layers are incorporated into a two layer deep and fully connected graph convolutional architecture.
In each hidden layer, we set the width (i.e. the hidden feature dimension) to $	F_1 = F_2 = 64$.
For ChebNet, we set the polynomial order to $K=2$ as is customary. Lanczos uses $20$ Lanczos iterations, as proposed in the original paper \citep{LanczosNet}. UFGNet uses Haar wavelets.
For all baselines, the standard mean-aggregation scheme is employed after the graph-convolutional layers to generate graph level features. Finally, predictions are generated via an MLP.

For the Spectral$_{\text{Res.}}$-architecture, we set $\lambda = 1$ and and build filters using the $k=1$ and $k=2$ atoms in $\Psi^{\text{Res}} = \{(z+\lambda)^{-k}\}_{k\in\mathds{N}}$. For MPNN$_{\text{Res.}}$ we use the $k=1$ atom as $\psi$.

For  the	Spectral$_{\text{Exp.}}$-, we set $t=1$ and and build filters using the $k=1$ and $=2$ atoms in $\Psi^{\text{Exp}} = \{e^{-(kt_0)z}\}_{k\in\mathds{N}}$.
For MPNN$_{\text{Res.}}$ we use the $k=1$ atom as $\psi$.

As aggregation, we employ the graph level feature aggregation scheme discuseed in Definition \ref{def:gl_agg}, with node weights set to atomic charges of individual atoms. Predictions are then generated via a final MLP with the same specifications as the one used for baselines.

\subsection{Further discussions on stability results of Table \ref{tab:qm7_latent_diff} }\label{eta_explanation}

We can further understand the small difference of latent embeddings
using
Theorem \ref{subsec:sc_gnns}. 
Let use use the shorthand notation
\begin{align}
	\eta(t) = \|e^{-Lt} - J^\uparrow e^{-t\underline{L}} J^\downarrow\|.
\end{align}

	\begin{figure}[H]
			\centering
			\includegraphics[scale=.5]{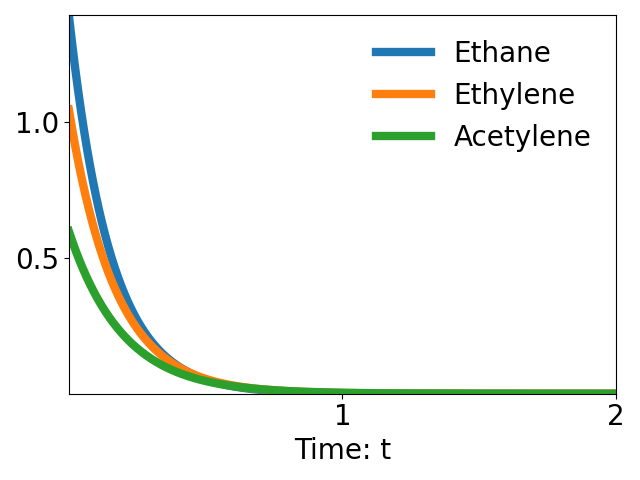}
			\vspace{-2mm}
			\captionof{figure}{ 
					$\|e^{-Lt} - J^\uparrow e^{-t\underline{L}} J^\downarrow\|$ for molecules in QM$7$
				} 
			\label{molecular_drops}
		\end{figure}

	Exemplarily considering 
	exponential propagation	matrices (cf.\ Section \ref{solution})	we have (with $t_k = k$) that $ \int_0^\infty |\hat{\psi}_k(t)| \|e^{-tL} - J^\uparrow e^{- t\underline{L}} J^\downarrow\|  dt =  \|e^{-kL} - J^\uparrow e^{- k\underline{L}} J^\downarrow\| $,
	we thus have $\|F - \underline{F}\| \lesssim \max_{k \geq 1} |\eta(k)|$. 
	Investigating the differences $\eta(t) = \|e^{-tL} - J^\uparrow e^{- t\underline{L}} J^\downarrow\| $ of diffusion flows, we find that $\eta(t)$ drops to zero fast, as exemplarily plotted in Fig. \ref{molecular_drops} for the first few molecules of QM$7$. In particular $|\eta(k)|_{k \geq 1}  \approx 0$.

	Using this as an upper bound in Theorem \ref{subsec:sc_gnns} shows that embeddings $F, \underline{F}$ of graphs describing the same molecule at different resolution scales are similar. This explains the ability to generalize between scales.

\subsection{Transferability on Graphs generated via Stochastic Block Models}\label{sbm_experiments}

\paragraph{Stochastic Block Models:}

Stochastic block models \citep{stochastic_block_models} are generative models for random graphs that produce graphs containing strongly connected communities. 
In our experiments in this section, we consider a stochastic block model whose distributions is characterized by four parameters: The number of communities $c_{\text{number}}$ determine how many (strongly connected) communities are present in the graph that is to be generated. The community size  $c_{\text{size}}$ determines the number of nodes belonging to each (strongly connected) community. The probability $p_{\text{connect}}$ determines the probability that two nodes within the same community are connected by an edge. The probability $p_{\text{inter}}$ determines the probabilities that two nodes in \textit{different} communities are connected by an edge.

\paragraph{Experimental Setup:}
Since stochastic block models do not generate node-features, we equip each node with a randomly-generated unit-norm feature vector. Given such a graph $G$ drawn from a stochastic block model, we then compute a  version $\underline{G}$ of this graph, where all communities are collapsed to single nodes. We then compare the feature vectors generated for $G$ and $\underline{G}$. All experiments were performed on a single NVIDIA Quadro RTX 8000 graphics card.

\paragraph{Experiment: Varying the Connectivity within the Communities:}
We fix the parameters $c_{\text{number}}, c_{\text{size}}$ and $p_{\text{inter}}$ in our stochastic block model. We then vary the probability $p_{\text{connect}}$ that two nodes within the same community are connected by an edge from $p_{\text{connect}} = 0$ to $p_{\text{connect}} = 1$. This corresponds to varying the connectivity within the communities from very sparse (or in fact no connectivity) to full connectivity (i.e. the community being a clique). In Figure \ref{inner_peace} below, we then plot the difference of feature vectors generated by models for $G$ and $\underline{G}$ respectively. For each $p_{\text{connect}} \in [0,1]$, results are averaged over $100$ graphs randomly drawn from the same stochastic block model.

\begin{figure}[H]
	(a)\includegraphics[scale=0.35, trim = 0 -90 0 0]{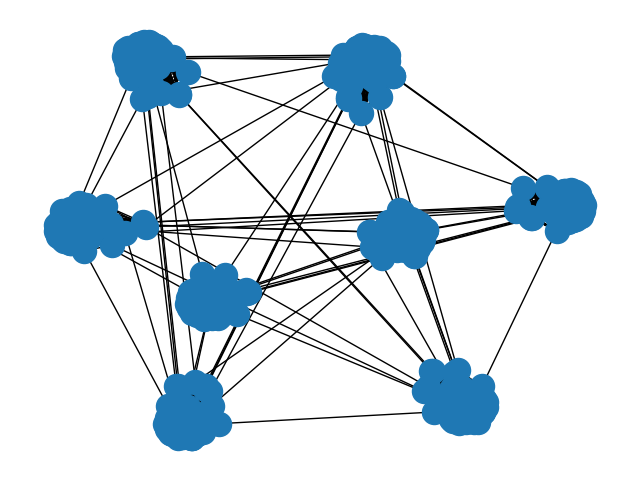}\hfill
	(b)\includegraphics[width=0.5\linewidth]{figures/sbm_final_plot}
	\captionof{figure}{ (a) Example Graph 
		(b) Varying the parameter $p_{\text{connect}} \in [0,1]$ for fixed  $c_{\text{size}} = 60$, $p_{\text{inter}} = 2/c_{\text{size}}^2$   and $c_{\text{number}} = 12$.} 
\label{inner_peace}
\end{figure}

We have chosen $p_{\text{inter}} = 2/c_{\text{size}}^2$ so that -- on average -- \textit{clusters} are connected by two edges. The choice of two edges (as opposed to $1, 3,4,5,...$)  between clusters is not important.

	\subsection{Node Level Transferability and Graphs with Varying Connectivity}\label{imba_geom}

Here, we duplicated individual nodes on popular node-classification datasets (\textsc{Citeseer} \&  \textsc{Cora }\citep{Citeseer, Cora})  $k$-times to form (fully connected) $k$-cliques, while keeping the train-val-test partition constant.
	
	\vspace{-3mm}
	\begin{figure}[h!]
		\centering
		(a)\includegraphics[width=0.3\linewidth, trim= 0 -45 0 0]{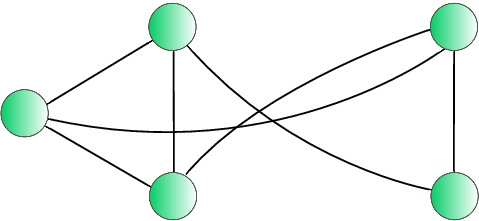}\ \ \ \ \ 
		(b)\includegraphics[width=0.4\linewidth]{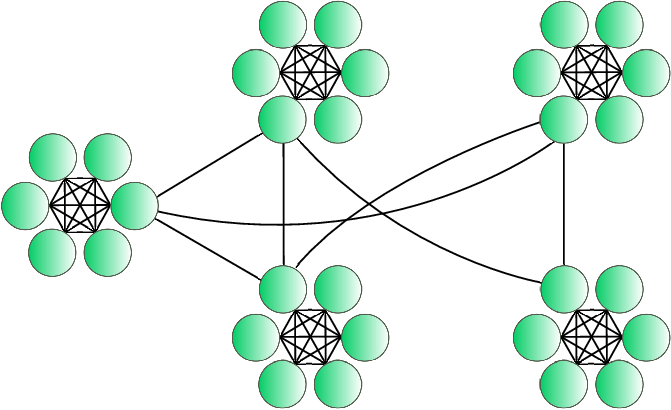}

		\captionof{figure}{Individual nodes (a) replaced by $k$-cliques (b)
		} 
		\label{tease_exp_gs_app}
	\end{figure}

	Models were then trained on the same ($k$-fold expanded) train-set and asked to classify nodes on the ($k$-fold expanded) test-partition. Baselines were chosen to  form a representative selection of common information-propagation methods.

	\paragraph{Additional details on training and models:}\label{cora_citeseer_app}

	All experiments were performed on a single NVIDIA Quadro RTX 8000 graphics card.
	We closely follow the experimental setup of \cite{PredictThenPropagate} on which our codebase for this experiment builds:
	All models are trained for a fixed maximum (and unreachably high) number of $n=10000$
	epochs. Early stopping is performed when the validation performance has not improved for $100$ epochs. Test-results for the parameter set achieving the highest validation-accuracy are then reported. Ties are broken by
	selecting the lowest loss (c.f. \cite{GAT}).
	Confidence intervals are calculated  over multiple splits and  random seeds at the $95\%$ confidence level via bootstrapping.

	We train all models on a fixed learning rate of
	$\text{lr} = 0.1$.
	Global dropout probability $p$ of all models is optimized individually over
	$p  \in \{0.3, 0.35, 0.4, 0.45, 0.5\}$.
	We use $\ell^2$ weight decay and optimize the weight decay parameter $\lambda$ for all models over
	$\lambda \in \{0.0001, 0.0005\}.$
	We choose a two-layer deep convolutional  architecture with the dimensions of hidden features optimized over 
	\begin{equation}\label{hidden_dims}
		K_\ell \in \{32,64,128\}.
	\end{equation}

		In addition to the hyperparemeters specified above, some baselines have additional hyperparameters, which we detail here:
	 ChebNet  uses $K=2$ to avoid the known over-fitting issue \cite{Kipf} for higher polynomial orders. 
		The graph attention network \cite{GAT} uses $8$ attention heads, as suggested in \cite{GAT}.

		For the LTF-models, we optimize  depth over $K=1,2$ with hidden feature dimension optimized over the values in (\ref{hidden_dims}) as for baselines. 
		We empirically observed in the setting of \textit{unweighted} graphs, that rescaling the Laplacian as
$
			L_{nf} := \frac{1}{c_{nf}}L 
	$
		with a normalizing factor $c_{nf}$ 
		on which we base LTF architectures 
		improved performance.
		We express this normalizing factor in terms of the largest singular value $\|\Delta\|$ of the (non-normalized) graph Laplacian. It is then selected among $
			c_{nf}/\|\Delta\| \in \{0.001,0.01,0.1,2\}$.
		The value $\lambda$ for the resolvent is selected among $
			\lambda \in \{0.14,0.15, 0.2, 0.25\}$.

		\subsection{The Setting og Graphs discretizing a common ambient Manifold}\label{app:torus_exp_description}

		\begin{figure}[H]
		\centering
		\includegraphics[scale=0.25]{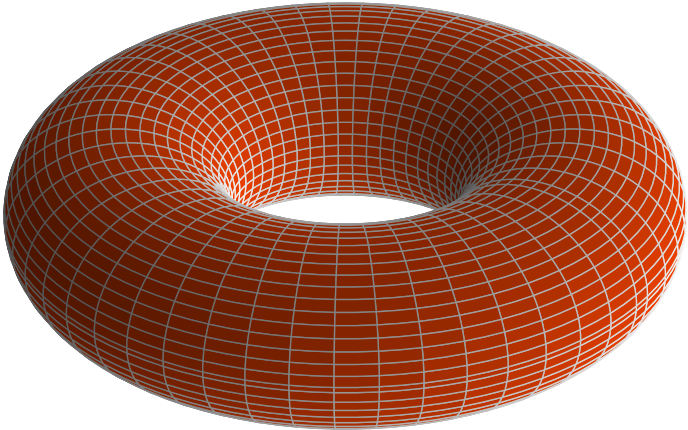}
		\includegraphics[scale=0.25]{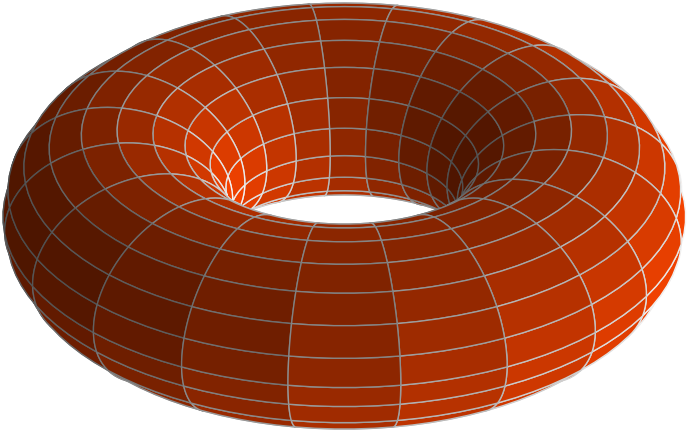}\\
		\captionof{figure}{Distinct Torus Discretizations} 
		\label{fig:torus_app}
	\end{figure}

	The concept of
	operators capturing the geometry of underlying spaces
	also applies to manifolds $\mathcal{M}$, where the Laplace-Beltrami  operator	$\Delta_{\mathcal{M}}$ 	can be thought of as a continuous analogue of the Graph Laplacian \citep{Hein2006GraphLA}. This is hence is a prime setting for studying transferability.
	Counter to
	previous works \citep{DBLP:journals/corr/abs-1907-12972, Wang2021StabilityON}, our  framework here allows to derive transferability guarantees beyond the settings of bandlimited signals and probabalistic guarantees.

	We 
	consider
	the setting of two graphs $G_1, G_2$ discretely approximating the same manifold (c.f. e.g. Fig. \ref{fig:torus_app}). 
	This
	can be made mathematically precise using the concept of generalized norm resolvent convergence (c.f. e.g. \citep{PostBook} for a discussion). Here we note the following: Given projection operators $J^{\downarrow}_i$ mapping from $\mathcal{M}$ to $G_i$
	and interpolation operators $J^{\uparrow}_i$ mapping from $G_i$ to $\mathcal{M}$,	
	we may measure the difference $\|e^{-t\Delta_{\mathcal{M}}} - J_i^{\uparrow}e^{-tL_i}J^{\downarrow}_i\| \leq \delta_i$ in heat diffusion flows on the respective spaces. The fidelity of the discrete approximation is then essentially determined by the size of   $\delta_i \ll 1$.
	As discussed in detail below, we have in this setting:
	\begin{equation}\label{torus_triangle}
		\|e^{-tL_1}- (J_1^\downarrow J_2^{\uparrow})e^{-tL_2} (J^{\downarrow}_2 J_1^\uparrow)\| \lesssim (\delta_1 + \delta_2)
	\end{equation}

	Together with Theorem \ref{thm:quant_norm_est}, this then also implies 
	\begin{align}\label{eq:cauchy}
	\|F_1 - F_2\| \lesssim (\delta_1 + \delta_2)
	\end{align}
		for the generated latent embeddings of the two graph approximations $G_1$ and $G_2$. 
Below we show using the explicit example of the torus, that as the approximating mesh becomes finer and finer (i.e. the number of nodes $N$ increases), the corresponding 	approximation error $\delta$ tends to zero. 
Equation (\ref{eq:cauchy}) then guarantees that the corresponding graph level latent embeddings form a Cauchy sequence. Since the latent space $\mathbb{R}^d$ is complete, we thus have convergence of the latent embeddings $F_N$ (indexed by the mesh size $N$) towards a limit embedding ($F_N \rightarrow F_\infty$), with the limit embedding $F_\infty$ corresponding to the latent embedding our Laplace-transform based model would  generate if we could to deploy it on the continuous manifold $\mathcal{M}$.

Hence let us further discuss the setting of two graphs $G_i$ discretizing the same ambient manifold $\mathcal{M}$ in the sense of
\begin{align}
	\|J^{\uparrow }_ie^{-tL_i}J^{\downarrow}_i- e^{-t\Delta_\mathcal{M}}\| \leq \delta.
\end{align}

We will assume $J^{\downarrow}_iJ_i^{\uparrow} = Id_{G_i}$, which is a justified assumption, as out torus example below elucidates. In this setting, we then have

\begin{align}
	&\|e^{-tL_1} - (J_1^\downarrow J_2^{\uparrow})e^{-tL_2}(J^{\downarrow}_2 J_1^\uparrow)\| \\
	=& \|e^{-tL_1} -    J_1^\downarrow e^{-t\Delta_\mathcal{M}}J_1^{\uparrow}       + J_1^\downarrow(\Delta_\mathcal{M} + Id)^{-1}J_1^{\uparrow}            - (J_1^\downarrow J_2^{\uparrow})e^{-tL_2}(J^{\downarrow}_2 J_1^\uparrow)\|\\
	\leq& \|e^{-tL_1} -    J_1^\downarrow e^{-t\Delta_\mathcal{M}}J_1^{\uparrow}  \| + \| J_1^\downarrow e^{-t\Delta_\mathcal{M}}J_1^{\uparrow}            - (J_1^\downarrow J_2^{\uparrow})e^{-tL_2}(J^{\downarrow}_2 J_1^\uparrow)\|\\
\end{align}

We note
\begin{align}
	&\|e^{-tL_1} -    J_1^\downarrow e^{-t\Delta_\mathcal{M}}J_1^{\uparrow}  \|  \\
	=&  \|J^{\downarrow}_1J_1^{\uparrow} e^{-tL_1 } J^{\downarrow}_1J_1^{\uparrow}  -    J_1^\downarrow e^{-t\Delta_\mathcal{M} }J_1^{\uparrow}  \|\\
	\leq&\|J_1^\downarrow \|\|J_1^{\uparrow} \| \cdot \|e^{-tL_1} - J_1^{\uparrow}e^{-t\Delta_\mathcal{M}}J^{\downarrow}_1\| \lesssim \delta.
\end{align}

We consider:

\begin{align}
	&\| e^{-t\Delta_\mathcal{M}}       - (J_1^\downarrow J_2^{\uparrow})e^{-tL_2}(J^{\downarrow}_2 J_1^\uparrow)\|\\
	\leq&\|J_1^\downarrow \|\|J_1^{\uparrow} \| \cdot \| e^{-t\Delta_\mathcal{M}}       -J_2^{\uparrow}e^{-tL_2}J^{\downarrow}_2 \|\\
	\lesssim& \| e^{-t\Delta_\mathcal{M} } - J_2^{\uparrow}e^{-tL_2}J^{\downarrow}_2 \| \leq \delta.\\
\end{align}

Hence we have indeed established 

\begin{equation}
	\|e^{-tL_1} - (J_1^\downarrow J_2^{\uparrow})e^{-tL_2}(J^{\downarrow}_2 J_1^\uparrow)\| \lesssim 2 \delta.
\end{equation}

Next let us consider an explicit example.

	To this end, let us revisit the torus-setting introduced in Fig. \ref{fig:torus_plot}. 
	To see the origin of the estimate (\ref{torus_triangle}) in the torus setting, we begin by recalling that the standard torus $\mathds{T}$ arises as the cartesian product of two circles $S_1$ of circumference $2\pi$:
	\begin{equation}
		\mathds{T} = S^1 \times S^1.
	\end{equation}
	Let us parametrize these circles via angles $0 \leq \theta_1, \theta_1 \leq 2 \pi$. The Laplacian on $\mathds{T}$ can then be written as
	\begin{equation}
		\Delta_{\mathds{T}} = -\partial_{\theta_1}^2 -\partial_{\theta_2}^2.
	\end{equation}
	A set of corresponding normalized eigenfunctions are given as
	\begin{equation}
		\phi_{k_1,k_2} = \frac{1}{2 \pi} e^{-ik_1\theta_1} e^{-ik_2\theta_2}
	\end{equation}
	with corresponding eigenvalues
	\begin{equation}
		\lambda_{k_1,k_2} = k_1^2 + k_2^2
	\end{equation}
	and $k_1, k_2 \in \mathds{Z}$.
	
	We now consider a regular discretization of $\mathds{T}$ using $N^2$ nodes. This mesh can be thought of as arising from regular discretizations of each $S^1$ factor; with a node being placed at angles $\phi = \frac{2\pi}{N}k$ with $0 \leq k \leq N$. The individual node weight of each node in the mesh discretization of $\mathds{T}$ is set to $\mu = \frac{(2\pi)^2}{N^2}$. We might think of this discretization $\mathds{T}_N$ pf $\mathds{T}$ as arising via a cartesian product of the group $\mathds{Z}/N\mathds{Z}$ (i.e. the group of integers modulo $N$) with itself. Each node of  $\mathds{T}_N = \mathds{Z}/N\mathds{Z} \times \mathds{Z}/N\mathds{Z}$ is then specified by a tuple $(a,b) \in \mathds{T}_N$, with $a\in \mathds{Z}/N\mathds{Z}$ and $b\in \mathds{Z}/N\mathds{Z}$.

	The graph Laplacian $L_N$ on $\mathds{T}_N = \mathds{Z}/N\mathds{Z} \times \mathds{Z}/N\mathds{Z}$ then acts on a scalar node signal $x_{ab}$ as
	\begin{equation}
		(L_N x)_{ab} = \frac{N^2}{(2\pi)^2}\left(4x_{ab} - x_{(a+1)b}- x_{(a-1)b} - x_{a(b+1)} - x_{a(b-1)}  \right).
	\end{equation}
	Henceforth we will adopt the notation $x(a,b) \equiv x_{ab}$. \\
	Normalized eigenvectors for this Laplacian  $L_N$ on $\mathds{T}_N$ are given as 
	\begin{equation}
		\phi_{k_1,k_2}^N = \frac{1}{2\pi}e^{-i\frac{2\pi k_1}{N}a}e^{-i\frac{2\pi k_1}{N}b}
	\end{equation}
	with $0 \leq k_1,k_2 \leq (N-1)$. Corresponding eigenvalues are found to be
	\begin{equation}
		\lambda^N_{k_1,k_2} = \frac{N^2}{\pi^2}\left[\sin^2\left(\frac{\pi}{N}\cdot k_1\right) + \sin^2\left(\frac{\pi}{N}\cdot k_2\right)\right].
	\end{equation}

	To facilitate contact between $\mathds{T}$ and its graph approximation $\mathds{T}_N$, we define an interpolation operator $J^\uparrow_N$ that maps a graph signal $f(a,b)$ defined on $\mathds{T} = \mathds{Z}/N\mathds{Z}\times \mathds{Z}/N\mathds{Z}$ to a function $\overline{f}$ defined on $\mathds{T}$ by defining
	\begin{equation}
		\overline{f}(\theta_1,\theta_2) = f(a,b)
	\end{equation}
	whenever $\frac{2\pi}{N}(a-1) \leq\theta_1 \leq \frac{2\pi}{N}a$ and $\frac{2\pi}{N}(b-1) \leq \theta_2 \leq \frac{2\pi}{N}b$. \\
	We then take $J^\downarrow$ to be the adjoint of $J^\uparrow$ (i.e. $J^\downarrow = (J^\uparrow)^*$. It is not hard to see that $J^\downarrow J^\uparrow = Id_{\mathds{T}_N}$.\\
	We now want to show that (for $t > 0$)
	\begin{align}\label{torus_convergence}
		\|e^{-t\Delta_\mathds{T}} - J^\uparrow e^{-tL_N}J^\downarrow\| \rightarrow 0
	\end{align}
	as $N \rightarrow \infty$. To this end, denote by $P_{k_1,K_2}$ the orthogonal projection onto $\phi_{k_1,k_2}$. Denote by $P^N_{k_1,K_2}$ the orthogonal projection onto $\overline{\phi_{k_1,k_2}^N}$. We note
	\begin{align}
		\|e^{-t\Delta_\mathds{T}} - J^\uparrow e^{-tL_N}J^\downarrow\|  = \left\| \sum_{k_1,k_2 \in \mathds{Z}} e^{-\lambda_{k_1,k_2}t}  P_{k_1,k_2}  -      \sum\limits_{- -\frac{N-1}{2} \leq p_1, p_2 \leq \frac{N-1}{2}}   e^{-\lambda_{k_1,k_2}t} P^N_{p_1,p_2}    \right\|.
	\end{align}
	From this we observe
	\begin{align}
		&\|e^{-t\Delta_\mathds{T}} - J^\uparrow e^{-tL_N}J^\downarrow\|  = \left\| \sum_{k_1,k_2 \in \mathds{Z}} e^{-\lambda_{k_1,k_2}t}  P_{k_1,k_2}  -      \sum\limits_{- -\frac{N-1}{2} \leq p_1, p_2 \leq \frac{N-1}{2}}   e^{-\lambda^N_{p_1,p_2}t} P^N_{p_1,p_2}    \right\|\\
		\leq & \left\|\sum_{\frac{N-1}{2} < |k_1|,|k_2|  } e^{-\lambda_{k_1,k_2}t}  P_{k_1,k_2} \right\| + 
		\left\|  \sum\limits_{- -\frac{N-1}{2} \leq k_1, k_2 \leq \frac{N-1}{2}} \left( e^{-\lambda_{k_1,k_2}t}  P_{k_1,k_2}  -       e^{-\lambda^N_{k_1,k_2}t} P^N_{k_1,k_2}   \right) \right\|
	\end{align}
	For the first summand, we already have
	\begin{align}
		\left\|\sum_{\frac{N-1}{2} < |k_1|,|k_2|  } e^{-\lambda_{k_1,k_2}t}  P_{k_1,k_2} \right\| \leq e^{-t\frac{(N-1)^2}{2}}.
	\end{align}
	Hence let us investigate the second summand.
	We note
	\begin{align}\label{hardstyle}
		&\left\|  \sum\limits_{-\frac{N-1}{2} \leq k_1, k_2 \leq \frac{N-1}{2}} \left( e^{-\lambda_{k_1,k_2}t}  P_{k_1,k_2}  -       e^{-\lambda^N_{k_1,k_2}t} P^N_{k_1,k_2}   \right) \right\| \\
		\leq &\left\|  \sum\limits_{- \frac{N-1}{2} \leq k_1, k_2 \leq \frac{N-1}{2}} \left( e^{-\lambda_{k_1,k_2}t}-       e^{-\lambda^N_{k_1,k_2}t} \right) P^N_{k_1,k_2}     \right\| +
		\left\|  \sum\limits_{ -\frac{N-1}{2} \leq k_1, k_2 \leq \frac{N-1}{2}} e^{-\lambda_{k_1,k_2}t} ( P_{k_1,k_2}  -       P^N_{k_1,k_2}   ) \right\|
	\end{align}
	For the first summand we note
	\begin{align}
		&\left\|  \sum\limits_{- \frac{N-1}{2} \leq k_1, k_2 \leq \frac{N-1}{2}} \left( e^{-\lambda_{k_1,k_2}t}-       e^{-\lambda^N_{k_1,k_2}t} \right) P^N_{k_1,k_2}     \right\|  \\
		= &\sup\limits_{- \frac{N-1}{2} \leq k_1, k_2 \leq \frac{N-1}{2}} \left|e^{-\lambda_{k_1,k_2}t}-       e^{-\lambda^N_{k_1,k_2}t} \right|\\
		= &\sup\limits_{- \frac{N-1}{2} \leq k_1, k_2 \leq \frac{N-1}{2}}  e^{-t(k_1^2 + k_2^2)} \left|1 -  e^{-t\left( \frac{N^2}{\pi^2}\sin^2\left(\frac{\pi}{N}k_1\right) - k_1^2\right)} e^{-t\left( \frac{N^2}{\pi^2}\sin^2\left(\frac{\pi}{N}k_2\right) - k_2^2\right)}\right|
	\end{align}
	We note 
	\begin{align}
		\left( \frac{N^2}{\pi^2}\sin^2\left(\frac{\pi}{N}k\right)-k^2\right) = \mathcal{O}\left(\frac{k^4}{N^2}\right).
	\end{align}
	Using
	\begin{equation}
		\frac{N^2}{\pi^2}\sin^2\left(\frac{\pi}{N}N^{\frac13}\right) \lesssim N^\frac{2}{3}
	\end{equation}
	we note 
	\begin{align}
		&\sup\limits_{- \frac{N-1}{2} \leq k_1, k_2 \leq \frac{N-1}{2}}  e^{-t(k_1^2 + k_2^2)} \left|1 -  e^{-t\left( \frac{N^2}{\pi^2}\sin^2\left(\frac{\pi}{N}k_1\right) - k_1^2\right)} e^{-t\left( \frac{N^2}{\pi^2}\sin^2\left(\frac{\pi}{N}k_2\right) - k_2^2\right)}\right|\\
		\leq &
		\sup\limits_{|k_1|, |k_2| \leq N^{\frac13}}  e^{-t(k_1^2 + k_2^2)} \left|1 -  e^{-t\left( \frac{N^2}{\pi^2}\sin^2\left(\frac{\pi}{N}k_1\right) - k_1^2\right)} e^{-t\left( \frac{N^2}{\pi^2}\sin^2\left(\frac{\pi}{N}k_2\right) - k_2^2\right)}\right|\\
		+
		& \sup\limits_{|k_1|, |k_2| > N^{\frac13}}  e^{-t(k_1^2 + k_2^2)} \left|1 -  e^{-t\left( \frac{N^2}{\pi^2}\sin^2\left(\frac{\pi}{N}k_1\right) - k_1^2\right)} e^{-t\left( \frac{N^2}{\pi^2}\sin^2\left(\frac{\pi}{N}k_2\right) - k_2^2\right)}\right|\\
		&\leq e^{-t(2N^{\frac23})} +  e^{-t(2N^{\frac23})}  + e^{-t(N^{\frac23})} .
	\end{align}
	Hence it remains to bound the second summand in (\ref{hardstyle}). We note

	\begin{align}
		&\left\|  \sum\limits_{ -\frac{N-1}{2} \leq k_1, k_2 \leq \frac{N-1}{2}} e^{-\lambda_{k_1,k_2}t} ( P_{k_1,k_2}  -       P^N_{k_1,k_2}   ) \right\|\\
		\leq & \sum\limits_{ |k_1|, |k_2| \leq \frac{N-1}{2}}   e^{-(k_1^2 + k_2^2)t} \|P_{k_1,k_2}  -       P^N_{k_1,k_2} \|.
	\end{align}
	
	Next we note
	\begin{equation}
		\|P_{k_1,k_2}  -       P^N_{k_1,k_2} \| \leq 2 \left\|\phi_{k_1,k_2} - \phi_{k_1,k_2}\right\|.
	\end{equation}
	
	It is not hard to see that
	\begin{align}
		\left\|\phi_{k_1,k_2} - \overline{\phi^N_{k_1,k_2}}\right\| \leq 2C(|k_1| + |k|_2) \frac{2\pi}{N}
	\end{align}
	for some appropriately chosen $C > 0$. Hence we have

	\begin{align}
		&\left\|  \sum\limits_{ -\frac{N-1}{2} \leq k_1, k_2 \leq \frac{N-1}{2}} e^{-\lambda_{k_1,k_2}t} ( P_{k_1,k_2}  -       P^N_{k_1,k_2}   ) \right\|\\
		\leq& \sum\limits_{ |k_1|, |k_2| \leq \frac{N-1}{2}}   e^{-(k_1^2 + k_2^2)t}  \cdot 2C(|k_1| + |k|_2) \frac{2\pi}{N}\\
		=& \mathcal{O}(1/N).
	\end{align}
	
	Where the lass claim follows from summability in $k_1,k_2$. Thus we have in total indeed established that (\ref{torus_convergence}) holds.

In our experiments we then make use of the operators $J^{\uparrow\downarrow}_i$ defined above.  The function $f\in L^2(\mathcal{M})$ on the torus from which node features are derived is chosen as $
f =\frac{1}{4 \pi^2} \sin(\phi)\cos(\theta)$.
All networks are asked to predict a scalar signal on the respective graphs.

\section{Further Discussion of the setting of coarse-graining Graphs}\label{coarse_grain_proofs}

Let us 
consider  graphs that
contain clusters of nodes  which are connected by significantly larger edge weights than 
those
of edges outside of these clusters. 	From a diffusion perspective, information in a graph equalizes
faster along edges with large weights.

	\begin{figure}[H]
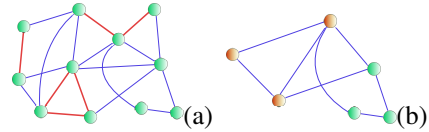

\centering
		\includegraphics[scale=0.2]{figures/graph}(a)\ \
		\includegraphics[scale=0.2]{figures/graphcollapsed}(b)
		\captionof{figure}{(a) $G$ (stongly connected) clusters in \textcolor{red}{red} (b) Coarse grained $\underline{G}$} 
		\label{collapse_weighted_II_app}
	\end{figure}

	In the limit where 
	edge-weights within certain  sub-graphs tend to infinity, information  within these clusters  equalizes immediately. Such clusters thus effectively behave as single nodes. 	We might thus consider a coarse grained graph $\underline{G}$ where  strongly connected clusters are  fused together and represented only via single nodes. 
	This naturally leads to the notion of  graph coarsification, as first formalized and studied in \citep{DBLP:conf/icml/LoukasV18, loukas2019graph}.

In our case at hand the node set $\underline{\mathcal{G}}$ of the coarse grained graph $\underline{G}$ is then given by the set of connected components
in $G_{\text{cluster}}$ (c.f. Fig \ref{G_high}). 
Edges $\underline{\mathcal{E}}$  
are given by elements  $ (R,P) \in \underline{\mathcal{G}} \times \underline{\mathcal{G}}$ with non-zero accumulated edge weight  
\nolinebreak[4]$\underline{W}_{RP} = \sum_{r\in R}\sum_{p\in P} W_{rp}$. 
Node weights in $\underline{G}$ are defined accordingly 
by  
\begin{minipage}{0.81\textwidth}
	aggregating	as $\underline{\mu}_R = \sum_{r \in R}\mu_r$.
	To compare signals on these two graphs, we define intertwining operators $J^\downarrow, J^\uparrow$ transferring information between  $G$ and $\underline{G}$: Let $x$ be a scalar graph signal and let  $\mathds{1}_R$ be the vector  that has $1$ as entry for nodes $r \in R$ and is zero otherwise. Denote by $u_R$  the entry of $u$ at node $R\in \underline{\mathcal{G}}$. 
	Projection $J^\downarrow$ is then defined
	component-wise by
	evaluation at node  $R \in \underline{\mathcal{G}}$ as the average of $x$ over $R$:  $(J^\downarrow x)_R = 
	\langle \mathds{1}_R, x\rangle/\underline{\mu}_R$.  
	Going in the opposite direction, 
\end{minipage}
\hfill
\begin{minipage}{0.17\textwidth}
	\begin{figure}[H]
		\vspace{-3mm}
		\includegraphics[scale=0.2]{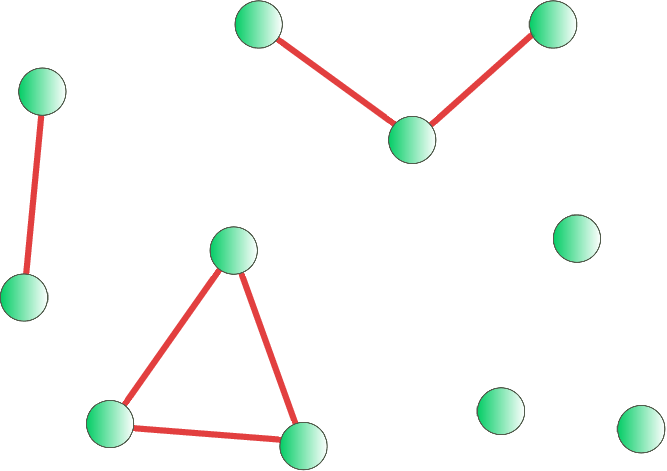}
		\captionof{figure}{
			$G_{\text{cluster}}$} 
		\label{G_high}
	\end{figure}
\end{minipage}
interpolation  
is defined as $J^\uparrow u = \sum_{R \in \underline{\mathcal{G}}} u_R \cdot\mathds{1}_R $.
In this setting, we
have (c.f. the discussion below) that 

\begin{equation}\label{orig_exp_conv_result}
	\|e^{-tL} - J^\uparrow e^{- t\underline{L}} J^\downarrow\| \lesssim 1/w_{\text{high}}^{\text{min}} \ \ \text{for any} \ \ t>0.
\end{equation}

Here $w_{\text{high}}^{\text{min}} \gg 1$ denotes the minimal 
edge weight inside the strongly connected clusters in $G$.
As the strength of the edge-weights in $G_{\text{cluster}}$ tends to infinity, we have by (\ref{orig_exp_conv_result}) that also $\eta(t) = \|e^{-Lt} - J^\uparrow e^{- \underline{L}t} J^\downarrow\| \rightarrow 0$ for any $t >0$. Thus (for $t>0$) the diffusion process $e^{-Lt} $ on $G$ acts essentially as first projecting the input-signal to $\underline{G}$ via $J^\downarrow$, then diffusing information on the coarse grained graph $\underline{G}$ via $e^{-\underline{L}t} $ and finally interpolating back to the original graph $G$ via $J^\uparrow$.

Following the proof in \cite{koke2026b}, we here illustrate
\begin{align}
	\|(L + Id)^{-1} - J^\uparrow (\underline{L} + Id)^{-1} J^\downarrow\| \lesssim 1/\lambda_{1}(L_{\text{high}}).
\end{align}
in this Appendix.
Using \citep[Corollary 3.18]{koke2026b}, this then also implies
\begin{equation}\label{orig_exp_conv_result_app}
	\|e^{-tL} - J^\uparrow e^{- t\underline{L}} J^\downarrow\| \lesssim 1/w_{\text{high}}^{\text{min}} \ \ \text{for any} \ \ t>0,
\end{equation}
after noting the linear relation in scaling behaviour $\lambda_1(L_{\text{cluster}}) \sim w_{\text{high}}^{\text{min}}$.

For convenience, we here explicitly restate  the definitions leading up to this setting in full clarity and generality:

\begin{Def}\label{app_lim_graph_def}
	Denote by 	$\underline{\mathcal{G}}$ the set of connected components in $G_{\text{high}}$. We give this set a graph structure as follows: Let $R$ and $P$ be elements of $\underline{\mathcal{G}}$ (i.e. connected components in $G_{\text{high}}$). We define the real number 
	\begin{equation}\label{new_W}
		\underline{W}_{RP} = \sum_{r\in R}\sum_{p\in P} W_{rp},
	\end{equation}
	with $r$ and $p$ nodes in the original graph $G$.
	We define the set of edges $\underline{\mathcal{E}}$ on $\underline{G}$ as  
	\begin{equation}
		\underline{\mathcal{E}} = \{(R,P)\in\underline{\mathcal{G}}\times\underline{\mathcal{G}}: \underline{W}_{RP} >0 \}
	\end{equation}
	and assign $\underline{W}_{RP}$ as weight to such edges.
	Node weights of limit nodes are defined similarly as aggregated weights of all nodes $r$ (in $G$) contained in the component $R$ as
	\begin{equation}\label{eq:weight_aggr}
		\underline{\mu}_R = \sum_{r \in R} \mu_r.
	\end{equation}
\end{Def}
In order to translate signals between the original graph $G$ and the limit description $\underline{G}$, we need translation operators mapping signals from one graph to the other:
\begin{Def}\label{J_defs}
	Denote by $\mathds{1}_R$ the vector that has $1$ as entries on nodes $r$ belonging to the connected (in $G_{\text{hign}}$) component $R$ and has entry zero for all nodes not in $R$. We define the down-projection operator $J^\downarrow$ component-wise via evaluating at node $R$ in $\underline{\mathcal{G}}$ as
	\begin{equation}\label{J_down}
		(J^\downarrow x)_R = \langle \mathds{1}_R, x\rangle/\underline{\mu}_R.
	\end{equation}
	The upsampling operator $J^\uparrow$ 
	is defined as 
	\begin{equation}\label{J_up}
		J^\uparrow u = \sum_R u_R \cdot\mathds{1}_R ;
	\end{equation}
	where $u_R$ is a scalar value (the component entry of $u$ at $R\in \underline{\mathcal{G}}$) and the sum is taken over all connected components in $G_{\text{high}}$.
\end{Def}

Slightly extending established results in \cite{koke2026b}, we then have the following result:

\begin{Thm}\label{main_resolvent_theorem_app}
	We have 
	\begin{equation}\label{resolvent_closeness_app}
		\left\|R_z(L) -  J^\uparrow R_z(\underline{L}) J^\downarrow\right\| =\mathcal{O}\left(  \frac{\|L_{\text{reg.}}\|}{\lambda_1(L_{\text{high}})}\right)
	\end{equation}
	holds; with $\lambda_1(L_{\text{high}})$ denoting the first non-zero eigenvalue of $L_{\text{high}}$.
\end{Thm}
\begin{equation}
	\lambda_{\max}(L_{\text{reg.}}) = \|L_{\text{reg.}}\|.
\end{equation}
We here inclue the proof of this fact for convenience and self-containedness.
\begin{proof}
	We will split the proof of this result into multiple steps.
	For $z<0$ Let us denote by
	\begin{align}
		R_z(L) &= (L - z Id)^{-1},\\
		R_z(L_{\textit{high}}) &= (L_{\textit{high}} - z Id)^{-1}\\
		R_z(L_{\textit{reg.}}) &= (L_{\textit{reg.}} - z Id)^{-1}
	\end{align}
	the resolvents correspodning to $L$, $L_{\textit{high}}$ and $L_{\textit{reg.}}$ respectively.\\
	Our first goal is establishing that we may write
	\begin{equation}
		R_z(L) = \left[Id + R_z(L_{\textit{high}})L_{\textit{reg.}} \right]^{-1}\cdot R_z(L_{\textit{high}})
	\end{equation}
	This will follow as a consequence of what is called the second resolvent formula \cite{Teschl}: 
	
	"Given self-adjoint operators $A,B$, we may write
	\begin{equation}
		R_z(A+B) - R_z(A) =  - R_z(A)BR_z(A+B)."
	\end{equation}

	In our case, this translates to
	\begin{equation}
		R_z(L) - R_z(L_{\textit{high}}) = -R_z(L_{\textit{high}}) L_{\text{reg.}}   R_z(L) 
	\end{equation}
	or equivalently
	\begin{equation}
		\left[Id + R_z(L_{\textit{high}})L_{\text{reg.}} \right]R_z(L) =  R_z(L_{\textit{high}}).
	\end{equation}
	Multiplying with $\left[Id +R_z(L_{\textit{high}}) L_{\text{reg.}} \right]^{-1}$ from the left then yields 
	\begin{equation}
		R_z(L) = \left[Id + R_z(L_{\textit{high}})L_{\textit{reg.}} \right]^{-1}\cdot R_z(L_{\textit{high}})
	\end{equation}
	as desired. \\
	Hence we need to establish that $\left[Id + R_z(L_{\textit{high}}) L_{\textit{reg.}} \right]$ is invertible for $z<0$.\\
	
	To establish a contradiction, assume it is not invertible. Then there is a signal $x$ such that
	\begin{equation}
		\left[Id + R_z(L_{\textit{high}}) L_{\textit{reg.}} \right]x = 0.
	\end{equation}
	Multiplying with $(L_{\text{high}} - z Id)$ from the left yields
	\begin{equation}
		(L_{\text{high}} + L_{\text{reg.}} - z Id)x = 0
	\end{equation}
	which is precisely to say that 
	\begin{equation}
		(L - z Id)x = 0
	\end{equation}
	But since $L$ is a graph Laplacian, it only has non-negative eigenvalues. Hence we have reached our contradiction and established
	\begin{equation}
		R_z(L) = \left[Id +R_z(L_{\textit{high}}) L_{\textit{reg.}} \right]^{-1}R_z(L_{\textit{high}}).
	\end{equation}
	\ \\
	Our next step is to establish that 
	\begin{equation}
		R_z(L_{\textit{high}}) \rightarrow \frac{P^{\text{high}}_0}{-z},
	\end{equation}
	where $P^{\text{high}}_0$ is the spectral projection onto the eigenspace corresponding to the lowest lying eigenvalue $\lambda_0(L_{\textit{high}}) = 0$ of $L_{\textit{high}}$. Indeed, by the spectral theorem for finite dimensional operators (c.f. e.g. \cite{Teschl}), we may write
	\begin{equation}
		R_z(\Delta_{\textit{high}}) \equiv (L_{\textit{high}} - z Id)^{-1} = \sum\limits_{\lambda \in \sigma(L_{\textit{high}})} \frac{1}{\lambda-z}\cdot P^{\textit{high}}_{\lambda}.
	\end{equation}
	Here $\sigma(L_{\textit{high}})$ denotes the spectrum (i.e. the collection of eigenvalues) of $L_{\textit{high}}$ and the $\{P_\lambda^{\textit{high}}\}_{\lambda \in \sigma(L_{\textit{high}})}$ are the corresponding (orthogonal) eigenprojections onto the eigenspaces of the respective eigenvalues.
	Thus we find
	\begin{equation}
		\left\|R_z(L_{\textit{high}}) - \frac{P_0^{\textit{high}}}{-z}\right\| = \left|\sum\limits_{0<\lambda \in \sigma(L_{\textit{high}})} \frac{1}{\lambda-z}\cdot P^{\textit{high}}_{\lambda}\right\|;
	\end{equation}
	where the sum on the right hand side now excludes the eigenvalue $\lambda = 0$.

	Using orthonormality of the spectral projections, the fact that $z <0$ and monotonicity of $1/(\cdot+|z|)$ we find
	\begin{equation}
		\left\|R_z(L_{\textit{high}}) - \frac{P_0^{\textit{high}}}{-z}\right\| = \frac{1}{\lambda_1(\Delta_{\textit{high}})+|z|}.
	\end{equation}
	Here $\lambda_1(L_{\textit{high}})$ is the firt non-zero eigenvalue of $(L_{\textit{high}})$.\\
	Non-zero eigenvalues scale linearly with the weight scale since we have
	\begin{equation}
		\lambda(S\cdot L) = S \cdot \lambda(L)
	\end{equation}
	for any graph Laplacian (in fact any matrix) $L$ with eigenvalue $\lambda$. Thus we have
	\begin{equation}
		\left\|R_z(L_{\textit{high}}) - \frac{P_0^{\textit{high}}}{-z}\right\| = \frac{1}{\lambda_1(L_{\textit{high}})+|z|}  \leq \frac{1}{\lambda_1(L_{\textit{high}})} \longrightarrow 0
	\end{equation}
	as $\lambda_1(L_{\textit{high}})\rightarrow \infty$.\\
	
	Our next task is to use this result in order to bound the difference  
	\begin{equation}
		I := \left\| \left[Id + \frac{P_0^{\textit{high}}}{-z}L_{\textit{reg.}} \right]^{-1}\frac{P_0^{\textit{high}}}{-z} - \left[Id + R_z(L_{\textit{high}})L_{\textit{reg.}} \right]^{-1} R_z(L_{\textit{high}})\right\|.
	\end{equation}
	To this end we first note that the relation
	\begin{equation}
		[A+B-zId]^{-1} = [Id+R_z(A)B]^{-1}R_z(A)
	\end{equation}
	provided to us by the second resolvent formula, implies 
	\begin{equation}
		[Id+R_z(A)B]^{-1} = Id -B [A+B-zId]^{-1} .
	\end{equation}
	Thus we have
	\begin{align}
		\left\|\left[Id + R_z(L_{\textit{high}})L_{\textit{reg.}} \right]^{-1}\right\| &\leq  1 + \|L_{\text{reg.}}\|\cdot\|R_z(L)\|\\
		&\leq 1+ \frac{\|L_{\text{reg.}}\|}{|z|}.
	\end{align}
	With this, we have

	\begin{align}
		&\left\| \left[Id + \frac{P_0^{\textit{high}}}{-z}L_{\textit{reg.}} \right]^{-1}\cdot\frac{P_0^{\textit{high}}}{-z} -R_z(L)\right\|\\
		=&
		\left\| \left[Id + \frac{P_0^{\textit{high}}}{-z}L_{\textit{reg.}} \right]^{-1}\cdot\frac{P_0^{\textit{high}}}{-z} - \left[Id + R_z(L_{\textit{high}})L_{\textit{reg.}} \right]^{-1}\cdot R_z(L_{\textit{high}})\right\|\\
		\leq&\left\|\frac{P_0^{\textit{high}}}{-z} \right\|\cdot \left\| \left[Id + \frac{P_0^{\textit{high}}}{-z}L_{\textit{reg.}} \right]^{-1} - \left[Id + R_z(L_{\textit{high}})L_{\textit{reg.}} \right]^{-1}\right\| + \left\|\frac{P_0^{\textit{high}}}{-z} - R_z(L_{\text{high}})\right\|\cdot \left\|\left[Id + R_z(L_{\textit{high}})L_{\textit{reg.}} \right]^{-1} \right\|\\
		\leq& \frac{1}{|z|} \left\| \left[Id + \frac{P_0^{\textit{high}}}{-z}L_{\textit{reg.}} \right]^{-1} - \left[Id + R_z(L_{\textit{high}})L_{\textit{reg.}} \right]^{-1}\right\| + \left( 1+ \frac{\|L_{\text{reg.}}\|}{|z|}\right) \cdot \frac{1}{\lambda_1(L_{\textit{high}})}.
	\end{align}
	Hence it remains to bound the left hand summand. For this we use the following fact (c.f. \cite{horn}, Section 5.8. "Condition numbers: inverses and linear systems"):
	\ \\
	\ \\ 
	Given square matrices $A,B,C$ with $C = B-A$ and $\|A^{-1}C\|<1$, we have
	\begin{equation}
		\|A^{-1} - B^{-1}\|\leq \frac{\|A^{-1}\|\cdot \|A^{-1}C\|}{1 - \|A^{-1}C\|}.
	\end{equation}
	In our case, this yields (together with $\|P_0^{\textit{high}}\| = 1$) that
	\begin{align}
		&\left\| \left[Id + P_0^{\textit{high}}/(-z)\cdot L_{\textit{reg.}} \right]^{-1} - \left[Id + R_z(L_{\textit{high}})L_{\textit{reg.}} \right]^{-1}\right\|\\
		\leq&\frac{\left(1 + \|L_{\text{reg.}}\|/|z|\right)^2\cdot\|L_{\text{reg.}}\|\cdot\|\frac{P_0^\text{high}}{-z} - R_z(L_{\text{high}})\|}{1-\left(1 + \|L_{\text{reg.}}\|/|z|\right)\cdot\|L_{\text{reg.}}\|\cdot\|\frac{P_0^\text{high}}{-z} - R_z(L_{\text{high}})\|}
	\end{align}
	For $S_{\text{high}}$ sufficiently large, we have
	\begin{equation}
		\|-P_0^\text{high}/z - R_z(L_{\text{high}})\| \leq \frac{1}{2 \left(1 + \|L_{\text{reg.}}\|/|z|\right)}
	\end{equation}
	so that we may estimate
	
	\begin{align}
		&\left\| \left[Id + L_{\textit{reg.}}\frac{P_0^{\textit{high}}}{-z} \right]^{-1} - \left[Id + L_{\textit{reg.}}R_z(L_{\textit{high}}) \right]^{-1}\right\|\\
		\leq& 2\cdot(1+\|L_{\text{reg.}}\|)\cdot\|\frac{P_0^\text{high}}{-z} - R_z(L_{\text{high}})\|\\
		=&2\frac{1+\|L_{\text{reg.}}\|/|z|}{\lambda_1(L_{\text{high}})}
	\end{align}
	Thus we have now established
	\begin{equation}
		\left| \left[Id +\frac{P_0^{\textit{high}}}{-z}  L_{\textit{reg.}} \right]^{-1}\cdot\frac{P_0^{\textit{high}}}{-z} -R_z(L)\right| = \mathcal{O}\left(\frac{\|L_{\text{reg.}}\|}{\lambda_1(L_{\text{high}})}\right).
	\end{equation}
	\ \\
	Hence we are done with the proof, as soon as we can establish
	\begin{equation}
		\left[-z Id + P_0^{\textit{high}}L_{\textit{reg.}} \right]^{-1}P_0^{\textit{high}} = J^\uparrow R_z(\underline{L})J^\downarrow,
	\end{equation}
	with $J^\uparrow, \underline{L},  J^\downarrow$ as defined above.
	To this end, we first note that
	\begin{equation}\label{J_proj}
		J^\uparrow\cdot J^\downarrow = P_0^{\textit{high}}
	\end{equation}
	and
	\begin{equation}\label{J_id}
		J^\downarrow\cdot J^\uparrow = Id_{\underline{G}}.
	\end{equation}
	Indeed,the relation (\ref{J_proj}) follows from the fact that the  eigenspace corresponding to the eignvalue zero is spanned by the vectors $\{\mathds{1}_R\}_{R}$, with $\{R\}$ the connected components of $G_{\text{high}}$. Equation (\ref{J_id}) follows from the fact that
	\begin{equation}
		\langle \mathds{1}_R,\mathds{1}_R\rangle = \underline{\mu}_R.
	\end{equation}
	With this we have
	\begin{equation}
		\left[Id + P_0^{\textit{high}}L_{\textit{reg.}} \right]^{-1}P_0^{\textit{high}} = \left[Id + J^\uparrow J^\downarrow L_{\textit{reg.}} \right]^{-1}J^\uparrow J^\downarrow.
	\end{equation}
	To proceed, set 
	\begin{equation}
		\underline{x} := F^\downarrow x
	\end{equation}
	and
	\begin{equation}
		\mathscr{X} = \left[P_0^{\textit{high}}L_{\textit{reg.}} -z Id  \right]^{-1}P_0^{\textit{high}} x.
	\end{equation}
	Then 
	\begin{equation}
		\left[P_0^{\textit{high}}L_{\textit{reg.}} -z Id  \right] \mathscr{X} = P_0^{\textit{high}}x
	\end{equation}
	and hence $\mathscr{X} \in \text{Ran}(P_0^{\textit{high}})$.
	Thus we have
	\begin{equation}
		J^\uparrow J^\downarrow (L_{\text{reg.}}-z Id)J^\uparrow J^\downarrow\mathscr{X} = J^\uparrow J^\downarrow x.
	\end{equation}
	Multiplying with $J^\downarrow$ from the left yields
	\begin{equation}
		J^\downarrow (L_{\text{reg.}}-z Id)J^\uparrow J^\downarrow\mathscr{X} = J^\downarrow x.
	\end{equation}
	Thus we have
	\begin{equation}
		(J^\downarrow L_{\text{reg.}}J^\uparrow-z Id)J^\uparrow J^\downarrow\mathscr{X} = J^\downarrow x.
	\end{equation}
	This -- in turn -- implies 
	\begin{equation}
		J^\uparrow J^\downarrow  \mathscr{X} 
		= \left[J^\downarrow L_{\text{reg.}}J^\uparrow - z Id \right]^{-1}J^\downarrow x.
	\end{equation}
	Using 
	\begin{equation}
		P_0^{\textit{high}} \mathscr{X} = \mathscr{X},
	\end{equation}
	we then have
	\begin{equation}
		\mathscr{X} = J^\uparrow\left[J^\downarrow L_{\text{reg.}}J^\uparrow - z Id \right]^{-1}J^\downarrow x.
	\end{equation}
	We have thus concluded the proof if we can prove that $J^\downarrow L_{\text{reg.}}J^\uparrow$ is the Laplacian corresponding to the graph $\underline{G}$ defined in Definition \ref{app_lim_graph_def}. But this is a straightforward calculation.
\end{proof}

As a corollary, one finds

\begin{Cor}\label{polynom_cor}
	We have
	\begin{equation}
		R_z(L)^k \rightarrow J^\uparrow R^k(\underline{L})J^\downarrow
	\end{equation}
\end{Cor}
\begin{proof}
	This follows directly from the fact that
	\begin{equation}
		J^\downarrow J^\uparrow = Id_{\underline{G}}.
	\end{equation}
\end{proof}

\section{Computational cost analysis}\label{comp_cost}

To evaluate the practical efficiency of our proposed framework, we conducted a thorough empirical investigation of the computational costs associated with constructing and operating on a representative propagation matrix---specifically, the resolvent formulation introduced in Def, \ref{def:LTF_def}.

\subsection{Computational Efficiency on Small-Scale Graphs (QM7)}
We first established a baseline for resource consumption during model training on the QM7 molecular dataset. To quantify computational overhead, we monitored key performance metrics throughout training:
\begin{itemize}
	\item \textbf{Peak GPU Memory Allocation} was tracked using  \texttt{torch.cuda.memory.max\_memory\_allocated}.
	\item \textbf{Training Time per Epoch}  was measured as the wall-clock time required to complete a single forward and backward pass across all training samples.
\end{itemize}

\begin{table}[h!]
		\caption{
		Computational overhead on small graphs
	}
	\label{app_comp_ov_small}
	\centering
	\begin{tabular}{l c c}
		\toprule
		Model & Max Memory [MiB] & Time/Epoch [s] \\
		\midrule
		ChebNet & \small{230.8$\pm$0.8} & 0.8\small{$\pm0.1$} \\
		GAT &  248.0\small{$\pm$0.5}  & 1.0\small{$\pm$0.1}  \\
		Resolvent &184.8\small{$\pm0.1$}& 0.7\small{$\pm$0.1} \\
		\bottomrule
	\end{tabular}
\end{table}

As evident from Table \ref{app_comp_ov_small}, our method has significantly less overhead than standard GNN baselines.

\subsection{Scalability to Large-Scale Graphs}
To assess how our approach scales to larger, more complex topologies, we extended our evaluation to datasets with significantly higher node counts, namely \texttt{ms\_academic} ($>18$,$000$ nodes) and \texttt{pubmed} ($>19$,$000$ nodes).

\paragraph{Verification of effective sparsity:}
We first experimentally verify the claim made towards the end of Section \ref{subsec:sc_gnns} that the propagation matrix is effectively sparse.

\begin{figure}[H]
	(a)\includegraphics[width=0.5\linewidth]{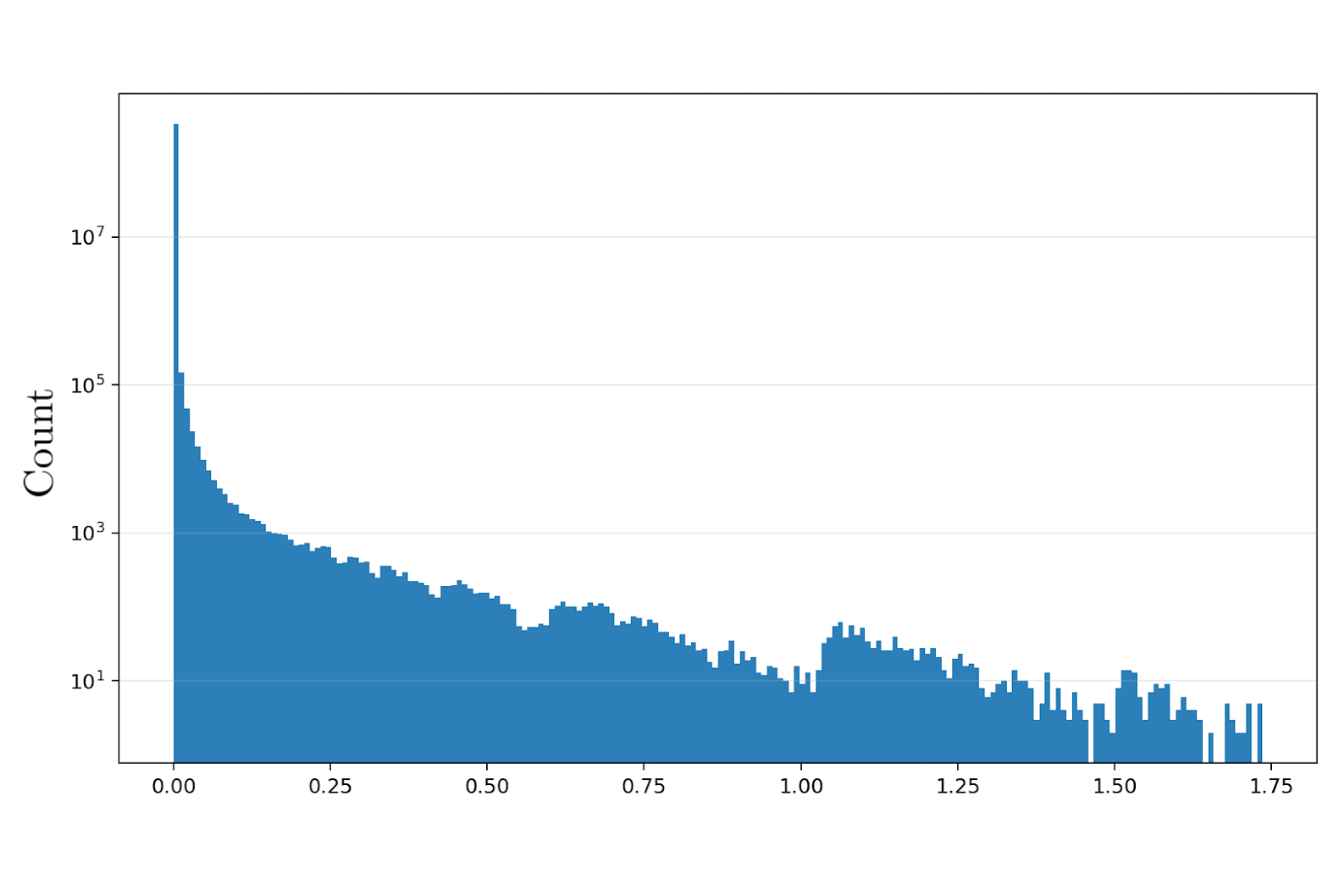}\hfill
	(b)\includegraphics[width=0.5\linewidth]{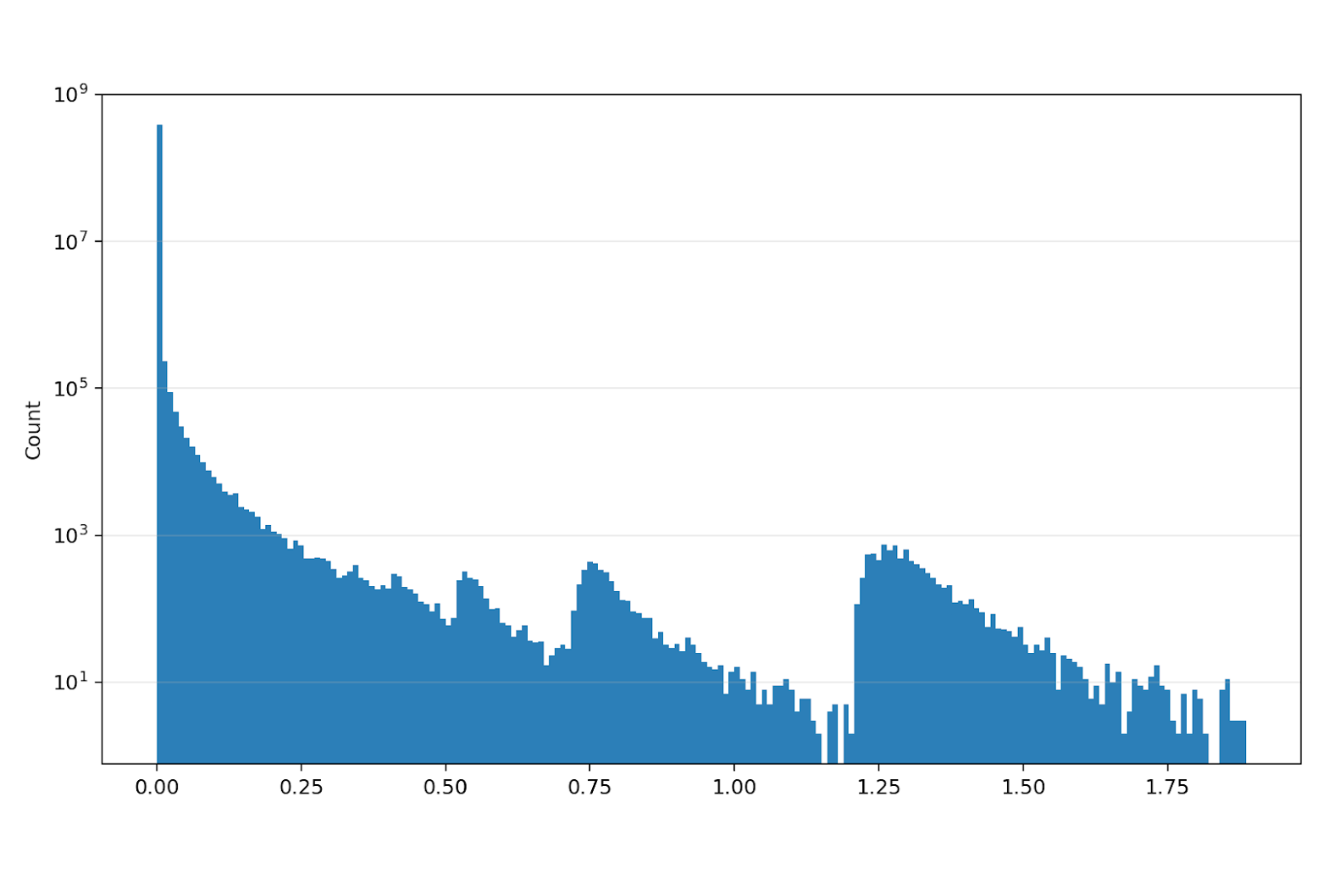}
	\captionof{figure}{ Histogram of $|R_{ij}|$ for resolvent matrix $R$ on (a) \texttt{ms\_academic} and (b)  \texttt{pubmed} . } 
\label{fig:histograms}
\end{figure}

As is evident from Figure \ref{fig:histograms}, the overwhelming majority of entries $R_{ij}$ in the resolvent propagation matrix $R$ are indeed effectively zero. Indeed, we list results for small cutoff values $\epsilon \geq |R_{ij}|$ in Table \ref{tab:hist} below.

\begin{table}[h!]
			\caption{
		Effective sparsity structure
	}
	\label{tab:hist}
	\centering
	\begin{tabular}{ l | c | c | c | c }
		\hline
		Dataset & \multicolumn{2}{c|}{$|R_{ij}| \geq 0.05$} & \multicolumn{2}{c}{$|R_{ij}| \geq 0.01$} \\
		\hline
		& Total number [\#] & Percentage [\%] & Total number [\#] & Percentage [\%] \\
		\hline
		PubMed & 121,053 & 0.03\% & 500,571 & 0.13\% \\
		MS-Academic & 56,096 & 0.02\% & 257,475 & 0.07\% \\
		\hline
	\end{tabular}
\end{table}

\paragraph{Computational overhead analysis}

Using the identical logging infrastructure and metrics as in the QM7 experiments, we tracked the scaling behavior of memory allocation and epoch duration.

\begin{table}[h!]
	\caption{Computational overhead on large graphs}
	\centering
	\begin{tabular}{l c c}
		\toprule
		Model & Max Memory [MiB] & Time/Epoch [ms] \\
		\midrule
		ChebNet & 551.0\small{$\pm$9.0} & 107.14\small{$\pm$1.21} \\
		GAT &  2473.6\small{$\pm$10.4}  & 270.78\small{$\pm$101.25}  \\
		Resolvent &168.9\small{$\pm$0.9}& 35.86\small{$\pm$0.19} \\
		\bottomrule
	\end{tabular}
\end{table}

Our empirical results  validate the theoretical assertions presented in our scalability discussion at the end of Section \ref{subsec:sc_gnns}.

\end{document}